\documentclass[11pt,letterpaper]{article}
\usepackage[T1]{fontenc}
\usepackage[utf8]{inputenc}
\usepackage[margin=1in]{geometry}
\usepackage[authoryear,round]{natbib}

\usepackage{amsmath,amsfonts,bm}

\def\1{\bm{1}}

\DeclareMathAlphabet{\mathsfit}{\encodingdefault}{\sfdefault}{m}{sl}
\SetMathAlphabet{\mathsfit}{bold}{\encodingdefault}{\sfdefault}{bx}{n}

\newcommand{\E}{\mathbb{E}}

\DeclareMathOperator{\softmax}{softmax}

\newcommand{\KL}{D_{\mathrm{KL}}}
\newcommand{\Var}{\mathrm{Var}}

\newcommand{\ProximalLogMoment}{F.1}
\newcommand{\ProximalTransfer}{F.2}
\newcommand{\ProximalAngular}{F.9}

\usepackage{amsmath,amssymb,amsfonts,amsthm,mathtools}
\usepackage{booktabs}
\usepackage{graphicx}
\usepackage{caption}
\usepackage{float}
\floatstyle{ruled}
\newfloat{algorithm}{tbp}{loa}
\floatname{algorithm}{Algorithm}
\floatstyle{plain}
\newcounter{algline}
\newcommand{\algline}[1][]{\refstepcounter{algline}\if\relax\detokenize{#1}\relax\else\label{#1}\fi\makebox[1.4em][r]{\scriptsize\thealgline:}\>}
\usepackage{microtype}
\usepackage{xcolor}
\usepackage{hyperref}
\hypersetup{hidelinks}
\usepackage{url}
\newtheorem{theorem}{Theorem}
\newtheorem{proposition}[theorem]{Proposition}
\newtheorem{lemma}[theorem]{Lemma}

\theoremstyle{definition}
\newtheorem{definition}{Definition}
\newtheorem{assumption}{Assumption}
\theoremstyle{remark}
\newtheorem{remark}[theorem]{Remark}

\newcommand{\piref}{\pi_{\mathrm{ref}}}
\newcommand{\Tref}{T_{\mathrm{ref}}}
\newcommand{\pib}{\pi_{\mathrm{beh}}}
\newcommand{\pibt}[1]{\pi_{\mathrm{beh},#1}}
\newcommand{\Jbeh}{\mathcal J_{\mathrm{beh}}}
\newcommand{\Rth}{R_{\theta}}
\newcommand{\qwenone}{Qwen3-1.7B-Base}
\newcommand{\qwenfour}{Qwen3-4B-Base}
\newcommand{\actorlr}{10^{-6}}

\title{\Large Gibbs-Gap Policy Optimization for Off-Policy RL with Verifiable Rewards}

\author{Xuanchang Zhang \qquad Tong Zhang\\
\normalfont University of Illinois Urbana-Champaign}
\date{}
\hypersetup{
  pdftitle={Gibbs-Gap Policy Optimization for Off-Policy RL with Verifiable Rewards},
  pdfauthor={Xuanchang Zhang and Tong Zhang}
}

\begin{document}
\maketitle

\begin{abstract}
Reinforcement learning with verifiable rewards often trains language models on
responses generated by a different, possibly stale behavior policy. Correcting
this mismatch with learner-to-behavior importance weights can introduce high
variance for long responses. We introduce Gibbs-Gap Policy Optimization
(GGPO), which minimizes a log-mean-exp contrast of KL-adjusted rewards among
responses to the same prompt, without learner-to-behavior importance weights.
We derive GGPO from KL-regularized policy optimization: on policy, its population
loss is exactly the KL-regularized optimality gap, and under suitable coverage the
learner's single-response expected reward is at least the behavior policy's
best-of-$K$ reward, minus terms for the raw finite-group loss, regularization,
and finite training data. We characterize reference refresh through exact
Gibbs-target composition and a finite-data final-policy reward bound under
local reference control and tokenwise behavior mismatch. This refresh is
empirically important: in an ablation, fixed-reference GGPO scores $32.96$ on
the five-benchmark problem-weighted average, versus $46.58$ with the default
refresh schedule. Experiments on mathematical reasoning further show that GGPO
outperforms GRPO and AGRO under training--inference mismatch and on fixed
offline data, while maintaining its performance as controlled policy staleness
grows.\footnote{Code is available at \url{https://github.com/thunderous77/GGPO}.}
\end{abstract}

\section{Introduction}

Reinforcement learning with verifiable rewards (RLVR) improves language-model
reasoning by using automatic checks to evaluate task outcomes
\citep{deepseekai2025deepseekr1}. In group-based methods, the model generates
several responses to each prompt, and differences in their rewards provide
a learning signal \citep{shao2024deepseekmath,yu2025dapo}. These groups also
suggest a natural performance target. If selecting the highest-reward
response from a group improves on drawing a single response, we would like
to train a learner whose individual responses approach the reward of that
selection procedure. This motivation also underlies reward-ranked
fine-tuning methods such as RAFT \citep{dong2023raft} and reasoning
variants such as RAFT++ \citep{xiong2025minimalist}. In practice, however,
the policy that generates the
group may differ from the learner being trained. During asynchronous
generation, responses can arrive from an earlier policy version; even within
a rollout batch, sequential learner updates can make the remaining responses
stale \citep{noukhovitch2025asynchronous,zheng2025stabilizing}. We refer to
the policy generating these responses as the behavior policy $\pib$, and
to the policy being trained as the learner $\pi_\theta$. The distinction is
also relevant in offline training, where sampled responses remain fixed
while the learner continues to change. Moreover, synchronizing policy
parameters does not necessarily eliminate the mismatch: training and
inference can use different implementations or numerical precision
\citep{minimax2025m1,zhong2026diagnosing}.

A standard way to account for this mismatch is importance sampling, which
weights behavior-policy data by the learner-to-behavior likelihood ratio
$\pi_\theta/\pib$. For an autoregressive response, this ratio is a product
of conditional token ratios. Small differences at individual token positions
can therefore accumulate over a long response, producing highly variable
weights and making the update sensitive to a few sampled responses.
Practical methods control these contributions through clipping, truncation,
or masking
\citep{schulman2017ppo,yu2025dapo,zhong2026diagnosing,lingteam2025everystep,zheng2025gspo}.
These choices involve both where to apply the correction, such as at the
token or response level, and how strongly to limit it. They can reduce
variance, but also introduce bias. This tradeoff motivates an alternative:
learn directly from behavior-policy response groups through an objective
that retains a quantitative connection to the learner's reward, without
requiring learner-to-behavior importance weights in the update.

We introduce Gibbs-Gap Policy Optimization (GGPO), which learns from
contrasts in KL-adjusted rewards among responses to the same prompt, without
learner-to-behavior importance weights. Under bounded coverage, its
finite-group loss controls the gap between the learner's single-response
expected reward and the reward obtained by best-of-$K$ selection from the
behavior policy. Thus the response groups serve both as training data and
as a reward benchmark for the learner. Groups of order $K$ suffice for this
comparison, and a uniform generalization bound extends it to finite data.

The construction follows from KL-regularized duality, which identifies a
Gibbs policy whose adjusted reward is constant across responses. GGPO
measures departures from this condition through contrasts within each group.
On policy, its population loss equals the regularized optimality gap.
For finite groups, we establish the reward comparison directly, without
requiring accurate estimation of that population loss.

A fixed reference defines a single Gibbs target. Refreshing the reference
to the current learner changes that target, allowing successive rounds to
build on earlier progress. We characterize how exact targets compose across
rounds and bound the reward of averaged outputs when updates are imperfect,
under coverage and bounded adjusted-reward spread. Under binary rewards and
local reference control, a finite-data guarantee bounds the final policy's
reward using aggregate tokenwise mismatch and achieved fitting accuracy.
Under its assumptions, the expected reward of one response from the trained
policy can approach that of selecting the best of $K$ responses from the
initial policy.
Section~\ref{sec:integration-convergence} summarizes this guarantee;
Appendix~\ref{sec:integration-formal-setup} gives the precise result.

GGPO builds on earlier work on off-policy consistency. AGRO
\citep{tang2025agro} minimizes the variance of the same adjusted reward and
identifies the same regularized optimum under common support. Although the
two objectives share this target, they measure departures from it differently.
We compare their objectives and guarantees in Section~\ref{sec:related-work}
and evaluate GGPO against GRPO and AGRO under training--inference mismatch,
controlled policy staleness, and offline training.

Our contributions are:
\begin{itemize}
  \item \textbf{An off-policy group objective with a direct reward guarantee.}
  GGPO minimizes the log-mean-exp contrast of KL-adjusted rewards within
  behavior-policy response groups, without learner-to-behavior importance
  weights. On policy, its population value is exactly the KL-regularized
  optimality gap. Off policy, coverage bounds that gap. Its raw finite-group
  loss bounds the gap between the learner's single-response expected reward and the
  behavior policy's best-of-$K$ reward, with explicit terms for the loss,
  regularization, reference--behavior mismatch, and group size. A uniform empirical bound and
  oracle inequality extend this comparison to finite data.
  \item \textbf{A characterization of periodic reference refresh.}
  When each round exactly fits its Gibbs target, $T$ refresh rounds at
  temperature $\beta$ implement one Gibbs update at temperature $\beta/T$.
  This explains how refresh changes the target across rounds. For imperfect
  updates, a finite-data
  guarantee controls the final policy reward using aggregate tokenwise
  mismatch and achieved inner optimization accuracy.
  An ablation shows that periodic refresh substantially improves GGPO under
  the default training setting.
  \item \textbf{Evaluation across behavior--learner mismatch regimes.}
  On five mathematical-reasoning benchmarks and two model sizes, GGPO
  outperforms GRPO and its closest objective-level comparator AGRO under
  training--inference mismatch. It maintains its accuracy as the synchronization
  delay grows from $N=8$ to $N=128$, while GRPO and AGRO decline. At $N=128$,
  token-level TIS improves GRPO, while GGPO achieves higher accuracy under
  the same synchronization delay.
  On a fixed offline rollout pool, GGPO has the highest MATH-500 accuracy at every
  checkpoint from 160 updates onward.
\end{itemize}

\section{Related work}
\label{sec:related-work}

\paragraph{KL-regularized and direct policy objectives.}
Gibbs variational identities connect KL regularization to exponential tilting
\citep{richemond2017variational,geist2019regularized}.
These identities also provide a starting point for direct policy objectives.
Objectives derived from related policy identities include DPO
\citep{rafailov2023dpo}, InfoNCA/NCA \citep{chen2024nca}, REBEL
\citep{gao2024rebel}, and $A^*$-PO \citep{brantley2025astarpo}.
ConSPO \citep{zhang2026conspo} contrasts verified positive and negative
responses using length-normalized sequence log probabilities in an
InfoNCE-style objective. GGPO instead applies a log-mean-exp gap to
KL-adjusted rewards. In this construction, the behavior policy generates
the samples, while a distinct reference policy provides the KL anchor.
InfoNCA uses learner-to-reference logits with reward-softmax labels;
GGPO instead minimizes a KL divergence from uniform group weights to an
adjusted-reward tilt.

\paragraph{Group RLVR and behavior--learner mismatch.}
GRPO \citep{shao2024deepseekmath} uses within-group advantages in a PPO
\citep{schulman2017ppo} surrogate. Building on this group-based approach, DAPO \citep{yu2025dapo}, Dr.GRPO
\citep{liu2025understanding}, and GSPO \citep{zheng2025gspo} modify sampling,
normalization, or PPO-style ratio aggregation.
MaxRL \citep{tajwar2026maxrl} derives success-normalized on-policy updates
from a maximum-likelihood formulation. Its centered estimator is related to
a special binary-reward limit of GGPO's update coefficients, detailed in
Appendix~\ref{app:gradient}.

When the behavior policy differs from the learner, the update must also
account for how the responses were sampled. Under this mismatch,
asynchronous RLHF \citep{noukhovitch2025asynchronous} studies staleness and
IMPALA \citep{espeholt2018impala} uses V-trace. In LLM training,
\citet{zhong2026diagnosing} studies truncated old-trainer-to-rollout
correction weights alongside a PPO ratio whose denominator is the old
trainer policy. This correction and PPO denominator serve different roles
from GGPO's KL reference. Other methods mask ratio-based
contributions \citep{lingteam2025everystep,li2026trustmasking} or restructure
learner-to-behavior ratios \citep{zhang2026rethinking}. These corrections trade
importance-weight variance against clipping, truncation, or masking bias
\citep{huang2026stable}. GGPO does not evaluate learner-to-behavior correction
ratios, while its population guarantee depends on behavior coverage.

In concurrent work, \citet{marek2026scorecentering} center
token-level policy scores under the sampling distribution to remove an
additional drift term induced by training--inference mismatch. This additive
correction can be combined with importance weighting. GGPO derives a
response-group objective from KL-regularized duality and contrasts
KL-adjusted rewards using a separately maintained reference. Its update uses
likelihoods of sampled responses under the learner and reference, rather than
the sampler's next-token distribution.

\paragraph{Objectives without learner-to-behavior importance weighting.}
PCL/Trust-PCL \citep{nachum2017pcl,nachum2018trustpcl} optimize squared soft
consistency residuals on off-policy trajectories. In language-model training, Kimi k1.5
\citep{kimiteam2025kimi} uses squared Gibbs-consistency regression with an
iteratively updated reference. Contrastive Policy Gradient
\citep{fletberliac2024copg} uses a pairwise quadratic objective under fixed
sampling distributions. OAPL \citep{ritter2026offpolicy} regresses an estimated
optimal advantage using the behavior policy as its KL anchor; its
learner-to-behavior log ratio is a regression term, rather than an
importance weight. AsymRE
\citep{arnal2025asymre} scales positive and
negative feedback asymmetrically. DRO
\citep{richemond2024dro} applies a squared penalty to the optimality residual
while jointly estimating the value offset. These methods show that off-policy learning can be approached through
consistency conditions, without learner-to-behavior correction ratios.
GGPO follows this general approach, using a contrastive loss whose value
can be related directly to reward.

\paragraph{Relation to AGRO.}
AGRO \citep{tang2025agro} is our closest objective-level comparator and an
earlier example with an off-policy consistency guarantee. It minimizes the
behavior variance of the same KL-adjusted reward. Theorem~2 of
\citet{tang2025agro} establishes that the regularized optimal policy is the
unique global minimizer under common support, shared by GGPO.

GGPO's distinction is the direct alignment of its loss value with the
KL-regularized objective. Its exponential form is derived from the exact
duality residual. This gives a direct interpretation of the loss value:
at the population level and when behavior equals learner,
the loss equals the regularized optimality gap for every positive candidate
policy. Off policy, the connection becomes a
coverage-dependent upper bound. AGRO's variance loss, scaled by $1/\beta$,
is the local quadratic approximation to this population contrastive loss when
adjusted-reward spread is small relative to $\beta$
(Appendix~\ref{app:variance}). GGPO's finite-group loss additionally admits a
direct best-of-$K$ reward bound and an empirical generalization bound.
The finite-group variance comparison also yields a reward bound in terms of
AGRO's population variance objective (Eq.~\eqref{eq:agro-reward-bound}).
Appendices~\ref{app:gradient} and~\ref{app:variance} give the gradient and
finite-group weighting comparisons. The experiments compare the practical
consequences under matched settings.

We also acknowledge related unpublished work by Audrey Huang and collaborators
on coverage collapse and reference re-anchoring in policy optimization after
pre-training (private communication with Nan Jiang, October 2026).

\section{From KL regularization to a contrastive loss}
\label{sec:theory}
\label{sec:behavior-profile}
We treat a complete response as one action in a contextual bandit. Fix a
prompt $x$, a finite action space $\mathcal A$, reward $r(x,a)$, temperature
$\beta>0$, and a reference policy $\piref$ with positive support. Define
\begin{equation}\label{eq:main-regularized-value}
 V_\beta(\pi;x)=\E_{a\sim\pi}r(x,a)-\beta\KL(\pi\|\piref).
\end{equation}
Let $V_\beta^*(x)=\max_\pi V_\beta(\pi;x)$ and
$\pi^*(\cdot\mid x)=\arg\max_\pi V_\beta(\pi;x)$ over all policies $\pi$ on $\mathcal A$.
The unique maximizer is the Gibbs policy
\begin{equation}\label{eq:main-gibbs}
 \pi^*(a\mid x)=\frac{\piref(a\mid x)e^{r(x,a)/\beta}}
 {\E_{a'\sim\piref}e^{r(x,a')/\beta}}.
\end{equation}
The Gibbs variational identity \citep{richemond2017variational} gives a
useful way to recognize this optimum. When $\pi_\theta$ is the Gibbs
policy, its adjusted reward $\Rth(x,a)=r(x,a)-\beta\log[\pi_\theta(a\mid x)/\piref(a\mid x)]$
is constant across responses to the same prompt. This suggests learning by
comparing adjusted rewards: if they differ among sampled responses, the
learner has not yet satisfied this optimality condition.

We can test this condition using behavior-policy samples, even when the
behavior policy differs from the learner. To quantify the differences, define
the population contrastive loss
\begin{equation}\label{eq:main-population-loss}
 \Jbeh(\pi_\theta;x)=\beta\log\E_{a\sim\pib}e^{\Rth(x,a)/\beta}
                  -\E_{a\sim\pib}\Rth(x,a).
\end{equation}
The following proposition states its dual interpretation and its relation
to the original KL-regularized objective.

\begin{proposition}[Population duality and regularized optimality]
\label{prop:population-duality}
Suppose $\pib$, $\pi_\theta$, and $\piref$ are positive on $\mathcal A$. Define
$h(z)=e^z-z-1$ and
$q_\theta(a\mid x)\propto\pib(a\mid x)e^{\Rth(x,a)/\beta}$. Then
\begin{equation}\label{eq:main-profiled-duality}
 \Jbeh(\pi_\theta;x)
 =\min_{\nu\in\mathbb R}\beta\E_{a\sim\pib}
 h\!\left(\frac{\Rth(x,a)-\nu}{\beta}\right)
 =\beta\KL(\pib\|q_\theta)\ge0.
\end{equation}
With $C_x=\max_a\pi_\theta(a\mid x)/\pib(a\mid x)$,
\begin{equation}\label{eq:main-population-coverage}
 V_\beta^*(x)-V_\beta(\pi_\theta;x)
 =\beta\KL(\pi_\theta\|\pi^*)\le C_x\Jbeh(\pi_\theta;x).
\end{equation}
Over all positive policies, the minimum contrastive loss is zero and its
unique minimizer is $\pi^*$. When $\pib=\pi_\theta$, $\Jbeh$ equals the exact
KL-regularized optimality gap.
\end{proposition}
The ratio $C_x$ enters the analysis but is not evaluated as an update weight.
\noindent{For long responses, the worst-case sequence ratio
$C_x$ in Eq.~\eqref{eq:main-population-coverage} can be large.
Lemma~\ref{lem:integration-moment-transfer} in
Appendix~\ref{app:integration-proof} gives an alternative bound
that converts finite-group loss into a proximal objective gap using
aggregate tokenwise mismatch. Under its local
reference assumptions, this can reduce the coverage dependence of the
error bound without requiring a smaller $C_x$.}
Appendix~\ref{app:main-population-proof} proves the proposition. Training uses
finite groups, whose expected loss differs from $\Jbeh$ and requires the
separate guarantee in Section~\ref{sec:group-bon-motivation}.

\section{Gibbs-Gap Policy Optimization}
\label{sec:method}

We first describe the finite-group objective, its gradient, and the training
procedure. We then state the reward guarantees for finite groups and periodic
reference refresh.

\subsection{From the population objective to finite groups}

Given $m$ responses $a_{1:m}\sim\pib(\cdot\mid x)$, let $r_i:=r(x,a_i)$
and define their KL-adjusted rewards
$R_i:=r_i-\beta\log[\pi_\theta(a_i\mid x)/\piref(a_i\mid x)]
=\Rth(x,a_i)$. Replacing the behavior expectations in
Eq.~\eqref{eq:main-population-loss} with their empirical averages yields the
single-group loss
\begin{equation}
 \widehat{\mathcal J}_m(\pi_\theta;x,a_{1:m})
 =\beta\log\!\left(\frac1m\sum_{i=1}^m e^{R_i/\beta}\right)
 -\frac1m\sum_{i=1}^mR_i\ge0.
 \label{eq:empirical_group_loss}
\end{equation}
This log-mean-exp gap compares the adjusted rewards within the sampled
group. It is zero if and only if all responses in the group have the same
adjusted reward, matching the optimality condition used to construct the
population objective.

Let $\rho$ be the prompt distribution. For fixed $\pib$ and $\piref$,
the expected finite-group target and its prompt average are
\begin{equation}
\begin{aligned}
H_m(\pi_\theta;\pib,x)
&:=\E_{a_{1:m}\sim\pib^m}\widehat{\mathcal J}_m(\pi_\theta;x,a_{1:m}),\\
L_m(\pi_\theta)
&:=\E_{X\sim\rho}H_m(\pi_\theta;\pib,X).
\end{aligned}
\label{eq:finite-group-population}
\end{equation}
We suppress the fixed $\beta$ and reference in the loss notation, writing
$H_m(\pi_\theta;\pib,x;\piref)$ when specifying the reference is necessary.
Although the group loss replaces expectations by sample averages, its
expectation is not generally the population loss: the logarithm is nonlinear.
For i.i.d. groups,
$H_m\le\Jbeh$, the two share the same zero set for every $m\ge2$, and
$H_m\to\Jbeh$ as the group size grows
(Proposition~\ref{prop:finite-group-target}). Sharing the same zero set means that both losses identify the same
unrestricted target. It does not, however, make their values interchangeable
in a reward bound. We therefore compare finite-group loss directly with a
reward benchmark.

\subsection{Gradient and normalized training loss}
\label{sec:gcpo-gradient}

Minimizing the group loss gives a weighted contrast between responses.
To see how each response contributes to the update, define the within-group tilt
\begin{equation}
 q_i=\frac{e^{R_i/\beta}}
 {\sum_j e^{R_j/\beta}}.
 \label{eq:q-weights}
\end{equation}
Let $U_m$ denote the uniform distribution on the $m$ sample indices. Then
$\widehat{\mathcal J}_m/\beta=\KL(U_m\|q)$, so the finite-group loss penalizes deviations
of the adjusted-reward tilt from uniformity. Holding rewards, actions, and
reference likelihoods fixed during an actor update gives
\begin{equation}
 \nabla_\theta\widehat{\mathcal J}_m
 =-\beta\sum_{i=1}^m\left(q_i-\frac1m\right)
 \nabla_\theta\log\pi_\theta(a_i\mid x).
 \label{eq:gcpo-gradient}
\end{equation}
Under gradient descent, responses with $q_i>1/m$ receive positive update
coefficients, while those with $q_i<1/m$ receive negative coefficients. The behavior policy
determines which responses appear in the group; the update uses no
learner-to-behavior correction weight. The coefficient can also be written
\begin{equation}
 q_i=\softmax_i\!\left(\frac{r_i}{\beta}
 -\log\frac{\pi_\theta(a_i\mid x)}{\piref(a_i\mid x)}\right).
 \label{eq:q-logits}
\end{equation}
The reward term in Eq.~\eqref{eq:q-logits} separates reward classes, while the
learner-to-reference ratio redistributes weight among similarly rewarded
responses. This separates two effects in the update: rewarding successful
responses and adjusting their relative likelihoods with respect to the reference.
With binary rewards and small $\beta$, $q$ concentrates on correct
responses. Incorrect responses receive the same limiting negative coefficient,
while the weights among correct responses follow a softmax of their negative
learner-to-reference log ratios. GGPO therefore gives larger coefficients to
correct responses whose likelihood has grown less relative to the reference,
and the net coefficient summed over correct responses is largest when only
one response in the group is correct. Figure~\ref{fig:group-weight} compares these coefficients analytically.
At the population level, the frozen-behavior gradient
of $\Jbeh$ contrasts the behavior distribution with its induced exponential
tilt. When $\pib=\pi_\theta$, $\Jbeh$ equals the exact regularized gap, while
its frozen-behavior gradient follows the forward-KL/cross-entropy direction
toward $\pi^*$.
Appendix~\ref{app:gradient} gives the population relation, limiting cases,
and coefficient analysis; Appendix~\ref{app:implementation} reports
weight-concentration diagnostics.

For a minibatch of $G_{\mathrm{mb}}$ retained groups, let
$\widehat{\mathcal J}_m^{\,g}:=\widehat{\mathcal J}_m(\pi_\theta;x_g,a^g_{1:m})$. To reuse
GRPO's training hyperparameters, we use
\begin{equation}
 L_{\mathrm{GGPO}}
 =\left(\frac{1}{G_{\mathrm{mb}}}\sum_{g=1}^{G_{\mathrm{mb}}}\widehat{\mathcal J}_m^{\,g}\right)
 \frac{1}{\beta\bar\ell\sqrt{m-1}},
 \label{eq:gcpo_loss_norm}
\end{equation}
where $\bar\ell$ is the minibatch mean response length. This positive
normalization leaves the direction of each realized minibatch gradient
unchanged. It removes the exterior $\beta$ and converts the sequence-level
group average to a token-mean coefficient scale, allowing us to use the
same training hyperparameters across methods. Appendix~\ref{app:gradient}
gives the full calculation.

Training retains only mixed-reward groups, so the minibatches used by the
optimizer differ from the unfiltered groups in the analysis.
For deterministic binary rewards, filtering and positive normalization
preserve the zero set on behavior support for each prompt with
$0<\Pr_{a\sim\pib(\cdot\mid x)}[r(x,a)=1]<1$
(Proposition~\ref{prop:mixed-group-zero-set}). The resulting training
objective averages over the retained prompt distribution.
The population and empirical best-of-$K$ bounds concern the raw, unfiltered loss.

\subsection{Practical GGPO algorithm}
\label{sec:ggpo-algorithm}

Let $N$ be the number of learner minibatches consumed from one frozen behavior
collection. This number is the synchronization delay $N$: the learner
continues updating while the behavior collection remains fixed.
We synchronize the rollout parameters with the learner, collect
$N G_{\mathrm{mb}}$ mixed-reward groups, and partition them into $N$ disjoint
minibatches. Each minibatch receives one optimizer update and is consumed once.
Parameter synchronization does not remove the training--inference mismatch
discussed in Section~\ref{sec:experiments}.

Periodic reference refresh is part of the GGPO training algorithm used in our
main experiments. Every $\Tref$ optimizer updates, we set
$\piref\leftarrow\pi_\theta$ and hold it fixed between refreshes. This resets the learner-to-reference log
ratio in Eq.~\eqref{eq:q-logits}, keeping the anchor local to the current
training stage (Appendix~\ref{app:anchoring}). AGRO shares the same
reference-dependent Gibbs target, so this motivation applies to its reference
updates as well.

Algorithm~\ref{alg:ggpo} gives the original, unclipped training procedure.
The optional token-clipping variant is described below. The reference
schedule uses the global optimizer-update counter, independently of collection
boundaries. If a refresh occurs within a collection, the remaining responses
stay fixed and the next minibatch uses the new reference.
Collection allows at most 32 generation batches, including the initial
batch, and stops with an error if the target cannot be filled
(Appendix~\ref{app:data-evaluation}).

\begin{algorithm}[t]
\caption{Gibbs-Gap Policy Optimization with periodic reference refresh}
\label{alg:ggpo}
\small
\textbf{Input:} Initial policy $\pi_0$, training prompts $\mathcal D$, binary verifier $r$;
 group size $m\ge2$, groups per minibatch $G_{\mathrm{mb}}$, synchronization delay $N$;
 temperature $\beta>0$, reference interval $\Tref$, total optimizer updates $U$.\par
\textbf{Output:} Final learner $\pi_\theta$.\par
\setcounter{algline}{0}
\begin{tabbing}
\hspace{1.8em}\=\hspace{1.2em}\=\hspace{1.2em}\=\hspace{1.2em}\=\kill
\algline Initialize $\pi_\theta\gets\pi_0$, $\piref\gets\pi_0$, and $u\gets0$.\\
\algline \textbf{while} $u<U$ \textbf{do}\\
\algline \>Synchronize rollout parameters with $\pi_\theta$; freeze the resulting $\pib$.\\
\algline \>Initialize the ordered collection $\mathcal C\gets\varnothing$.\\
\algline \>\textbf{while} $|\mathcal C|<N G_{\mathrm{mb}}$ \textbf{do}\\
\algline \>\>Take the next prompt batch from the shuffled dataloader of $\mathcal D$.\\
\algline \>\>Sample $m$ responses per prompt under $\pib$.\\
\algline \>\>Evaluate $r_i=r(x,a_i)$; append complete groups with $0<\sum_i r_i<m$ to $\mathcal C$.\\
\algline \>\textbf{end while}\\
\algline \>Keep the first $N G_{\mathrm{mb}}$ groups; partition into $N$ disjoint minibatches.\\
\algline \>\textbf{for each} minibatch of $G_{\mathrm{mb}}$ groups \textbf{do}\\
\algline \>\>\textbf{if} $u=U$ \textbf{then return} $\pi_\theta$.\\
\algline \>\>For every response, compute masked sequence log probabilities:\\
\algline \>\>\>$s_{g,i}\gets\sum_t M_{g,i,t}\log\pi_\theta(a^g_{i,t}\mid x_g,a^g_{i,<t})$;\\
\algline \>\>\>$s^{\mathrm{ref}}_{g,i}\gets\sum_t M_{g,i,t}\log\piref(a^g_{i,t}\mid x_g,a^g_{i,<t})$.\\
\algline[algline:ggpo-residual] \>\>Form $R_{g,i}\gets r(x_g,a^g_i)-\beta(s_{g,i}-s^{\mathrm{ref}}_{g,i})$.\\
\algline[algline:ggpo-group-loss] \>\>For each group, $\widehat{\mathcal J}_m^{\,g}\gets\beta\log(\frac1m\sum_i e^{R_{g,i}/\beta})-\frac1m\sum_i R_{g,i}$.\\
\algline \>\>Compute $\bar\ell\gets(G_{\mathrm{mb}}m)^{-1}\sum_{g,i,t}M_{g,i,t}$.\\
\algline[algline:ggpo-loss] \>\>Set $L\gets\sum_g\widehat{\mathcal J}_m^{\,g}/(G_{\mathrm{mb}}\beta\bar\ell\sqrt{m-1})$.\\
\algline[algline:ggpo-update] \>\>Take one AdamW update using $\nabla_\theta L$; set $u\gets u+1$.\\
\algline \>\>\textbf{if} $u\bmod\Tref=0$ \textbf{then}\\
\algline \>\>\>Copy the current learner parameters to $\piref$.\\
\algline \>\>\textbf{end if}\\
\algline \>\textbf{end for}\\
\algline \textbf{end while}\\
\algline \textbf{return} $\pi_\theta$.
\end{tabbing}
\noindent{Optional token clipping: replace the adjusted reward
in line~\ref{algline:ggpo-residual} by Eq.~\eqref{eq:integration-optional-clipping}
and the loss in lines~\ref{algline:ggpo-group-loss}--\ref{algline:ggpo-loss}
by the detached-coefficient surrogate in Eq.~\eqref{eq:clip-surrogate}.}
\end{algorithm}

Here $M_{g,i,t}$ is the valid-response-token mask. Rewards, sampled tokens,
reference probabilities, and masks are held fixed in differentiation; gradients
flow directly through the learner probabilities in the log-mean-exp loss.
The weights $q_i$ explain this gradient but need not be constructed as a separate
policy-gradient surrogate. We evaluate the log-mean-exp stably, and compute
$\bar\ell$ across the complete learner minibatch, including all devices and
microbatches.

\begingroup
\paragraph{Optional token clipping.}
The main experiments use the original update in Algorithm~\ref{alg:ggpo}.
The clipping experiment in Section~\ref{sec:token-clipping-experiment}
replaces the adjusted reward in line~\ref{algline:ggpo-residual} by
\begin{equation}\label{eq:integration-optional-clipping}
 \bar R_{g,i}=r(x_g,a_i^g)-\beta\sum_t M_{g,i,t}
 \operatorname{clip}\!\left(
 \log\frac{\pi_\theta(a^g_{i,t}\mid x_g,a^g_{i,<t})}
 {\piref(a^g_{i,t}\mid x_g,a^g_{i,<t})},-c_{\mathrm{clip}},c_{\mathrm{clip}}\right).
\end{equation}
Here $c_{\mathrm{clip}}>0$ is the token log-ratio clipping radius; the experiment uses
$c_{\mathrm{clip}}=0.05$. These adjusted rewards determine detached response coefficients
for the surrogate in Eq.~\eqref{eq:clip-surrogate}, replacing the loss
construction in lines~\ref{algline:ggpo-group-loss}--\ref{algline:ggpo-loss}.
Line~\ref{algline:ggpo-update} then updates with this surrogate.
The coefficients are clipped through their adjusted rewards, but the
surrogate is differentiated using unclipped learner log probabilities.
Consequently, saturated tokens retain their score gradients. Direct differentiation through a
clipped log-mean-exp loss would give a different update. The appendix
specifies the coefficients and normalization.
\endgroup

The default setting uses $G_{\mathrm{mb}}=32$, $m=8$, $N=8$, $\beta=0.01$, and
$\Tref=160$. Appendices~\ref{app:data-evaluation} and~\ref{app:implementation}
specify the refill budget, masking, EOS, padding, and truncation details.
Fixed-pool offline training replaces behavior collection with successive
minibatches from a precollected pool; each is still consumed once, and reference
refresh follows the same optimizer-update schedule.

\subsection{\texorpdfstring{Best-of-$K$}{Best-of-K} reward and the finite-group loss}
\label{sec:group-bon-motivation}
For each prompt, best-of-$K$ sampling draws $K$ independent responses from
a policy $\pi$ and returns the one with the highest verifier reward. Its expected
reward, averaged over the prompt distribution $\rho$, is
\begin{equation}
 J_K(\pi):=\E_{X\sim\rho}\E_{A_{1:K}\sim\pi(\cdot\mid X)^K}
             \left[\max_{i\le K}r(X,A_i)\right].
 \label{eq:best-of-k-reward}
\end{equation}
The comparison of interest is between this selection procedure and drawing
a single response from the learner. A single response from a policy $\pi$
has expected reward
$J(\pi)=\E_{X\sim\rho,A\sim\pi(\cdot\mid X)}r(X,A)$; in particular,
$J_1(\pi)=J(\pi)$. The selection budget $K$ is distinct from the
synchronization delay $N$ (Section~\ref{sec:ggpo-algorithm}).
For rewards in $[0,1]$,
Proposition~\ref{prop:gcpo-minimizer} (Appendix~\ref{app:group-bon-proof})
gives the bound
\begin{equation}
 J_K(\pib)-J(\pi_\theta)
 \le\kappa_{K,m}\{\beta(\log m+\Gamma)+C L_m(\pi_\theta)\}.
 \label{eq:main-finite-group-bon}
\end{equation}
Here $\Gamma$ bounds the spread (maximum minus minimum) of
$\log(\piref/\pib)$ across responses to a fixed prompt and vanishes when
$\piref=\pib$, while $C$ uniformly bounds $\pi_\theta/\pib$ over the
hypothesis class. For the individual candidate $\pi_\theta=\pib$, the
ratio is one; a class-wide bound may still exceed one.  The tail-comparison multiplier $\kappa_{K,m}$, defined in
Appendix~\ref{app:group-bon-proof}, satisfies
$\kappa_{K,m}\le\max\{1,K/(m-1)\}$. In particular,
$m\ge K+1$ gives $\kappa_{K,m}=1$.
Theorem~\ref{thm:group-bon} in Appendix~\ref{app:group-generalization}
converts it into an empirical bound based on raw loss evaluated on
i.i.d. unfiltered groups.

\subsection{Reference refresh guarantees}
\label{sec:reference-refresh}

A fixed reference defines a single KL-regularized Gibbs target in
Eq.~\eqref{eq:main-gibbs}, independent of group size. Reference refresh changes
this target by using the current learner as the next reference. Exact
target composition makes this effect explicit: successive refreshes increase
the effective reward tilt.

\paragraph{Refresh composes temperatures.}
Let $\pi_0$ be the initial policy and $\pi_t$ the output of reference-update round $t$. If each round exactly attains its Gibbs target relative to the previous
output, Eq.~\eqref{eq:main-gibbs} gives
\begin{equation}
 \pi_T(a\mid x)\propto\pi_0(a\mid x)e^{Tr(x,a)/\beta}.
 \label{eq:refresh-temperature}
\end{equation}
Thus $T$ exact rounds at temperature $\beta$ implement one Gibbs update
at temperature $\beta/T$. The following result extends the best-of-$K$
comparison to reference refresh when the round losses are nonzero.
\begin{proposition}[Reference refresh with imperfect rounds, informal]
\label{prop:refresh-loss-informal}
For rewards in $[0,1]$, consider $T$ reference-update rounds with group
size $m\ge2$. Assume positive policies, uniform behavior coverage of
both the round outputs and the best-of-$K$ comparator, and bounded
regularized residual range, as specified in
Proposition~\ref{prop:refresh-loss}, Appendix~\ref{app:refresh-bound}.
Let $\bar\pi_T$ be the uniform mixture of the round outputs and
$\bar L_T$ their average population raw GGPO loss. Then
\[
 J_K(\pi_0)-J(\bar\pi_T)
 \le \frac{\beta\log K}{T}
       +O\!\left(\sqrt{\beta\bar L_T}\right),
\]
where the hidden constant depends only on the coverage and
residual-range bounds and the group size.
\end{proposition}
The mixture $\bar\pi_T$ selects a round uniformly and samples one response
from its output policy. The loss $\bar L_T$ averages the population group
loss in Eq.~\eqref{eq:finite-group-population} across rounds, using behavior
$\pibt{t}$ and reference $\pi_{t-1}$ in round $t$. Here $T$ counts reference
updates, while $K$ is the number of initial-policy responses used by the
best-of-$K$ comparator. More refresh rounds reduce the initial-reference
penalty $\beta\log K/T$, while accurate fitting controls the remaining error.
The regularized residual is reward minus $\beta$ times the log ratio of
the round output to its reference. Its range is the spread of these values
across responses to a fixed prompt; bounding that range and the coverage
ratios controls the constant hidden in the loss term.
Appendix~\ref{app:refresh-bound} gives the precise statement and proof.

The distinction between the mixture and the final policy matters for
deployment. At zero loss, exact composition gives the same
$\beta\log K/T$ bound for the final output $\pi_T$. For imperfect updates,
the next section gives a final-policy guarantee under binary rewards and
aggregate tokenwise mismatch control.

\begingroup
\section{Final-policy reward under reference refresh}
\label{sec:integration-convergence}
Reference refresh changes the Gibbs target as training progresses. With binary
rewards, each exact update multiplies the odds of a correct response by
$e^{1/\beta}$. Thus a correct response that already has positive probability under
the initial policy can receive progressively more probability over successive
rounds. In practice, each round only approximates its Gibbs target, so the
question is whether these fitting errors allow the improvement to continue.
The result below connects the accuracy of individual updates to the final
policy's single-response reward.

This requires relating two distributions: GGPO's loss is measured on responses
from the behavior policy, while reward is evaluated under the returned learner.
The bounds in Eqs.~\eqref{eq:main-population-coverage}
and~\eqref{eq:main-finite-group-bon} use worst-case sequence coverage
to make this connection.
Here we use moments of the learner-to-behavior response likelihood ratio
to relate the two distributions. Small conditional token log ratios can
keep these moments moderate even when the worst-case sequence ratio is
large. The guarantee accounts for imperfect fitting and finite data.

\paragraph{Analyzed procedure.}
Start from $\pi_0$. In round $t$, freeze reference $\pi_{t-1}$, behavior
$\pibt{t}$, and candidate class $\Pi_t$ before sampling. Draw $n$ fresh,
independent prompts from $\rho$. With group size $m\ge2$, sample $m$
independent responses under $\pibt{t}$ for each prompt draw, without filtering.
Let $\widehat L_t(\pi)$ be the average raw GGPO loss in
Eq.~\eqref{eq:empirical_group_loss} over these $n$ groups. Fit this loss
over $\Pi_t$, assuming the returned policy $\pi_t\in\Pi_t$ satisfies
\[
 \widehat L_t(\pi_t)-\inf_{\pi\in\Pi_t}\widehat L_t(\pi)\le\xi.
\]
Here $\xi\ge0$ bounds the achieved empirical optimization error in each
round, as specified in Eq.~\eqref{eq:integration-error-assumptions}.
Use $\pi_t$ as the next reference. Data may be reused within a round.
Let $T$ denote the total number of reference rounds.

\begin{theorem}[Final-policy reward guarantee, informal]
\label{thm:integration-convergence-informal}
Consider the reference-update procedure above with binary rewards and
responses of at most $H$ tokens. Assume common positive support, a
sufficiently large temperature $\beta$ of order $H$, and small
conditional token log ratios relative to each reference.
Let $\pi_\star$ be a comparator whose success probability at each prompt
is at most $C_\star$ times that of the initial policy $\pi_0$. Assume a
uniform bound on a moment of order at least two of the learner-to-behavior
response ratio, or on the ratio itself. Assume the representation,
achieved optimization accuracy, and loss-cover conditions of
Theorem~\ref{thm:integration-convergence}.

For any target error $0<\varepsilon<1$ and failure probability $0<\delta<1$,
suitable numbers of reference rounds and fresh prompt draws give, with
probability at least $1-\delta$ over the sampled prompts and responses,
\[
 J(\pi_T)\ge J(\pi_\star)-\varepsilon-\mathcal E_{\rm fit},
 \qquad J(\pi)=\E_{X\sim\rho,y\sim\pi(\cdot\mid X)}r(X,y).
\]
Here $\mathcal E_{\rm fit}$ is the fitting residual in
Eq.~\eqref{eq:integration-final} or~\eqref{eq:integration-uniform-reward},
according to the coverage condition. For fixed moment order and coverage
coefficient, sufficient prompt and response counts are polynomial in $H$,
$C_\star$, and $1/\varepsilon$, with logarithmic dependence on the
loss-cover size and $1/\delta$.
\end{theorem}
Theorem~\ref{thm:integration-convergence} in
Appendix~\ref{sec:integration-formal-setup} gives the full finite-moment
statement, including the assumptions, sample budget, and fitting residual.
Remark~\ref{rem:integration-bounded-coverage} gives the bounded-coverage case.

\paragraph{Interpretation of the theorem.}
The guarantee concerns the expected reward of a single response from the
final policy $\pi_T$. This expected reward is at least the comparator's
reward minus $\varepsilon+\mathcal E_{\rm fit}$. The coverage constant $C_\star$ describes how much improvement is
requested relative to initialization. For example, oracle best-of-$K$ sampling
from $\pi_0$ satisfies the coverage condition with $C_\star=K$, because the
probability of finding a correct response in $K$ draws is at most $K$ times
the single-draw success probability. Thus the expected reward of one response from the trained
policy can approach that of selecting the best of $K$ responses from
the initial policy. The training cost is accounted for separately.

The target error $\varepsilon$ is controlled by the number of reference
rounds and fresh prompt draws. The residual $\mathcal E_{\rm fit}$ captures
representation error, achieved empirical optimization error, and loss-cover
approximation error, whose combined size is $\varepsilon_A$ in
Eq.~\eqref{eq:integration-approximation}. Increasing the sample size alone
does not remove persistent fitting error. For fixed temperature and
coverage, the bound multiplies fitting error by a factor quadratic in $T$.
Longer runs therefore require more accurate fits to keep the same reward bound. Under bounded coverage, the residual is linear in $\varepsilon_A$. At finite order $p$, it also includes a term
proportional to $\varepsilon_A^{1-1/p}$, whose coefficient decreases as
the learner-to-behavior coverage moment approaches one. The residual
vanishes when the three fitting errors are zero.

The sample bound counts $n$ independent prompt draws per round, each with
$m$ responses, giving $Tn$ prompt groups and $mTn$ scored responses.
Prompts may repeat under independent sampling, but additional groups on
a fixed prompt pool do not supply fresh prompt draws.
Loss-class complexity enters because the learner is chosen using the data:
the loss cover approximates all candidate policies' group-loss functions
uniformly, and its size measures how many functions are needed at the
chosen accuracy. Under these sampling and policy constraints, the guarantee
applies to any solver attaining the stated empirical accuracy; it does not
establish convergence of AdamW.
Under the token conditions in Appendix~\ref{sec:integration-formal-setup},
the moment coefficient is
bounded by $c_\nu=\exp((1+\nu^{-1})V_{\log})$ for $0<\nu\le1$, where
$V_{\log}$ bounds the sum of conditional token log-ratio second moments
along each response path. The statistical group budget has accuracy dependence
$1/\varepsilon+c_\nu/\varepsilon^{1+\nu}$, apart from round and problem
factors. Bounded full-response coverage gives the direct $1/\varepsilon$
dependence. A smaller $\nu$ reduces the
accuracy exponent but can increase $c_\nu$ and requires tighter token
ratios. The benefit therefore depends on both the token mismatch and the
chosen moment order.

\paragraph{Behavior lag and practical clipping.}
If behavior is a stored policy at most $d$ reference rounds old, local
reference control gives $|\log(\pi_{t,h}/\pi_{\mathrm{beh},t,h})|\le d/\beta$
and allows $V_{\log}=Hd^2/\beta^2$. This supplies the stated log-moment
coefficient when $d/\beta\le\min\{\log2,3\nu/2\}$. The same lag bound
also gives direct sequence coverage $\mathcal C_\infty\le e^{Hd/\beta}$,
which can be used in the bounded-coverage branch without this token cap.
In particular, sampling from the current reference has $d=1$ and
$\mathcal C_\infty\le e^{H/\beta}<\sqrt e$ under the stated temperature
condition, so its group budget has linear accuracy dependence apart from
round and problem factors.
The lag counts reference rounds rather than the experimental synchronization
delay $N$. Appendix~\ref{sec:integration-local-relaxation} gives the relaxed
constraint and an alternative bound from token spans.

The theorem's small token log ratios keep every candidate policy close
to its frozen reference at each response prefix. The precise constraint
in Eq.~\eqref{eq:integration-local} bounds the actual conditional
probabilities and permits the exact Gibbs target, as shown in
Appendix~\ref{sec:integration-formal-setup}.
This motivates practical token clipping, which instead limits log-ratio
contributions to response coefficients, with some bias, and does not
enforce the same policy constraint. Section~\ref{sec:token-clipping-experiment}
evaluates this stabilization, and Appendix~\ref{app:token-clipping}
specifies its implementation. The theorem analyzes constrained fitting
of the raw loss.

\endgroup
\section{Experiments}
\label{sec:experiments}

We compare GRPO, AGRO, and GGPO in default training ($N=8$), including
pass@$K$ and sampling cost, then examine reference refresh, controlled
policy staleness ($N=8$ to $128$), offline training on fixed data, and
optional token clipping.
These settings examine how the methods respond to behavior--learner mismatch.
Training--inference mismatch remains present even immediately after
the rollout parameters are synchronized with the learner.
Appendices~\ref{app:additional-experiments}
and~\ref{app:experiments} give additional results and implementation details.

\subsection{Experimental setup}
\label{sec:setup}

\textbf{Models, data, and evaluation.}
We evaluate mathematical reasoning with \qwenone{} and \qwenfour{}
\citep{yang2025qwen3}. Training uses ALP-Math, which merges a
difficulty-filtered OpenR1-derived corpus \citeyearpar{weqweasdasOpenR1Scores} with the math split of Guru-RL
\citep{cheng2025guru}. Following standard
RLVR protocols \citep{deepseekai2025deepseekr1}, rewards are binary and verified
by symbolic equivalence with Math-Verify \citep{kydlicek2025mathverify}. Under
the inherited training configuration described below, progress is monitored on MATH-500
\citep{hendrycks2021math,lightman2024verify}. Monitoring results were not used
for hyperparameter or checkpoint selection. For the final comparison, we
report mean@32, the
mean accuracy over 32 newly sampled responses per problem, on MATH-500,
\href{https://huggingface.co/datasets/math-ai/aime24}{AIME 2024},
\href{https://huggingface.co/datasets/math-ai/aime25}{AIME 2025},
OlympiadBench \citep{he2024olympiadbench}, and Minerva Math
\citep{lewkowycz2022solving}. Mean@32 estimates single-response mean reward
under the evaluation rollout policy. Training curves also report MATH-500
best@32, the fraction of problems with at least one correct response among
the 32 samples, equivalent to pass@32 for binary rewards
(Appendix~\ref{app:data-evaluation}).
The problem-weighted average over the five benchmarks gives
every evaluated problem equal weight (counts in
Appendix~\ref{app:data-evaluation}).

\textbf{Seeds and checkpoint reporting.}
All five-benchmark tables report final checkpoints; default and controlled
staleness runs use 8{,}000 optimizer updates. Each checkpoint is evaluated
once with 32 responses per problem. We report means and sample standard
deviations over training seeds 42, 1234, and 2027 for the \qwenone{} default
comparison (Table~\ref{tab:main-results}, Figure~\ref{fig:main-dynamics}),
offline accuracy (Table~\ref{tab:offline-results}, Figure~\ref{fig:offline},
left), and $N=128$ staleness, including GRPO+TIS
(Table~\ref{tab:lag-results-full}). Curves summarize each logged or evaluated
update. Other conditions and offline diagnostics use seed 42. When a
comparison reuses default results, those results retain their three-seed
means, as noted in table captions.

\textbf{Baselines.}
AGRO \citep{tang2025agro} is our direct objective-level comparator: it
shares GGPO's KL-adjusted reward, Gibbs target, and independently controlled
reference schedule. OAPL \citep{ritter2026offpolicy} ties its KL anchor and
value target to the rollout policy, coupling target refresh to behavior
synchronization. Matching our independent reference schedule would therefore
require modifying its target estimator. We discuss OAPL as related work and
focus the controlled objective comparison on AGRO, with GRPO
\citep{shao2024deepseekmath} and GRPO+TIS providing policy-gradient baselines.
In the default runs, GRPO uses token-level learner-to-behavior PPO
ratios with the log probabilities recorded during rollout as the denominator
and no additional off-policy correction.
We include GRPO+TIS \citep{zhong2026diagnosing} under the largest
synchronization delay, $N=128$, and in fixed-data offline training to evaluate
importance-corrected policy-gradient updates in these settings. It uses
token-level truncated importance sampling with a cap of 2
(Appendix~\ref{sec:grpo-tis}).

\textbf{Hyperparameters.}
By default, we set $\beta=0.01$ for all methods, following the offline experiments of
AGRO \citep{tang2025agro}, and draw $m=8$ responses per prompt.
Appendix~\ref{sec:group-weight-ablation} also evaluates $\beta=0.001$, used
in AGRO's on-policy experiments. The methods
share the same batch geometry and prompt order, and reuse the learning rate
and optimizer settings of our GRPO baseline. They also
use the same dynamic difficulty-filtering rule \citep{yu2025dapo,bae2026filtering}.
By default, our GRPO implementation keeps its initial reference fixed, while
AGRO and GGPO refresh the reference every $\Tref=160$
optimizer updates, matching our checkpointing cadence
(Section~\ref{sec:ggpo-algorithm}).
Appendices~\ref{sec:reference-regularization} and~\ref{sec:group-weight-ablation}
compare reference cadences and group weights. The complete reference, batch, decoding, refill, and
optimizer settings are given in Appendix~\ref{app:experiments}.

\textbf{Training budget.}
Method comparisons match optimizer update counts and retained responses per
update. Mixed-reward filtering and refill can nevertheless require different
numbers of generated responses across methods. Figure~\ref{fig:sampling-budget}
reports this sampling demand. Response counts do not equalize generated
tokens or GPU-hours, since those costs also depend on response length and
execution costs.

\subsection{Default training setting}
\label{sec:default-training}

Under default training ($N=8$; Section~\ref{sec:ggpo-algorithm}), all
three methods' token-mean K3 stays on the order of $10^{-3}$ throughout
training (Figure~\ref{fig:main-dynamics}, panel 1).
Table~\ref{tab:main-results} compares their final checkpoints at both model sizes.

\begin{table}[t]
  \centering
  \footnotesize
  \setlength{\tabcolsep}{3pt}
  \caption{\textbf{Comparison under the default training setting ($N=8$).}
  Final-checkpoint mean@32 accuracy after 8{,}000 optimizer updates on five mathematical-reasoning
  benchmarks (\%). The final column is the problem-weighted average over the
  five benchmarks, which gives every evaluated problem equal weight. For
  \qwenone{}, values are the mean $\pm$ sample standard deviation across
  training seeds 42, 1234, and 2027. For Avg., these statistics are computed
  across the per-seed problem-weighted averages. \qwenfour{} uses seed 42.
  Bold marks the best value in each model-size block.}
  \label{tab:main-results}
  \resizebox{\textwidth}{!}{
  \begin{tabular}{@{}lcccccc@{}}
    \toprule
    Model and method & MATH-500 & AIME 2024 & AIME 2025 & OlympiadBench & Minerva Math & Avg. \\
    \midrule
    \qwenone{} + GRPO & 68.86 $\pm$ 0.32 & 9.03 $\pm$ 0.90 & 5.14 $\pm$ 0.63 & 32.13 $\pm$ 0.27 & 28.33 $\pm$ 0.40 & 42.64 $\pm$ 0.18 \\
    \quad + AGRO & 70.68 $\pm$ 0.33 & 9.27 $\pm$ 1.15 & 6.56 $\pm$ 0.81 & 35.73 $\pm$ 0.30 & 29.30 $\pm$ 0.44 & 45.07 $\pm$ 0.19 \\
    \quad + GGPO & \textbf{72.43 $\pm$ 0.29} & \textbf{11.11 $\pm$ 0.89} & \textbf{8.30 $\pm$ 0.68} & \textbf{37.11 $\pm$ 0.25} & \textbf{30.64 $\pm$ 0.36} & \textbf{46.58 $\pm$ 0.16} \\
    \midrule
    \qwenfour{} + GRPO & 80.71 & \textbf{16.88} & 11.88 & 46.63 & 40.29 & 55.51 \\
    \quad + AGRO & 80.55 & \textbf{16.88} & 12.40 & 47.21 & 40.71 & 55.81 \\
    \quad + GGPO & \textbf{81.83} & \textbf{16.88} & \textbf{14.38} & \textbf{48.33} & \textbf{41.85} & \textbf{56.98} \\
    \bottomrule
  \end{tabular}
  }
\end{table}

\paragraph{Results.}
GGPO achieves the highest problem-weighted average at both model sizes
(Table~\ref{tab:main-results}). On \qwenone{}, it exceeds AGRO by 1.5 points
and GRPO by 3.9 points. The single-seed \qwenfour{} comparison shows the same
ordering of problem-weighted accuracy at the larger model size.
Figure~\ref{fig:main-dynamics} shows that GGPO initially improves more slowly
on MATH-500 mean@32, but overtakes both baselines later in training. GGPO
maintains higher policy entropy, while its best@32 continues to improve
later in training.
Appendix~\ref{sec:difficulty} reports performance stratified by the base
model's per-problem success rate.

\begin{figure}[htb]
  \centering
  \includegraphics[width=\textwidth]{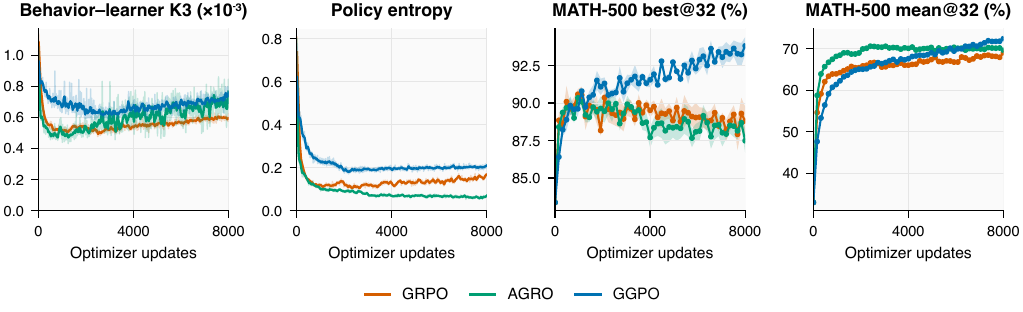}
  \captionsetup{font=small}
  \caption{\textbf{Training dynamics under the default training setting
  ($N=8$) on \qwenone{}.} From left to right: the token-mean K3
  diagnostic of behavior--learner mismatch
  \citep{schulman2020klapprox} ($\times10^{-3}$), policy entropy, MATH-500
  best@32 (\%), and MATH-500 mean@32 (\%), each against optimizer updates
  (GRPO orange, AGRO green, GGPO blue). K3 is computed on mixed-reward-filtered
  training minibatches. In the first two panels, faint lines
  are the logged series (one point per eight updates) and solid lines a
  nine-point moving average. At each reported update, curves summarize seeds
  42, 1234, and 2027 by their mean, with shaded bands spanning one sample standard
  deviation across seeds; MATH-500 is evaluated every 160 optimizer updates.}
  \label{fig:main-dynamics}
\end{figure}

\begin{samepage}
\paragraph{Pass@$K$ and best-of-$K$ comparison.}
On \qwenone{}, GGPO leads in problem-weighted five-benchmark pass@$K$
for every evaluated $K\in\{1,4,8,16,32\}$ with training seed 42
(Table~\ref{tab:passk-results}).
At $K=32$, it reaches $73.90\%$, compared with $69.39\%$ for the strongest
baseline and $62.68\%$ for the base model.
Its seed-42 pass@1 ($46.60\%$) exceeds the base model's pass@4 ($38.57\%$)
and approaches its pass@8 ($47.43\%$). Thus, in this comparison, a single
GGPO response succeeds more often than best-of-four sampling from the base
model, illustrating the comparison in Section~\ref{sec:integration-convergence}.
Pass@1 equals mean@32 under our estimator; the $46.60\%$ value uses seed 42,
whereas Table~\ref{tab:main-results} reports the three-seed mean of $46.58\%$.
Appendix~\ref{sec:passk} gives the full pass@$K$ results.
\par
\end{samepage}

\paragraph{Sampling cost.}
GGPO reaches higher final MATH-500 mean@32 with fewer generated training
responses than either baseline in the reported default-setting comparison
(Figure~\ref{fig:sampling-budget}).
This complements the comparison at matched optimizer updates and retained
responses; token and computation costs also depend on response lengths
and execution costs (Appendix~\ref{sec:sampling-budget}).

\begin{figure}[!htb]
  \centering
  \includegraphics[width=0.80\textwidth]{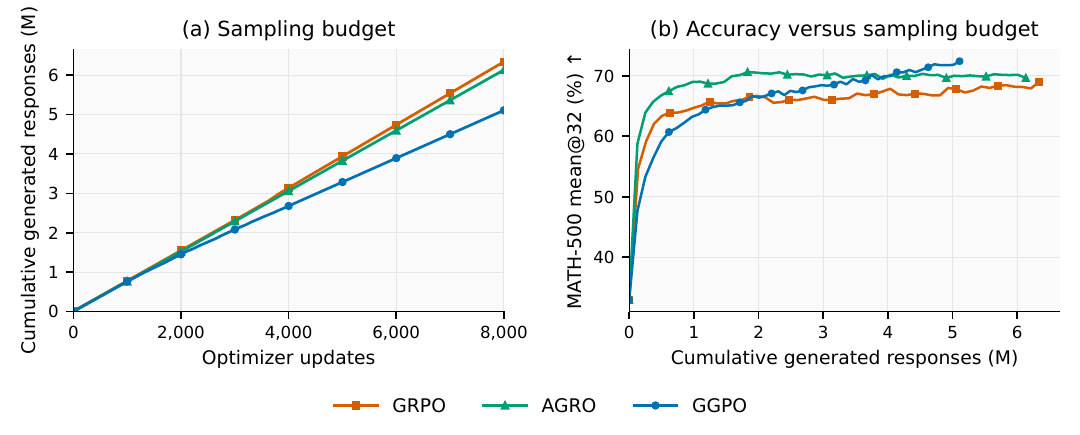}
  \caption{\textbf{Sampling budget under the default training setting on
  \qwenone{}.} Left: cumulative generated training responses, in millions,
  against optimizer updates. Right: MATH-500 mean@32 accuracy against the same
  cumulative response count. Each curve ends at 8{,}000 optimizer updates
  (GRPO orange squares, AGRO green triangles, GGPO blue circles).
  Counts include responses discarded by mixed-reward filtering and
  generated during refill. GGPO finishes with the highest accuracy and
  the lowest response count.}
  \label{fig:sampling-budget}
\end{figure}

\subsection{Reference refresh}
\label{sec:reference-refresh-main}

The default comparison uses a fixed reference for GRPO and refresh every
$\Tref=160$ updates for AGRO and GGPO.
Table~\ref{tab:reference-refresh-summary} compares all three methods across
these reference schedules at $\beta=0.01$.
GGPO's problem-weighted average is $32.96\%$ with a fixed reference
(seed 42), compared with $46.58\%$ with refresh every 160 updates
(three-seed mean); refresh every 320 updates reaches $44.96\%$ (seed 42). AGRO also benefits from refresh, while refreshing GRPO
does not close the accuracy gap.
GGPO's weaker fixed-reference result shows that reference refresh is an
important part of the complete method.
Appendix~\ref{sec:reference-regularization} gives per-benchmark results,
the $\beta=0.1$ conditions, and reference-drift diagnostics.

\begin{table}[htb]
  \centering
  \small
  \setlength{\tabcolsep}{12pt}
  \caption{\textbf{Reference refresh on \qwenone{}.}
  Problem-weighted five-benchmark mean@32 (\%) after 8{,}000 updates,
  with $\beta=0.01$. A dagger marks default settings reported as means
  over seeds 42, 1234, and 2027; other cells use seed 42.
  Table~\ref{tab:reference-refresh} gives the full results.}
  \label{tab:reference-refresh-summary}
  \begin{tabular}{@{}lccc@{}}
    \toprule
    Method & Fixed reference & $\Tref=160$ & $\Tref=320$ \\
    \midrule
    GRPO & $42.64^\dagger$ & 42.40 & 42.48 \\
    AGRO & 41.48 & $45.07^\dagger$ & 43.86 \\
    GGPO & 32.96 & $46.58^\dagger$ & 44.96 \\
    \bottomrule
  \end{tabular}
\end{table}

\subsection{Controlled policy staleness}
\label{sec:async}

With synchronization delay $N$, each frozen behavior collection of $32N$
prompts is split into $N$ disjoint 32-prompt minibatches, each used for one
optimizer update before the behavior policy is synchronized. Sample age is
therefore $0,\ldots,N-1$
optimizer updates. Training--inference mismatch remains present as a fixed background.
Following controlled asynchronous RLHF evaluations
\citep{noukhovitch2025asynchronous,huang2026stable}, we vary $N$ over the
staleness ladder $N\in\{8,16,32,64,128\}$ with method-specific
hyperparameters fixed; the $N=8$
condition reuses the \qwenone{} default-setting run. This varies the age
of the sampled responses without changing the method-specific hyperparameters.
Reference handling is
unchanged from the default setting at every delay;
Appendix~\ref{app:implementation} states how a refresh boundary inside a
collection is handled.

Figure~\ref{fig:async} summarizes performance with the problem-weighted
average of Section~\ref{sec:setup}; Appendix~\ref{sec:staleness-details} gives
the per-benchmark table and the training-time diagnostics against $N$.

\begin{figure}[htb]
  \centering
  \includegraphics[width=0.65\linewidth]{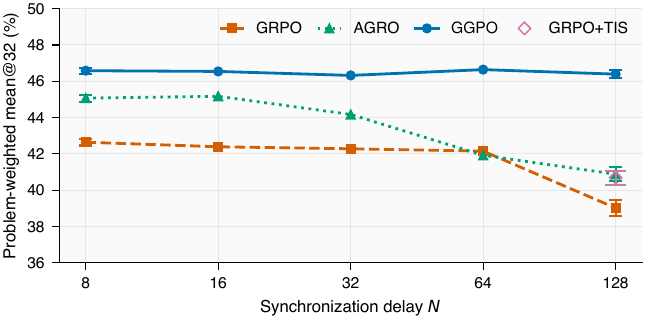}
  \captionsetup{font=small}
  \caption{\textbf{Controlled policy staleness on \qwenone{}.} Problem-weighted
  final-checkpoint mean@32 accuracy over the five benchmarks after 8{,}000
  optimizer updates, against synchronization delay $N$ (equally spaced
  positions; values in Table~\ref{tab:lag-results-full}). GRPO+TIS is evaluated
  at $N=128$. Results at $N=8,128$
  are three-seed means, with error bars showing $\pm1$ sample standard
  deviation; $N=16,32,64$ use seed 42.}
  \label{fig:async}
\end{figure}

\paragraph{Results.}
GGPO maintains stable performance as policy staleness increases from $N=8$
to $N=128$, with problem-weighted accuracy of $46.32\%$--$46.64\%$.
AGRO loses 4.2 points from $N=8$ to $N=128$, while
GRPO is essentially unchanged through $N=64$ before declining at $N=128$.
At the largest synchronization delay, $N=128$, token-level TIS raises GRPO's
problem-weighted accuracy from $39.02\%$ to $40.69\%$, while GGPO reaches
$46.39\%$ under the same delay
(Figure~\ref{fig:async} and Table~\ref{tab:lag-results-full}).
The training-time diagnostics show increasing behavior--learner mismatch for
GGPO from $N=8$ to $N=128$: its K3 rises and normalized effective sample size
(ESS) falls, while both diagnostics remain nearly unchanged for uncorrected GRPO
(Figure~\ref{fig:staleness-mismatch}).

GRPO's nearly constant K3 is consistent with the PPO surrogate suppressing
updates that would move learner-to-behavior ratios farther beyond its
clipping interval. This clipping changes the objective but imposes no
hard constraint on the policy ratios. GGPO uses no learner-to-behavior importance
weights (Section~\ref{sec:gcpo-gradient}); its rising K3 and stable accuracy
are consistent with learning as the learner departs from the collection's
behavior policy. AGRO also avoids learner-to-behavior importance weights,
but its accuracy declines at larger delays.
These comparisons still allow the behavior policy to catch up with the
learner every $N$ updates. The offline setting in
Section~\ref{sec:offline-main} removes this resynchronization, allowing
us to examine learning from a single fixed behavior pool.

\subsection{Offline training on fixed data}
\label{sec:offline-main}

In the offline setting, GRPO, GRPO+TIS, AGRO, and GGPO perform 400 optimizer
updates on the same fixed rollout pool within each seed without generating
new responses. This setting tests learning from fixed data as
behavior--learner mismatch accumulates without behavior resynchronization.
Appendix~\ref{app:protocols} gives the protocol and
Table~\ref{tab:offline-results} the accuracy at every evaluated checkpoint.

\begin{figure}[htb]
  \centering
  \includegraphics[width=\textwidth]{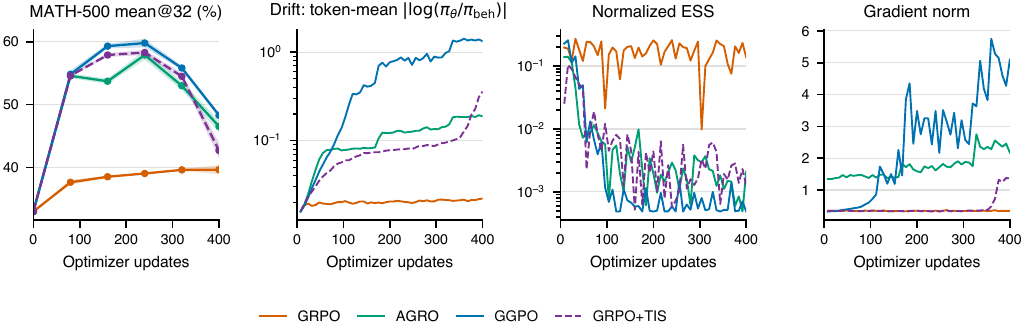}
  \captionsetup{font=small}
  \caption{\textbf{Offline training on \qwenone{}.} Within each seed, all
  methods use the same fixed rollout pool, batch geometry, prompt order, and
  400-update budget (GRPO orange, GRPO+TIS purple dashed, AGRO green,
  GGPO blue). From left: MATH-500 mean@32 (\%), with the reported mean
  accuracy and bands spanning one reported sample standard deviation
  (Table~\ref{tab:offline-results}); then, for seed 42,
  the token-mean absolute learner-to-behavior log ratio
  (log scale), the normalized ESS of the diagnostic sequence-level
  learner-to-behavior importance weights over each logged block of 2{,}048
  responses (log scale), and the actor gradient norm
  (Appendix~\ref{app:data-evaluation}). The mismatch diagnostics use
  mixed-reward-filtered training responses.}
  \label{fig:offline}
\end{figure}

\paragraph{Results.}
GGPO improves MATH-500 accuracy from $32.96\%$ to $59.79\%$ within
240 optimizer updates, leading the baselines at every evaluated checkpoint
from 160 updates onward
(Figure~\ref{fig:offline}; Table~\ref{tab:offline-results}).
GGPO, AGRO, and GRPO+TIS all peak at 240 updates and decline thereafter,
alongside increasing divergence from the fixed behavior policy.

The diagnostic panels show increasing drift and declining normalized ESS
for GRPO+TIS, AGRO, and GGPO, while uncorrected GRPO's drift stays near
the level set by the initial training--inference mismatch. For GGPO,
normalized ESS is already below $10^{-3}$ by 160 updates. This diagnostic
indicates that sequence-level importance weights on the filtered training
responses are highly concentrated at that point.

\subsection{Token clipping}
\label{sec:token-clipping-experiment}
The local reference constraint in Section~\ref{sec:integration-convergence}
motivates limiting individual token log-ratio contributions in practice.
In the default \qwenone{} setting, the optional clipped variant of
Algorithm~\ref{alg:ggpo}, with radius $c_{\mathrm{clip}}=0.05$, yields
higher early training rewards than the original update, with a smaller
gap later (Figure~\ref{fig:token-clipping},
Appendix~\ref{sec:token-clipping-results}).
This single-seed comparison measures reward on mixed-reward-filtered training
responses, so it does not establish improved held-out accuracy or verify the
theorem's uniform local reference constraint.
All other experiments use the original, unclipped update;
Appendix~\ref{app:token-clipping} gives implementation details.

\section{Conclusion}

GGPO learns from contrasts in KL-adjusted rewards among behavior-policy
responses without learner-to-behavior importance weights. Under suitable
coverage, its finite-group loss controls the gap between the learner's
single-response expected reward and the behavior policy's best-of-$K$ reward.
Reference refresh lets successive rounds build on earlier progress by
changing the Gibbs target as the learner improves.

Our experiments examine this approach on mathematical reasoning, where GGPO outperforms GRPO and AGRO under
training--inference mismatch while maintaining performance as policy
staleness grows. At $N=128$, GGPO also outperforms GRPO+TIS under the
same synchronization delay. On fixed offline data, it leads GRPO, GRPO+TIS,
and AGRO at every evaluated checkpoint from 160 updates onward.

{The theory motivates the algorithm by explaining how
small group loss and repeated reference updates improve reward, and how
tokenwise mismatch can give a more favorable dependence than sequence
coverage. The theorem provides sufficient conditions for convergence and explains
how errors in individual updates affect the final reward. Its parameter
settings need not coincide with those that work best in
practice. The experiments complement this analysis by evaluating GGPO
under practical training settings and substantial behavior mismatch.}

{Our results on mathematical reasoning with binary verifier
rewards motivate further study of richer reward signals and multi-turn
tasks. Extending the analysis to practical training choices and testing
GGPO across these settings are natural directions for future work.}

\section*{Acknowledgments}

{This work was supported in part by NSF grants 2416897, 2505932, and 2601666 and ONR grant N000142512318. Computing resources were provided by the National Center for Supercomputing Applications (NCSA) through Delta (NSF OAC-2005572) and DeltaAI (NSF OAC-2320345), as well as Illinois Computes and the NAIRR Pilot (NAIRR250157).}

Generative AI tools were used under the authors' direction for assistance with literature searches,
coding, conducting and analyzing experiments, deriving theoretical results and proofs, and polishing
the manuscript.

\bibliography{references}
\bibliographystyle{plainnat}

\appendix
\section{Proofs and extensions}
\label{app:proofs}

The proofs below explain the connections used in the main text, from the
population optimality condition to finite-group reward guarantees. We work
with a finite action space $\mathcal A$,
with policies positive on $\mathcal A$ and finite displayed expectations.
In particular, $0<Z_\beta=\E_{\piref}e^{r/\beta}<\infty$.
We suppress fixed prompts in pointwise derivations and write prompt averages
explicitly. Loss notation suppresses a fixed temperature and reference;
$H_m(\pi_\theta;\pib,x;\piref)$ specifies the reference when it varies.

\subsection{Proof of population duality and coverage}
\label{app:population-duality}
\label{app:main-population-proof}
This proves Proposition~\ref{prop:population-duality}, explaining both the
loss construction in Section~\ref{sec:theory} and the coverage condition in
its regularized-gap guarantee. The argument begins with the scalar dual,
whose residual provides a measure of departure from the Gibbs optimum.
Profiling this residual gives the contrastive loss, and coverage then
connects the behavior-policy expectation to the learner's regularized gap.

\begin{lemma}[Scalar duality and primal--dual gap]\label{lem:scalar-duality}
Fix a prompt $x$, temperature $\beta>0$, reward $r(x,a)$, and positive
reference policy $\piref(\cdot\mid x)$. Define
$V_\beta(\pi;x)=\E_{a\sim\pi}r(x,a)-\beta\KL(\pi\|\piref)$,
$V_\beta^*(x)=\max_\pi V_\beta(\pi;x)$, and
$Z_\beta=\E_{a\sim\piref}e^{r(x,a)/\beta}$.  For a positive candidate
$\pi_\theta$, define
$\Rth(x,a)=r(x,a)-\beta\log[\pi_\theta(a\mid x)/\piref(a\mid x)]$,
$h(z)=e^z-z-1$, and
\begin{equation}
\begin{aligned}
 \Phi_x(\nu)
   &:=\nu+\beta\bigl(\E_{a\sim\piref}e^{(r(x,a)-\nu)/\beta}-1\bigr),\\
 \underbrace{\max_\pi V_\beta(\pi;x)}_{\text{primal}}
   &=\beta\log Z_\beta
    =\underbrace{\min_{\nu\in\mathbb R}\Phi_x(\nu)}_{\text{scalar dual}},\\
 \Psi_x(\pi_\theta,\nu)
   &:=\Phi_x(\nu)-V_\beta(\pi_\theta;x).
\end{aligned}
 \label{eq:scalar-dual}
\end{equation}
Then $\Psi_x(\pi_\theta,\nu)\ge0$,
$\min_\nu\Psi_x(\pi_\theta,\nu)=V_\beta^*(x)-V_\beta(\pi_\theta;x)$, and
\begin{equation}
 \Psi_x(\pi_\theta,\nu)=\beta\E_{a\sim\pi_\theta}
 h((\Rth(x,a)-\nu)/\beta).
 \label{eq:residual-gap}
\end{equation}
\end{lemma}
\begin{proof}
The Gibbs policy
$\pi^*(a\mid x)=\piref(a\mid x)e^{r(x,a)/\beta}/Z_\beta$ satisfies
\[
 \beta\KL(\pi\|\pi^*)=\beta\log Z_\beta-V_\beta(\pi;x).
\]
Hence $V_\beta^*(x)=\beta\log Z_\beta$ and $\pi^*$ is its unique maximizer.
Moreover, $\Phi_x'(\nu)=1-Z_\beta e^{-\nu/\beta}$ and
$\Phi_x''(\nu)=Z_\beta e^{-\nu/\beta}/\beta>0$, so its minimum is
$\Phi_x(\beta\log Z_\beta)=V_\beta^*(x)$. Hence $\Phi_x(\nu)$ upper-bounds
the primal optimum for every $\nu$, so $\Psi_x$ is a nonnegative
primal--dual gap, and
$\min_\nu \Psi_x(\pi_\theta,\nu)=V_\beta^*(x)-V_\beta(\pi_\theta;x)$.
Finally, $\E_{\pi_\theta}\Rth=V_\beta(\pi_\theta;x)$ and
$\pi_\theta e^{(\Rth-\nu)/\beta}=\piref e^{(r-\nu)/\beta}$. Expanding
$h$ proves Eq.~\eqref{eq:residual-gap}.
\end{proof}

\begin{proof}[Proof of Proposition~\ref{prop:population-duality}]
Fix $x$ and use the quantities defined in Lemma~\ref{lem:scalar-duality}.
Replacing the learner expectation by $\pib$ and minimizing over $\nu$ gives
$\nu_\theta=\beta\log\E_{\pib} e^{\Rth/\beta}$: the first derivative is
$1-\E_{\pib} e^{(\Rth-\nu)/\beta}$ and the second is strictly positive.
Substitution gives $\Jbeh=\nu_\theta-\E_{\pib}\Rth$, the loss in
Eq.~\eqref{eq:main-population-loss}. Define
\begin{equation}
 q_\theta(a\mid x)
 =\frac{\pib(a\mid x)e^{\Rth(x,a)/\beta}}
 {\E_{a'\sim\pib}e^{\Rth(x,a')/\beta}}.
 \label{eq:population-tilt}
\end{equation}
Expanding its KL divergence gives $\beta\KL(\pib\|q_\theta)=\Jbeh$,
proving Eq.~\eqref{eq:main-profiled-duality}. Nonnegativity of $h$ and
$\pi_\theta/\pib\le C_x$ now imply
\[
 V_\beta^*-V_\beta(\pi_\theta;x)
 \le \Psi_x(\pi_\theta,\nu_\theta)
 \le C_x\beta\E_{\pib} h((\Rth-\nu_\theta)/\beta)=C_x\Jbeh.
\]
When $\pib=\pi_\theta$, profiling the same learner residual gives equality with the
regularized gap. Under positive support, zero loss forces $\Rth$ constant,
so $\pi_\theta\propto\piref e^{r/\beta}$ and hence $\pi_\theta=\pi^*$. Conversely,
this Gibbs policy has constant adjusted reward and zero loss.
\end{proof}

Averaging the coverage bound over prompts gives
$\E_X[V_\beta^*(X)-V_\beta(\pi_\theta;X)]
\le\E_X[C_X\Jbeh(\pi_\theta;X)]$; moving coverage outside the expectation
requires a uniform bound. Appendix~\ref{sec:integration-comparison}
compares sequence coverage with aggregate tokenwise mismatch.

\subsection{Finite-group guarantees}\label{app:group-bon-proof}
We establish the common zero set and population limit used by the
population and reference-refresh results, then state and prove the
best-of-$K$ bound of Section~\ref{sec:group-bon-motivation}.

\begin{proposition}[Finite-group target]
\label{prop:finite-group-target}
Suppose that $\Rth$ is bounded on the support of $\pib$ and that
$a_1,\ldots,a_m$ are sampled independently from $\pib$. Then, for every integer
$m\ge1$,
\begin{equation}
 0=H_1\le H_m\le H_{m+1}\le\Jbeh,
 \qquad
 H_m\longrightarrow\Jbeh.
 \label{eq:finite-group-monotonicity}
\end{equation}
For every $m\ge2$, $H_m=0$ if and only if $\Rth$ is constant on the support
of $\pib$, equivalently if and only if $\Jbeh=0$. If $\pi_\theta$, $\piref$,
and $\pib$ are positive on $\mathcal A$,
this zero condition also implies $\pi_\theta=\pi^*$.
\end{proposition}

\begin{proof}
Let $X=e^{\Rth/\beta}$ and $\bar X_m=m^{-1}\sum_{i=1}^mX_i$ for i.i.d.
copies $X_i$. Then
\begin{equation}
 H_m=\beta(\E\log\bar X_m-\E\log X),\qquad
 \Jbeh=\beta(\log\E X-\E\log X).
 \label{eq:finite-group-log-moments}
\end{equation}
Jensen gives $0=H_1\le H_m\le\Jbeh$. Given $m+1$ samples, the
leave-one-out means $A_j=m^{-1}\sum_{i\ne j}X_i$ average to $\bar X_{m+1}$.
Thus $\log\bar X_{m+1}\ge(m+1)^{-1}\sum_j\log A_j$; taking expectations
gives $H_{m+1}\ge H_m$, since each $A_j$ is distributed as the size-$m$
sample mean $\bar X_m$. The bounded i.i.d. sample average
$\bar X_m$ converges to $\E X$, and
boundedness of $\Rth$ bounds $\log\bar X_m$, so dominated convergence
gives $H_m\to\Jbeh$.

For $m\ge2$, equality in the group Jensen inequality requires all sampled
$X_i$ to agree. Hence $H_m=0$ implies two independent copies always agree on
the sampled support, so $X$ and $\Rth$ are constant there. The converse is immediate; the
population Jensen equality gives the same zero set for $\Jbeh$.
Under positivity on $\mathcal A$, constancy gives
$\pi_\theta\propto\piref e^{r/\beta}$, proving the final claim.
\end{proof}

\paragraph{Best-of-$K$ reward bound.}
Write $J_K(\pi;x)$ and $J(\pi;x)=J_1(\pi;x)$ for the expected best-of-$K$ and single-response rewards of a policy $\pi$
conditional on prompt $x$, so $J_K(\pi)=\E_X J_K(\pi;X)$ in
Eq.~\eqref{eq:best-of-k-reward}, and let $L_m(\pi_\theta)=\E_{X\sim\rho}H_m(\pi_\theta;\pib,X)$.
For $p_{x,t}=\Pr_{A\sim\pib(\cdot\mid x)}[r(x,A)>t]$, define
\begin{equation}\label{eq:kappa-tail-comparison}
 \kappa_{K,m}
 :=\max\left\{1,\sup_{x,t:\,p_{x,t}>0}
 \frac{1-(1-p_{x,t})^K}{1-(1-p_{x,t})^{m-1}}\right\}
 \le \max\left\{1,\frac{K}{m-1}\right\}.
\end{equation}
The inequality follows because $[1-(1-p)^k]/k$ decreases in $k$.  When
$K>m-1$, the upper bound is approached as a relevant behavior-policy
upper-tail probability tends to zero, but the actual multiplier can be smaller.
\begin{proposition}[Best-of-$K$ and the GGPO minimizer]\label{prop:gcpo-minimizer}
Let $r\in[0,1]$, fix positive behavior and reference policies with
$\Gamma_x\le\Gamma$ at every prompt, where $\Gamma_x$ is the oscillation of
$\log[\piref(\cdot\mid x)/\pib(\cdot\mid x)]$, and suppose every
policy $\pi_\theta$ in a positive hypothesis class $\Pi$ satisfies $\pi_\theta/\pib\le C$ uniformly.
For $m\ge2$ and $K\ge1$,
\begin{equation}\label{eq:gcpo-population-bon}
 J_K(\pib)-J(\pi_\theta)\le\kappa_{K,m}\{\beta(\log m+\Gamma)+CL_m(\pi_\theta)\}.
\end{equation}
In particular, with $\alpha_m=\inf_{\pi_\theta\in\Pi}L_m(\pi_\theta)$, a population
minimizer, if one exists, satisfies this bound with $L_m$ replaced by $\alpha_m$; an approximate minimizer
with loss at most $\alpha_m+\zeta$ incurs an additional $C\kappa_{K,m}\zeta$.
\end{proposition}
\begin{proof}
Fix a prompt and write $C_x=\max_a\pi_\theta(a\mid x)/\pib(a\mid x)$.
We first prove the prompt-level inequality
\begin{equation}\label{eq:group-bon-reference}
 J_K(\pib;x)-J(\pi_\theta;x)
 \le\kappa_{K,m}\{\beta(\log m+\Gamma_x)+C_xH_m(\pi_\theta;\pib,x)\}.
\end{equation}
The multiplier defined in Eq.~\eqref{eq:kappa-tail-comparison} satisfies
the required tail comparison at every prompt.
Suppress the prompt. Put $w_i=\pi_\theta(a_i)/\pib(a_i)$, $S=m^{-1}\sum_iw_i$,
$W=\max_iw_i$, and $P_i=w_i/(mS)$. Let
$u_i=r_i+\beta\log[\piref(a_i)/\pib(a_i)]$ and
$z_i=u_i/\beta-\log w_i=R_i/\beta$. With $U$ uniform on sample indices,
$c=\log\E_Ue^z$, and $g(t)=e^t-t-1\ge0$, one has
$\E_Ug(z-c)=\widehat{\mathcal J}_m/\beta$.
Define $Q_i\propto e^{u_i/\beta}$ and $A=\E_Pe^z$. Since
$Q_i=P_ie^{z_i}/A$,
\[
 \E_Pg(z-c)-\KL(P\|Q)=Ae^{-c}-1-\log(Ae^{-c})\ge0.
\]
The bound $P_i/U_i\le W/S$ therefore gives
$\beta S\KL(P\|Q)\le W\widehat{\mathcal J}_m$.
The finite Gibbs identity yields
$\max_i u_i-\E_Pu\le\beta\log m+\beta\KL(P\|Q)$.
Indeed, expanding $\KL(P\|Q)$, using
$\sum_iP_i\log P_i\ge-\log m$, and lower-bounding the log-partition by
$\max_i u_i/\beta$ gives this inequality. The oscillation of $u-r$ is at
most $\beta\Gamma_x$, hence
$\max_i r_i-\E_Pr\le\max_i u_i-\E_Pu+\beta\Gamma_x$.
Multiplying by $S$ and using
$S\E_Pr=m^{-1}\sum_iw_i r_i$ together with the preceding KL bound gives
\begin{equation}\label{eq:group-weighted-proof}
 S\max_i r_i-\frac1m\sum_iw_i r_i
 \le\beta S(\log m+\Gamma_x)+W\widehat{\mathcal J}_m.
\end{equation}
For $F=\max_i r_i$, symmetry and importance weighting give
$\E_{\pib^m}[SF]=\E_{\pi_\theta\otimes\pib^{m-1}}F$, while $\E S=1$ and
$\E[m^{-1}\sum_iw_i r_i]=\E_{\pi_\theta}r$.
At a threshold $t\in[0,1]$, write $p=\pib\{r>t\}$, $v=\pi_\theta\{r>t\}$,
and $d_k=1-(1-p)^k$. By the defining tail comparison,
\[
 d_K-v\le d_K(1-v)\le\kappa_{K,m}d_{m-1}(1-v).
\]
The layer-cake identities here are
\[
 J_K(\pib;x)=\int_0^1d_K\,dt,
 \qquad J(\pi_\theta;x)=\int_0^1v\,dt,
\]
and, because the mixed maximum has tail
$1-(1-v)(1-p)^{m-1}=v+(1-v)d_{m-1}$,
\[
 \E[SF]-\E_{\pi_\theta}r
 =\int_0^1(1-v)d_{m-1}\,dt.
\]
Integrating the threshold inequality therefore bounds
$J_K(\pib;x)-J(\pi_\theta;x)$ by
$\kappa_{K,m}(\E[SF]-\E_{\pi_\theta}r)$. Take expectations in
Eq.~\eqref{eq:group-weighted-proof} and use $W\le C_x$. Averaging
Eq.~\eqref{eq:group-bon-reference} over prompts, with $C_x\le C$ and
$\Gamma_x\le\Gamma$, proves the proposition. The minimizer statements
follow by substituting the corresponding population losses.
\end{proof}

On policy, $\pi_\theta=\pib$ and $C_x=1$ at every prompt;
with $\piref=\pib$, $\Gamma=0$. The class-wide coverage bound $C$
can still exceed one for other candidates.
The minimum loss is zero when the Gibbs policy belongs to the admissible
class; a restricted class or coverage cap may prevent this, and the
regularization term in the reward bound remains even at zero loss.

\subsection{Generalization and an oracle inequality}\label{app:group-generalization}
Theorem~\ref{thm:group-bon} below converts expected group loss into an
observed-loss bound. The oracle inequality then
states what approximate empirical minimization achieves in a restricted
hypothesis class.
Fix a prompt distribution $\rho$, behavior $\pib$, reference $\piref$, temperature
$\beta>0$, and hypothesis class $\Pi$ before collecting $n$ i.i.d. unfiltered
groups. For $f_{\pi_\theta}(Z)=\widehat{\mathcal J}_m(\pi_\theta;Z)$, write
$\mathcal F=\{f_{\pi_\theta}:\pi_\theta\in\Pi\}$ and let
$L_m(\pi_\theta)=\E f_{\pi_\theta}=\E_{X\sim\rho}H_m(\pi_\theta;\pib,X)$ and
$\widehat L_n(\pi_\theta)=n^{-1}\sum_jf_{\pi_\theta}(Z_j)$.
Assume $0\le f_{\pi_\theta}\le\Lambda$ and that the induced loss family admits a finite
uniform $\eta$-cover in the distance of Eq.~\eqref{eq:integration-cover-metric}, with centers in the family: each loss is within $\eta$
of a center on every possible group. Write
\[
 z=\log\frac{2\mathcal N_\infty(\eta,\mathcal F)}{\delta},\qquad
 \gamma_n=4\Lambda z/n+3\eta.
\]
A sufficient envelope is $\Lambda=1+\beta B$ when
$\operatorname{osc}_a\log[\pi_\theta(a\mid x)/\piref(a\mid x)]\le B$ uniformly,
since adjusted rewards then have range at most $\Lambda$. These assumptions apply
to all possible groups, not only observed responses. Coverage remains separate.

\begin{theorem}[Empirical GGPO bound]\label{thm:group-bon}
Under Proposition~\ref{prop:gcpo-minimizer}'s assumptions and the sampling and
cover conditions above, with probability at least $1-\delta$, every
data-selected $\pi_{\hat\theta}\in\Pi$ satisfies
\begin{equation}\label{eq:group-bon-observed}
 J_K(\pib)-J(\pi_{\hat\theta})
 \le\kappa_{K,m}\{\beta(\log m+\Gamma)+C[2\widehat L_n(\pi_{\hat\theta})+\gamma_n]\}.
\end{equation}
\end{theorem}
The bound applies to the output of GGPO; it requires no convexity or
successful empirical minimization. The reported loss uses unfiltered groups in
the original units, rather than the filtered, normalized training log.

\begin{proposition}[Raw-loss oracle inequality and reward consequence]\label{prop:gcpo-oracle}
Fix the behavior and reference policies and candidate class before drawing
$n$ independent unfiltered groups, with raw losses in $[0,\Lambda]$ and
the uniform loss cover specified above. Suppose the returned policy satisfies
\[
 \widehat L_n(\pi_{\hat\theta})
 \le\inf_{\pi_\theta\in\Pi}\widehat L_n(\pi_\theta)+\xi.
\]
Write $\alpha_m=\inf_{\pi_\theta\in\Pi}L_m(\pi_\theta)$ for the best
population raw loss in the class. With probability at least $1-\delta$,
\begin{equation}\label{eq:gcpo-raw-loss-oracle}
 L_m(\pi_{\hat\theta})\le4\alpha_m+2\xi+3\gamma_n.
\end{equation}
If the reward-comparison assumptions of
Proposition~\ref{prop:gcpo-minimizer} also hold, then on the same event,
\begin{equation}\label{eq:gcpo-oracle}
 J_K(\pib)-J(\pi_{\hat\theta})
 \le\kappa_{K,m}\{\beta(\log m+\Gamma)
       +C[4\alpha_m+2\xi+3\gamma_n]\}.
\end{equation}
\end{proposition}
\begin{proof}
We first compare population and empirical raw losses uniformly over the
candidate policies. A cover center is itself the raw loss of a candidate
policy, say $\pi'$. Its loss on one random group $Z$ is nonnegative and
bounded by $\Lambda$, so its variance satisfies
\[
 \operatorname{Var}(f_{\pi'}(Z))
 \le\E[f_{\pi'}(Z)^2]\le\Lambda L_m(\pi').
\]
Thus policies with small population loss also have small loss variance.
Bernstein's inequality gives, with failure probability at most $2e^{-z}$,
\[
 |\widehat L_n(\pi')-L_m(\pi')|
 \le\sqrt{2\Lambda L_m(\pi')z/n}+\frac{2\Lambda z}{3n}
 \le\frac{L_m(\pi')}{2}+\frac{5\Lambda z}{3n}.
\]
The last inequality absorbs the square-root term into half of the
population loss and a term of order $1/n$. Rearranging and rounding
$10/3$ up to $4$ gives both
\[
 L_m(\pi')\le2\widehat L_n(\pi')+4\Lambda z/n,\qquad
 \widehat L_n(\pi')\le2L_m(\pi')+4\Lambda z/n.
\]
The definition $z=\log[2\mathcal N_\infty(\eta,\mathcal F)/\delta]$
assigns failure probability $\delta/\mathcal N_\infty(\eta,\mathcal F)$
to each center. A union bound therefore makes these inequalities hold
for every center with probability at least $1-\delta$.

To extend them to a candidate $\pi_\theta$, choose a center $\pi'$ whose
raw loss differs by at most $\eta$ on every possible group. Both their
population and empirical losses then differ by at most $\eta$. Consequently,
\[
 L_m(\pi_\theta)\le L_m(\pi')+\eta
 \le2\widehat L_n(\pi')+4\Lambda z/n+\eta
 \le2\widehat L_n(\pi_\theta)+4\Lambda z/n+3\eta.
\]
The reverse direction has the same approximation cost. Recalling
$\gamma_n=4\Lambda z/n+3\eta$, we obtain, simultaneously for all
$\pi_\theta\in\Pi$,
\begin{equation}\label{eq:group-generalization}
 L_m(\pi_\theta)\le2\widehat L_n(\pi_\theta)+\gamma_n,\qquad
 \widehat L_n(\pi_\theta)\le2L_m(\pi_\theta)+\gamma_n.
\end{equation}
This event depends only on the sampled groups and holds regardless of
which policy the solver returns.

We now compare the returned policy with any candidate $\pi_\theta$.
The first inequality bounds the returned policy's population loss by its
empirical loss. The assumed solver accuracy compares that empirical loss
with the candidate's, and the second inequality returns to population loss:
\[
\begin{aligned}
 L_m(\pi_{\hat\theta})
 &\le2\widehat L_n(\pi_{\hat\theta})+\gamma_n\\
 &\le2\widehat L_n(\pi_\theta)+2\xi+\gamma_n\\
 &\le4L_m(\pi_\theta)+2\xi+3\gamma_n.
\end{aligned}
\]
Taking the infimum over candidates proves
Eq.~\eqref{eq:gcpo-raw-loss-oracle}; no attained minimizer is needed.
Substituting this raw-loss bound into
Proposition~\ref{prop:gcpo-minimizer} gives the reward consequence
in Eq.~\eqref{eq:gcpo-oracle}.
\end{proof}

\begin{proof}[Proof of Theorem~\ref{thm:group-bon}]
Use the uniform concentration bound in Eq.~\eqref{eq:group-generalization},
proved above before applying any solver-accuracy assumption.
Substitute its first inequality into Proposition~\ref{prop:gcpo-minimizer}.
Uniformity permits any data-selected output in the fixed class.
\end{proof}

The raw-loss conclusion of Proposition~\ref{prop:gcpo-oracle} also applies
within each refresh round, conditional on the preceding rounds.
Lemma~\ref{lem:integration-oracle} uses this conclusion with the local
loss bound $\Lambda=\beta$ and makes it simultaneous across rounds.

\paragraph{Prompt count and optimization setting.}
The oracle terms separate approximation, empirical optimization, and
statistics. These errors play different roles: optimizing the empirical
loss accurately cannot compensate for a class that cannot fit the target.
In particular, small $\xi$ alone is insufficient if the minimum loss
$\alpha_m$ is large. For an $\varepsilon$ statistical contribution in the observed-loss
bound, choose $\eta=\varepsilon/(6C\kappa_{K,m})$ and
$n\ge8\Lambda C\kappa_{K,m}\log[2\mathcal N_\infty(\eta,\mathcal F)/\delta]/\varepsilon$.
The oracle remainder has the same rate up to constants. These are conditional
small-loss rates; entropy is evaluated at the chosen accuracy, and coverage
or boundedness may depend on response length.
The analysis covers a fixed offline collection or a stage with fixed behavior
and reference policies, not arbitrary behavior and reference changes over a
full run. Additional response groups on a fixed set of sampled prompts do not
supply additional independent prompt draws from $\rho$. Independent prompt
draws count toward $n$ even when their values repeat.

\paragraph{Group size.}
With $m=K\ge2$, the multiplier is at most two; with $m=K+1$, it is one.
Thus groups of order $K$ directly support the best-of-$K$ target and can
reduce the required number of fresh prompt groups. The loss, its complexity,
and the temperature bias can also change with $m$, so monotone improvement
is not guaranteed. Total responses are $nm$. Appendix~\ref{sec:group-size-ablation}
studies group size empirically.

\begin{samepage}
\subsection{GGPO gradient, normalization, and filtering}
\label{app:gradient}

\paragraph{Off-policy tilt and on-policy limit.}
By Eq.~\eqref{eq:main-gibbs}, $e^{R_i/\beta}$ is proportional to
$\pi^*(a_i\mid x)/\pi_\theta(a_i\mid x)$. Because responses are sampled from
$\pib$, their normalized weighted empirical distribution converges~to
\[
 q_\theta(a\mid x)
 =\frac{\pib(a\mid x)\pi^*(a\mid x)/\pi_\theta(a\mid x)}
 {\E_{a'\sim\pib}[\pi^*(a'\mid x)/\pi_\theta(a'\mid x)]},
\]
\end{samepage}
the induced tilt in Eq.~\eqref{eq:population-tilt}. Holding $\pib$ fixed gives
\[
 \nabla_\theta\Jbeh
 =-\beta\left(\E_{q_\theta}\nabla_\theta\log\pi_\theta
 -\E_{\pib}\nabla_\theta\log\pi_\theta\right).
\]
When $\pib=\pi_\theta$, $q_\theta=\pi^*$ and the second expectation vanishes.
The objective value then equals the reverse-KL regularized gap, while the
frozen-$\pib$ derivative is the forward-KL/cross-entropy gradient toward $\pi^*$.
Weight concentration is summarized by
$\mathrm{ESS}_q=1/\sum_i q_i^2\in[1,m]$.

Let $\widehat{\mathcal J}_m^{\,g}:=\widehat{\mathcal J}_m(\pi_\theta;x_g,A_g)$, with
$A_g=a^g_{1:m}$, be the loss of group $g$, and let $\bar\ell$ be the mean response length in the
learner minibatch.
The implemented loss is given in Eq.~\eqref{eq:gcpo_loss_norm}.
If $s_i$ is the sequence log likelihood of response $i$, then
$\partial\widehat{\mathcal J}_m^{\,g}/\partial s_i=-\beta(q_i-1/m)$. After applying the
normalization above, its derivative is
$-(q_i-1/m)/(G_{\mathrm{mb}}\bar\ell\sqrt{m-1})$.
Writing the usual token-mean policy-gradient form as
$-c_i/(G_{\mathrm{mb}}m\bar\ell)$ therefore gives the group coefficient
\begin{equation}
 c_i=\frac{m(q_i-1/m)}{\sqrt{m-1}}.
 \label{eq:normalized-gcpo-gradient}
\end{equation}
The factors $1/\beta$ and $1/\bar\ell$ remove the raw $\beta$ prefactor and convert
the group average to a token mean; the factor $1/\sqrt{m-1}$ keeps the
coefficient scale similar to GRPO's standardized advantages.
Thus $\beta$ affects the implemented update through $q$, not through a
remaining global multiplier.

\paragraph{Binary-reward limit and weight concentration.}
Let $\mathcal S$ be the $k$ responses of a mixed-reward group with verifier
reward one, and write
$\Delta_i=\log[\pi_\theta(a_i\mid x)/\piref(a_i\mid x)]$.
For fixed finite $\Delta_i$, as $\beta\to0$ the inter-class logit gap in
Eq.~\eqref{eq:q-logits} dominates, so $q$ concentrates on $\mathcal S$ with
$q_i\to\softmax_{i\in\mathcal S}(-\Delta_i)$ and $q_i\to0$ otherwise. By
Eq.~\eqref{eq:normalized-gcpo-gradient}, every response outside $\mathcal S$ then
has $c_i\to-1/\sqrt{m-1}$, and the signed sum over correct responses is
\[
 \sum_{i\in\mathcal S}c_i\longrightarrow\frac{m-k}{\sqrt{m-1}}.
\]
For GRPO's reward advantages standardized using the sample standard
deviation with variance denominator $m-1$, and omitting the $10^{-6}$
stabilizer, the corresponding advantage sum is $\sqrt{k(m-k)(m-1)/m}$.
For AGRO, equal learner-to-reference log ratios within the group cancel
in leave-one-out centering, giving an advantage sum of $k(m-k)/(m-1)$.
Normalizing each curve to its value at $k=7$ gives
Figure~\ref{fig:group-weight}.

Because allocation inside $\mathcal S$ follows $q$ rather than being uniform, a
correct response can receive $c_i<0$. In this limit that happens exactly when
$(m-1)e^{-\Delta_i}<\sum_{j\in\mathcal S\setminus\{i\}}e^{-\Delta_j}$;
equivalently, writing
$\bar\Delta_{-i}=-\log[(k-1)^{-1}
\sum_{j\in\mathcal S\setminus\{i\}}e^{-\Delta_j}]$, when
$\Delta_i-\bar\Delta_{-i}>\log[(m-1)/(k-1)]$. A lone correct response
($k=1$) can never be penalized this way.

If the learner-to-reference log ratios $\Delta_i$ are equal across the
successful responses $\mathcal S$, the binary-reward limit above yields
\[
 c_i\longrightarrow\frac{1}{\sqrt{m-1}}
 \left(\frac{mr_i}{k}-1\right).
\]
The expression in parentheses is the centered success-normalized advantage
of MaxRL \citep{tajwar2026maxrl}.
MaxRL derives this update from an on-policy maximum-likelihood formulation,
whereas GGPO derives its log-mean-exp loss from KL-regularized duality and
permits behavior--learner mismatch without learner-to-behavior importance
weights. For equally rewarded responses,
\[
 \frac{q_i}{q_j}=\exp\!\left[-(\Delta_i-\Delta_j)\right]
\]
at every $\beta>0$. Thus, reference-dependent weighting persists even at
small temperatures when these log ratios differ.

\paragraph{Reference anchoring and refresh.}\label{app:anchoring}
The log-ratio difference $\Delta_i-\Delta_j$ controls the relative weights
of equally rewarded responses without attenuation by $\beta$.
At an exact policy-copy boundary, $\piref=\pi_\theta$, all $\Delta_i$ vanish,
so these responses receive equal weights. Subsequent learner updates
change the log ratios and favor responses with smaller $\Delta_i$.
This explains how reference refresh resets the within-class weighting.

The reward-only ablation sets $q_i=\softmax_i(r_i/\beta)$ in
Eq.~\eqref{eq:normalized-gcpo-gradient}, retaining the score-gradient update
and normalization. For mixed binary-reward groups, its coefficients approach
the MaxRL form above as $\beta\to0$, without requiring equal $\Delta_i$.
This removes the reference-dependent weighting that persists in full GGPO
at small temperatures.

\paragraph{Mixed-group filtering and normalization.}
The batch-dependent normalization in Eq.~\eqref{eq:gcpo_loss_norm} changes
the expected training loss, but its positive multiplier preserves zeros.
For a fixed prompt $x$ with
$0<\Pr_{a\sim\pib(\cdot\mid x)}[r(x,a)=1]<1$, define
$E_{\mathrm{mix}}=\{\min_i r(x,a_i)=0,\ \max_i r(x,a_i)=1\}$ and
\begin{equation}
 H_m^{\mathrm{mix}}(\pi_\theta;\pib,x)
 :=\E_{a_{1:m}\sim\pib^m}
 [\widehat{\mathcal J}_m(\pi_\theta;x,a_{1:m})\mid E_{\mathrm{mix}}].
 \label{eq:mixed-group-target}
\end{equation}

\begin{proposition}[Zero set under mixed-group filtering]
\label{prop:mixed-group-zero-set}
Suppose $m\ge2$, $\mathcal{A}$ is finite, the verifier reward is deterministic
and binary, and $0<\Pr_{a\sim\pib}[r(x,a)=1]<1$. Then
$H_m^{\mathrm{mix}}(\pi_\theta;\pib,x)=0$ if and only if $\Rth(x,a)$ is constant
on the support of $\pib$. Hence, for each such prompt, the filtered loss has the
same zero set on behavior support as the unfiltered target $H_m$ and the
population objective $\Jbeh$. Positive batch-dependent normalization
preserves this zero set over the retained prompt distribution.
\end{proposition}

\begin{proof}
The empirical loss is nonnegative and is zero exactly when all adjusted rewards
in the group agree. Any behavior-supported reward-zero action and reward-one
action can co-occur in a mixed group with positive probability, so a zero
conditional expectation forces their adjusted rewards to agree. Two supported
actions with the same reward can each be connected through any supported
action of the opposite reward. Hence all supported actions have the same
adjusted reward. The converse is immediate. Finally, a positive batch-dependent
multiplier cannot create or remove zeros of a nonnegative minibatch loss.
\end{proof}

\subsection{Reference refresh: exact and imperfect updates}
\label{app:refresh-bound}
We first show how exact Gibbs targets compose across reference updates.
We then prove Proposition~\ref{prop:refresh-loss} by accounting for the
residual error in each round and summing the resulting reward comparisons.

\paragraph{Fixed-reference bias and exact composition.}
For binary reward with reference success probability $s\in(0,1)$, the
zero-loss Gibbs policy succeeds with probability
$s e^{1/\beta}/(1-s+s e^{1/\beta})<1$. Best-of-$K$ from the same reference
succeeds with probability $1-(1-s)^K$, which can be strictly larger.
This gap remains at zero loss for every group size, including the
infinite-group limit. Reference refresh changes the target itself.

\begin{lemma}[Exact reference composition]
\label{lem:exact-refresh-composition}
Suppose the learner, reference, and behavior policies are positive on
$\mathcal A$. In every round $t$, suppose $\pi_t$ has zero population
finite-group loss for $m_t\ge2$ independent responses from $\pibt{t}$,
relative to reference $\pi_{t-1}$. Then temperatures $\beta_t>0$ compose as
$\pi_T(a\mid x)\propto\pi_0(a\mid x)
\exp\{r(x,a)\sum_{t=1}^T\beta_t^{-1}\}$. In particular, for constant
$\beta$ and the best-of-$K$ selection policy $u_K$ from $\pi_0$,
\[
 J_K(\pi_0)-J(\pi_T)\le \frac{\beta\log K}{T}.
\]
\end{lemma}
\begin{proof}
Under positivity on $\mathcal A$, Proposition~\ref{prop:finite-group-target}
identifies zero population finite-group loss with the Gibbs policy. With
reference $\pi_{t-1}$ in round $t$, Eq.~\eqref{eq:main-gibbs} therefore gives
$\pi_t(a\mid x)\propto\pi_{t-1}(a\mid x)e^{r(x,a)/\beta_t}$.
For constant temperature, iterating yields Eq.~\eqref{eq:refresh-temperature}. More generally, temperatures
$\beta_t>0$ give effective temperature
$\beta_{\rm eff}=(\sum_{t=1}^T\beta_t^{-1})^{-1}$.
Write $\mathcal K(u,v)=\E_X\KL(u(\cdot\mid X)\|v(\cdot\mid X))$.
For constant $\beta$, the final policy maximizes
$J(\pi)-(\beta/T)\mathcal K(\pi,\pi_0)$, so for every comparator $u$,
\[
 J(u)-J(\pi_T)\le (\beta/T)\mathcal K(u,\pi_0).
\]
Let $u_K$ be the distribution of the selected response among $K$ independent
draws from $\pi_0$, using a fixed tie-breaking rule. Selection of action $a$
requires at least one draw of $a$, hence $u_K(a\mid x)\le K\pi_0(a\mid x)$.
Consequently $J(u_K)=J_K(\pi_0)$ and $\mathcal K(u_K,\pi_0)\le\log K$,
which proves the reward bound.
\end{proof}

\paragraph{Imperfect reference updates.}
\begin{proposition}[Reference refresh with imperfect rounds]
\label{prop:refresh-loss}
Let $r\in[0,1]$, $m\ge2$, and $u_K$ be the best-of-$K$ selection policy
from a positive initial policy $\pi_0$. For each positive round output $\pi_t$ and
positive behavior $\pibt{t}$, assume uniformly over prompts and actions that
\begin{equation}
 \frac{\pi_t(a\mid x)}{\pibt{t}(a\mid x)}\le C,
 \qquad \frac{u_K(a\mid x)}{\pibt{t}(a\mid x)}\le C,
 \qquad
 \operatorname{osc}_a\!\left\{
 \frac{r(x,a)}{\beta}
 -\log\frac{\pi_t(a\mid x)}{\pi_{t-1}(a\mid x)}\right\}\le D.
 \label{eq:refresh-population-assumptions}
\end{equation}
Here $\operatorname{osc}_a$ denotes maximum minus minimum over responses
$a$ at a fixed prompt. Define the average output policy and average population loss by
\[
 \bar\pi_T=\frac1T\sum_{t=1}^T\pi_t,
 \qquad
 \bar L_T=\frac1T\sum_{t=1}^T
 \E_X H_m(\pi_t;\pibt{t},X;\pi_{t-1}).
\]
Then
\begin{equation}
 J_K(\pi_0)-J(\bar\pi_T)
 \le\frac{\beta\log K}{T}
       +\sqrt{\frac{4\beta C e^D}{1-1/m}\,\bar L_T}.
 \label{eq:refresh-loss}
\end{equation}
\end{proposition}

For round $t$, let $R_t=r-\beta\log(\pi_t/\pi_{t-1})$ and
$L_t(\pi_t)=\E_XH_m(\pi_t;\pibt{t},X;\pi_{t-1})$.
Use the common coverage, residual-range, and group-size bounds $C$, $D$,
and $m$ from the proposition.
\begin{proof}[Proof of Proposition~\ref{prop:refresh-loss}]
First, an adjusted reward $R$ with $\operatorname{osc}(R/\beta)\le D$
obeys the finite-group variance bound
\begin{equation}
 H_m\ge\frac{(1-1/m)e^{-D}}{2\beta}\operatorname{Var}_{\pib}(R).
 \label{eq:refresh-variance}
\end{equation}
For a realized group, set $z_i=R_i/\beta$ and
$g(s)=\log(m^{-1}\sum_i e^{sz_i})$. For $s\in[0,1]$, the normalized weight
$e^{sz_i}/\sum_j e^{sz_j}$ of every group member is at least $e^{-D}/m$.
The minimum-mean-square characterization of variance therefore gives
$g''(s)=\operatorname{Var}_s(z)\ge e^{-D}\operatorname{Var}_{\rm group}(z)$.
Integrating with weight $1-s$ bounds the group loss below by
$\beta e^{-D}\operatorname{Var}_{\rm group}(z)/2$. Its expectation is
Eq.~\eqref{eq:refresh-variance}, since empirical variance with denominator $m$
has expectation $(1-1/m)\operatorname{Var}_{\pib}(R/\beta)$.

Let $\bar R_t=\E_{\pibt{t}}R_t$. Since both density ratios have
$\pibt{t}$-mean one,
\[
 \E_{u_K}R_t-\E_{\pi_t}R_t
 =\E_{\pibt{t}}\!\left[
 \left(\frac{u_K}{\pibt{t}}-\frac{\pi_t}{\pibt{t}}\right)
 (R_t-\bar R_t)\right].
\]
Cauchy--Schwarz and the coverage assumptions in Proposition~\ref{prop:refresh-loss} then give pointwise
\[
 \E_{u_K}R_t-\E_{\pi_t}R_t
 \le\sqrt{2C\operatorname{Var}_{\pibt{t}}(R_t)}.
\]
Indeed, $\E_{\pibt{t}}[(u_K/\pibt{t}-\pi_t/\pibt{t})^2]\le 2C$, because both
densities integrate to one and their cross product is nonnegative.
Using Eq.~\eqref{eq:refresh-variance} and Jensen over prompts bounds this
residual difference by
$\{4\beta C e^DL_t(\pi_t)/(1-1/m)\}^{1/2}$.
The exact identity
\begin{align*}
 J(u_K)-J(\pi_t)
 ={}&\beta[\mathcal K(u_K,\pi_{t-1})-\mathcal K(u_K,\pi_t)
                           -\mathcal K(\pi_t,\pi_{t-1})]\\
 &+\E_X[\E_{u_K}R_t-\E_{\pi_t}R_t]
\end{align*}
then telescopes over rounds. Dropping nonnegative terminal and step KL
terms and using $\mathcal K(u_K,\pi_0)\le\log K$ yields the population
inequality
\[
 J_K(\pi_0)-J(\bar\pi_T)\le\frac{\beta\log K}{T}
 +\frac1T\sum_{t=1}^T
 \sqrt{\frac{4\beta C e^D}{1-1/m}L_t(\pi_t)}.
\]
Finally, $T^{-1}\sum_{t=1}^T\sqrt{L_t(\pi_t)}\le\sqrt{\bar L_T}$ gives
Eq.~\eqref{eq:refresh-loss}.

\end{proof}

\subsection{Relation to AGRO: approximation and weighting}\label{app:variance}
This calculation supports the comparison in Section~\ref{sec:related-work}:
AGRO is the local quadratic approximation to the population contrastive loss when
adjusted rewards have small spread. The finite-group comparison below also
explains the role of exponential weighting and the slack in a uniform
variance upper bound. Neither comparison globally orders the unscaled losses.

For the local relation to AGRO \citep{tang2025agro}, let
$\bar R_\theta=\E_{\pib}\Rth$ and $\tilde R=\Rth-\bar R_\theta$. The population contrastive loss can be written as the centered log moment
\begin{equation}
 \Jbeh=\beta\log\E_{\pib} e^{\tilde R/\beta}.
 \label{eq:centered-log-moment}
\end{equation}
Because $\E_{\pib}\tilde R=0$, Taylor expansion as
$\lVert\tilde R\rVert_\infty/\beta\to0$ gives
\begin{equation}
 \Jbeh
 =\frac{1}{2\beta}\E_{\pib}\tilde R^2
 +O\!\left(\frac{\lVert\tilde R\rVert_\infty^3}{\beta^2}\right)
 =\frac{1}{2\beta}\operatorname{Var}_{\pib}[\Rth]
 +O\!\left(\frac{\lVert\tilde R\rVert_\infty^3}{\beta^2}\right).
 \label{eq:agro-local-relation}
\end{equation}
AGRO's off-policy population loss is
$\mathcal L_{\mathrm{AGRO}}=\tfrac12\operatorname{Var}_{\pib}[\Rth]$, so locally
$\Jbeh=\mathcal L_{\mathrm{AGRO}}/\beta$ up to higher-order terms. Equation~(4) of
\citet{tang2025agro} gives the leave-one-out finite-group stochastic gradient
of this variance objective. Our implementation divides out its exterior
$\beta$ factor and averages over response tokens
(Eq.~\eqref{eq:agro-surrogate}). The corresponding
coefficient relation follows by defining the
exponentially tilted distribution
\begin{equation}
 \frac{dq_\theta}{d\pib}
 =\frac{e^{\Rth/\beta}}{\E_{\pib} e^{\Rth/\beta}}
 =1+\frac{\tilde R}{\beta}
 +O\!\left(\frac{\lVert\tilde R\rVert_\infty^2}{\beta^2}\right).
 \label{eq:agro-coefficient-taylor}
\end{equation}
Thus the raw GGPO coefficient multiplying the policy score is
\begin{equation}
 -\beta\left(\frac{dq_\theta}{d\pib}-1\right)
 =-\tilde R
 +O\!\left(\frac{\lVert\tilde R\rVert_\infty^2}{\beta}\right).
 \label{eq:agro-gradient-coefficient}
\end{equation}
The leading term $-\tilde R$ is exactly the score-function coefficient obtained
by differentiating the normalized AGRO loss $\mathcal L_{\mathrm{AGRO}}/\beta$, holding the
behavior distribution fixed. Hence the scaled AGRO loss is the local quadratic
approximation to the raw population GGPO loss, and its centered linear
coefficient is the first-order approximation to the raw GGPO coefficient when
$\lVert\tilde R\rVert_\infty/\beta\to0$. Our default training setting uses a smaller
temperature relative to verifier-reward differences, so we treat AGRO as a
distinct objective rather than an approximation along the observed training
trajectory.

\begin{remark}[Finite-group weighting and a variance upper bound]
\label{rem:agro-group-weighting}
For a fixed group of size $m\ge2$, write $\bar R=m^{-1}\sum_iR_i$ and
$A_m=(2m)^{-1}\sum_i(R_i-\bar R)^2$, the half-variance with denominator $m$.
Define $q_{s,i}=e^{sR_i/\beta}/\sum_j e^{sR_j/\beta}$ for $s\in[0,1]$.
The raw GGPO loss has the exact representation and conservative bound
\begin{equation}
 \widehat{\mathcal J}_m
 =\frac1\beta\int_0^1(1-s)\operatorname{Var}_{q_s}(R)\,ds
 \le\frac{m}{\beta}A_m.
 \label{eq:agro-group-weighting}
\end{equation}
Indeed, for $g(s)=\log(m^{-1}\sum_i e^{sR_i/\beta})$,
$g''(s)=\operatorname{Var}_{q_s}(R)/\beta^2$; Taylor's integral formula gives
the equality. Also
$\operatorname{Var}_{q_s}(R)\le\sum_iq_{s,i}(R_i-\bar R)^2
\le\sum_i(R_i-\bar R)^2=2mA_m$, which gives the inequality.
Locally, as adjusted-reward spread divided by $\beta$ vanishes,
$\widehat{\mathcal J}_m\sim A_m/\beta$ for nonconstant groups. Thus discarding
the adaptive weights yields an upper bound whose ratio to the local loss
approaches $m$. This is slack in this particular bound, not an unavoidable
factor for every variance-based analysis or a global ordering of
$\widehat{\mathcal J}_m$ and $A_m/\beta$.
For the unbiased half-variance
$A_m^{\mathrm{unb}}=mA_m/(m-1)$, the same bound reads
$\widehat{\mathcal J}_m\le(m-1)A_m^{\mathrm{unb}}/\beta$.
These are raw objective comparisons, separate from the normalization of the
implemented updates.

Let
$\mathcal V_{\mathrm{AGRO}}(\pi_\theta)
:=\E_X[\tfrac12\operatorname{Var}_{\pib}(\Rth)]$
be the prompt-averaged population half-variance. Taking expectations gives
$\E A_m^{\mathrm{unb}}=\mathcal V_{\mathrm{AGRO}}(\pi_\theta)$.
Proposition~\ref{prop:gcpo-minimizer} therefore yields, under its conditions,
\begin{equation}
\begin{aligned}
 L_m(\pi_\theta)
 &\le\frac{m-1}{\beta}\mathcal V_{\mathrm{AGRO}}(\pi_\theta),\\
 J_K(\pib)-J(\pi_\theta)
 &\le\kappa_{K,m}\left\{\beta(\log m+\Gamma)
 +\frac{C(m-1)}{\beta}\mathcal V_{\mathrm{AGRO}}(\pi_\theta)\right\}.
\end{aligned}
\label{eq:agro-reward-bound}
\end{equation}
\end{remark}

\begingroup
\section{Formal final-policy guarantee and proof}
\label{app:integration-proof}
\subsection{Formal setup and final-policy guarantee}
\label{sec:integration-formal-setup}
The population and finite-group bounds in
Eqs.~\eqref{eq:main-population-coverage} and~\eqref{eq:main-finite-group-bon}
use sequence coverage, which can grow rapidly with response length.
We relate behavior and learner distributions through moments of their
response likelihood ratio. Conditional token log moments provide a
sufficient bound on these coverage moments. Together with
local reference control, this yields a finite-group loss calibration and,
after repeated reference updates, a finite-data guarantee for the final policy.

Proposition~\ref{prop:refresh-loss} controls an average output using
population losses and behavior coverage of the learner and best-of-$K$
comparator. Here we assume binary rewards and initial correct-response
coverage. Each exact Gibbs update multiplies correct-response odds by
$e^{1/\beta}$; the theorem quantifies how errors in fitting these updates
affect the final reward.

\paragraph{Reference-update rounds.}
Responses have at most $H$ tokens (pad with a deterministic terminal token
when necessary). Write $y$ for a complete response and $a$ for a token.
For a policy $\pi$, $\pi_h(a\mid s)$ is its token conditional at position
$h$, where a legal history $s=(x,y_{<h})$ consists of a prompt and an
allowed response prefix. Each round holds the reference fixed while fitting
a new policy.
In round $t$, set the reference to $\pi_{t-1}$ and freeze
behavior $\pibt{t}$ and a candidate class $\Pi_t$ before drawing data.
Draw $n$ independent prompt groups from $\rho$, each with $m\ge2$ independent
responses under $\pibt{t}$. Minimize the raw loss in
Eq.~\eqref{eq:empirical_group_loss}, with reference $\pi_{t-1}$, over
$\Pi_t$ to the empirical accuracy specified below. Denote the output by
$\pi_t$ and use it as the next reference. Groups are unfiltered and fresh across rounds;
data may be reused within a round.

\begin{assumption}[Local reference control]
\label{ass:integration-local}
Rewards satisfy $r(x,y)\in\{0,1\}$, and all policies are positive on the
same finite response support. The temperature
satisfies $\beta\ge2H+1$, and every candidate in $\Pi_t$ satisfies
\begin{equation}\label{eq:integration-local}
 \left|\log\frac{\pi_h(a\mid s)}{\pi_{t-1,h}(a\mid s)}\right|
 \le\frac1\beta\qquad(h=1,\ldots,H)
\end{equation}
at every legal history $s$ and token $a$, in every round $t$.
\end{assumption}
The local bound limits variation in adjusted rewards. Together with
$\beta\ge2H+1$, it keeps the group loss in $[0,\beta]$, allowing us to
estimate its expectation from finitely many groups.

The exact KL-regularized optimum for reference $\mu=\pi_{t-1}$ is
$p_\mu^*(y\mid x)\propto\mu(y\mid x)e^{r(x,y)/\beta}$.
Its conditional normalizer $Z(s)=\E_\mu[e^{r/\beta}\mid s]$ lies in
$[1,e^{1/\beta}]$ because rewards are binary. Hence
\[
 \frac{p_{\mu,h}^*(a\mid s)}{\mu_h(a\mid s)}
 =\frac{Z(sa)}{Z(s)}\in[e^{-1/\beta},e^{1/\beta}],
 \qquad
 \left|\log\frac{p_{\mu,h}^*(a\mid s)}{\mu_h(a\mid s)}\right|
 \le\frac1\beta.
\]
Thus each round's optimum satisfies the constraint; the final comparator
$\pi_\star$ need not, since repeated updates can move far from the
initial policy. The errors below allow approximate fitting within the
constrained class. We can also relax the token bound to
$(1+\gamma)/\beta$ for $\gamma\ge0$ if $\beta\ge1+2H(1+\gamma)$;
Appendix~\ref{sec:integration-local-relaxation} gives the details,
including the revised behavior-lag bound.

The local constraint also motivates the practical token clipping in
Eq.~\eqref{eq:integration-optional-clipping}. That variant limits token
log-ratio contributions to the response coefficients, with some bias,
but leaves the policy class unconstrained. The theorem therefore analyzes
the raw loss under Eq.~\eqref{eq:integration-local}, whereas
Section~\ref{sec:token-clipping-experiment} evaluates clipping and
Appendix~\ref{app:token-clipping} specifies its implementation.

\begin{definition}[Round losses]
\label{def:integration-losses}
For reference $\pi_{t-1}$ and a group $g=(x,y_{1:m})$, let
$\ell_t(\pi;g)$ be the raw group loss in
Eq.~\eqref{eq:empirical_group_loss}. Define its population and empirical
averages by
\[
 L_t(\pi):=\E_{X\sim\rho}H_m(\pi;\pibt{t},X;\pi_{t-1}),\qquad
 \widehat L_t(\pi):=\frac1n\sum_{j=1}^n\ell_t(\pi;g_j),
\]
where $g_1,\ldots,g_n$ are the sampled groups in round $t$.
\end{definition}

Estimating the loss of a single fixed policy is not enough, because the
returned policy is chosen using the same data. The finite-data guarantees
therefore depend on the complexity of the group-loss functions. We measure
this complexity using the following uniform covering number.
\begin{definition}[Uniform loss covering number]
\label{def:integration-covering-number}
For a family $\mathcal F$ of real-valued functions on prompt-response groups,
define
\begin{equation}\label{eq:integration-cover-metric}
 \|f-f'\|_\infty:=\sup_g|f(g)-f'(g)|.
\end{equation}
The uniform $L_\infty$ covering number $\mathcal N_\infty(\eta,\mathcal F)$
is the smallest cardinality of a subset $\mathcal F_\eta\subseteq\mathcal F$
such that, for every $f\in\mathcal F$, some $f'\in\mathcal F_\eta$ satisfies
$\|f-f'\|_\infty\le\eta$. Set $\mathcal N_\infty=\infty$ if no finite
cover exists. The supremum ranges over all possible groups.
\end{definition}
A uniform cover therefore controls approximation even on groups absent
from the sampled data, which allows us to bound finite-data error.

\begin{assumption}[Representation, optimization, and loss complexity]
\label{ass:integration-errors}
Fix nonnegative bounds $\alpha$, $\xi$, and $\eta$, and an integer $M\ge1$.
In every round,
\begin{equation}\label{eq:integration-error-assumptions}
\begin{aligned}
 \inf_{\pi\in\Pi_t}L_t(\pi)
 &\le\alpha &&\text{(representation)},\\
 \widehat L_t(\pi_t)-\inf_{\pi\in\Pi_t}\widehat L_t(\pi)
 &\le\xi &&\text{(optimization)},\\
 \mathcal N_\infty(\eta,\{\ell_t(\pi;\cdot):\pi\in\Pi_t\})
 &\le M &&\text{(loss cover)}.
\end{aligned}
\end{equation}
The candidate class and cover are chosen before sampling in each round.
\end{assumption}

\paragraph{Relation to approximation assumptions in prior work.}
Bounds on approximation and estimation errors also appear in compatible
function-approximation analyses of policy gradients
\citep[Section~6]{agarwal2021pg} and inexact policy mirror descent
\citep[Section~3]{yuan2023loglinear}. Closure under policy improvement is
another related structural condition \citep{bhandari2024global}.
Here we bound population raw-loss representation error and achieved
empirical optimization error; these are distinct from the loss-cover
condition used for generalization. The Gibbs target satisfying the token
constraint does not by itself place it in a chosen model class.
For a tabular class containing that target,
$r-\beta\log(p_\mu^*/\mu)=\beta\log Z_\mu(x)$ is constant over responses,
so every group loss is zero and the representation bound can use
$\alpha=0$. Loss complexity and achieved empirical accuracy remain separate
requirements. \citet[Remark~24]{agarwal2021pg} give linear-MDP examples of
zero transfer bias for their compatible approximation error; that statement
alone does not establish zero raw-loss representation error here.

\begin{definition}[Combined approximation error]
\label{def:integration-error}
Define
\begin{equation}\label{eq:integration-approximation}
 \varepsilon_A:=4\alpha+2\xi+9\eta.
\end{equation}
\end{definition}
The three components describe different sources of error. Here $\alpha$
describes how well the candidate class can fit the Gibbs target in population
loss, and $\xi$ describes the accuracy of empirical minimization. The cover error $\eta$ allows the finite-data argument to
apply to an infinite class. For a finite class, use all its loss functions
as centers, set $\eta=0$, and let $M$ bound its cardinality.

Increasing the sample size alone does not remove these errors. The case $\varepsilon_A=0$ requires arbitrarily small attainable
population loss, exact empirical minimization, and no loss-cover error.

\paragraph{Behavior mismatch.}
The loss is measured under behavior, whereas reward is evaluated under the
returned learner. Define the full-response ratio
$W_t(y\mid x)=\pi_t(y\mid x)/\pibt{t}(y\mid x)$ and the uniform coefficients
\begin{equation}\label{eq:integration-coverage-moment}
\begin{aligned}
 \mathcal C_p&=\operatorname*{ess\,sup}_{t,x,\,\mathrm{training}}
 \left(\E_{\pibt{t}(\cdot\mid x)}W_t^p\right)^{1/(p-1)},
 &&2\le p<\infty,\\
 \mathcal C_\infty&=\operatorname*{ess\,sup}_{t,x,y,\,\mathrm{training}}W_t(y\mid x).
\end{aligned}
\end{equation}
The supremum includes all training randomness and the data-dependent
returned policy. A uniform bound over the candidate class before sampling
is sufficient but stronger than necessary. These coefficients describe
distributional coverage, separately from the loss-cover size $M$;
$\log\mathcal C_p$ is the uniform order-$p$ R\'enyi divergence
\citep{vanerven2014renyi}.

A tokenwise sufficient condition uses
$b_{t,h}=\pi_{\mathrm{beh},t,h}(\cdot\mid s_h)$ and
$w_{t,h}(a)=\pi_{t,h}(a\mid s_h)/b_{t,h}(a)$, where
$s_h=(x,y_{<h})$. Let $V_{\log}<\infty$ be a deterministic bound such that
\begin{equation}\label{eq:integration-mismatch}
 \operatorname*{ess\,sup}_{t,\,\mathrm{training},\,x,y}
 \sum_{h=1}^H\E_{a\sim\pi_{t,h}(\cdot\mid s_h)}
       [(\log w_{t,h}(a))^2]\le V_{\log}.
\end{equation}
The sum follows one supported response path before taking the worst case
over paths, prompts, rounds, and training randomness. For a chosen $0<\nu\le1$, the additional bound
\begin{equation}\label{eq:integration-token-cap}
 |\log w_{t,h}(a)|\le\min\{\log2,3\nu/2\}
\end{equation}
uniformly over rounds, legal histories, actions, and training randomness
implies, by Lemma~\ref{lem:integration-log-moments},
\begin{equation}\label{eq:integration-bernstein-coefficient}
 \mathcal C_{1+1/\nu}\le c_\nu,
 \qquad c_\nu:=\exp((1+\nu^{-1})V_{\log}).
\end{equation}
This bound holds for any finite $V_{\log}$, with its size entering
through $c_\nu$. The learner-to-behavior token condition is separate
from the learner-to-reference constraint in
Assumption~\ref{ass:integration-local}. Both conditions are used in
the analysis; the loss contains no learner-to-behavior importance weights.

Finally, the initial policy needs to assign enough probability to correct
responses for repeated Gibbs updates to reach the desired reward. Fix a
comparator policy $\pi_\star$ and use the following coverage condition.
\begin{assumption}[Initial correct-response coverage]
\label{ass:integration-coverage}
There is a constant $C_\star\ge1$ such that
\begin{equation}\label{eq:integration-initial-coverage}
 \Pr_{y\sim\pi_\star(\cdot\mid x)}[r(x,y)=1]
 \le C_\star\Pr_{y\sim\pi_0(\cdot\mid x)}[r(x,y)=1]
 \qquad\text{for every prompt $x$}.
\end{equation}
\end{assumption}
Thus $C_\star$ bounds how much more often the
comparator can succeed than the initialization. Best-of-$K$ sampling
from $\pi_0$ satisfies this condition with $C_\star=K$: the probability
of at least one correct response among $K$ independent draws is at most
$K$ times the single-response success probability.

\begin{theorem}[Final reward and finite-data convergence]
\label{thm:integration-convergence}
Suppose Assumptions~\ref{ass:integration-local}, \ref{ass:integration-errors},
and~\ref{ass:integration-coverage} hold, and use the sampling procedure
above. Fix
$\varepsilon,\delta\in(0,1)$ and set
\begin{equation}\label{eq:integration-budget}
 T=\left\lceil\beta\log\frac{6C_\star}{\varepsilon}\right\rceil,
 \qquad z=\log(2TM/\delta),\qquad X=C_\star T^2/\varepsilon.
\end{equation}
For a fixed real $2\le p<\infty$ with
$\mathcal C_p<\infty$, let $\nu=1/(p-1)$ and choose
\begin{equation}\label{eq:integration-moment-budget}
 n\ge\left\lceil\max\left\{1,768\pi^2zX,
       24z(32\pi^2X)^{1+\nu}\mathcal C_p\right\}\right\rceil.
\end{equation}
With probability at least $1-\delta$ over the sampled training groups,
\begin{equation}\label{eq:integration-final}
 J(\pi_\star)-J(\pi_T)\le\varepsilon+
 \frac{2\pi^2C_\star T^2}{\beta}
 \left[4\varepsilon_A+
 4\beta^{1/p}(\mathcal C_p^{p-1}-1)^{1/p}
                 \varepsilon_A^{1-1/p}\right].
\end{equation}
Here $\varepsilon_A=4\alpha+2\xi+9\eta$ is defined in
Eq.~\eqref{eq:integration-approximation}, and $J$ is expected
single-response reward. Under Eqs.~\eqref{eq:integration-mismatch}
and~\eqref{eq:integration-token-cap}, one may replace $\mathcal C_p$ by
$c_\nu$ in both the budget and fitting residual. Choosing the smallest
permitted integer gives
\[
 n=O(zX)+O(zc_\nu X^{1+\nu}),\qquad 0<\nu\le1,
\]
with absolute constants uniformly in $\nu$: the two terms can use
$768\pi^2$ and $24(32\pi^2)^2$, respectively.

Each run uses $Tn$ fresh prompt groups and $mTn$ scored responses.
The scalar $\pi$ in the constants is the circle constant.
\end{theorem}

\begin{remark}[Fourth moments and bounded coverage]\label{rem:integration-bounded-coverage}
Under the same local reference, sampling, approximation, and initial
coverage conditions, retain $T$, $z$, and $X$ from
Theorem~\ref{thm:integration-convergence}.

\emph{Fourth moments ($p=4$).} Here $\nu=1/3$. The budget in
Eq.~\eqref{eq:integration-moment-budget} gives
$n=O(zX)+O(z\mathcal C_4X^{4/3})$ and, with probability at least $1-\delta$,
\[
 J(\pi_\star)-J(\pi_T)\le\varepsilon+
 \frac{2\pi^2C_\star T^2}{\beta}
 \left[4\varepsilon_A+
 4\beta^{1/4}(\mathcal C_4^3-1)^{1/4}\varepsilon_A^{3/4}\right].
\]
Under the pathwise budget in Eq.~\eqref{eq:integration-mismatch} and
$|\log w_{t,h}|\le1/2$, one may use $\mathcal C_4\le e^{4V_{\log}}$.
Thus the additional accuracy exponent is $1/3$, with the coverage
coefficient and fitting residual still explicit.

\emph{Bounded coverage ($p=\infty$).} If $\mathcal C_\infty<\infty$,
the alternative choice
\begin{equation}\label{eq:integration-uniform-budget}
 n\ge\left\lceil\max\{1,384\pi^2\mathcal C_\infty zX\}\right\rceil
\end{equation}
gives, with probability at least $1-\delta$,
\begin{equation}\label{eq:integration-uniform-reward}
 J(\pi_\star)-J(\pi_T)\le\varepsilon+
 \frac{2\pi^2C_\star T^2}{\beta}\,4\mathcal C_\infty\varepsilon_A.
\end{equation}
This branch has linear dependence on $X$, corresponding to $\nu=0$ in
the accuracy exponent; $c_\nu$ is defined only
for $\nu>0$. When each behavior policy is the current frozen reference,
$\mathcal C_\infty\le e^{H/\beta}<\sqrt e$ by
Assumption~\ref{ass:integration-local}, giving $n=O(zX)$.
\end{remark}

\paragraph{Interpretation.}
The budgets control statistical reward error by $\varepsilon$. Prompt
draws may repeat, but additional groups on a fixed prompt pool do not
supply fresh independent draws from $\rho$. The fitting residual retains
representation, achieved empirical optimization, and loss-cover errors;
increasing $n$ alone does not remove them. When $\varepsilon_A=0$, it
vanishes. Applying the result to a chosen model requires these error
bounds as well as the actual policy constraints. The proof in
Appendix~\ref{sec:integration-finite-data} combines raw-loss calibration,
finite-data concentration, and relative angular transfer.

We first relate group loss to the error in one Gibbs update. We then track how
errors in these updates affect the final success probability. For a reference
policy $\mu$, define
$F_\mu(\pi)=\E_X V_\beta(\pi;X)$ using $\mu$ in place of $\piref$ in
Eq.~\eqref{eq:main-regularized-value}. Let $p_\mu^*$ be its Gibbs maximizer
and write $F_\mu^*=F_\mu(p_\mu^*)$.
Write $\mathcal K(\pi,p_\mu^*)=\E_X\KL(\pi(\cdot\mid X)\|p_\mu^*(\cdot\mid X))$.
The Gibbs identity gives
$\Delta_t=F_{\pi_{t-1}}^*-F_{\pi_{t-1}}(\pi_t)
=\beta\mathcal K(\pi_t,p_{\pi_{t-1}}^*)$.
Throughout, $\operatorname{osc}$ means maximum minus minimum over responses.

\subsection{Probability transfer and raw-loss calibration}
The next two lemmas restate the probability results of
\citet[Lemmas~\ProximalLogMoment{} and~\ProximalTransfer]{bai2026proximalkl}
in our notation. Their proofs are given there.

\begin{lemma}[Sequential log moments]\label{lem:integration-log-moments}
Fix a prompt and realized positive learner and behavior policies on a common
finite autoregressive support of at most $H$ tokens. At each history, let
$w_h(a)=\pi_h(a\mid s_h)/\pi_{\mathrm{beh},h}(a\mid s_h)$.
Pad terminated paths with ratio one. Write $Y_h=\log w_h$ and
$W=\prod_{h=1}^Hw_h$.
Suppose the sum of $\E_{\pi_h}Y_h^2$ along each supported path is at most
$V_{\log}$, $w_h\ge1/2$, and
$Y_h-\E_{\pi_h}Y_h\le b$ at every history for a deterministic $b>0$.
Then for every $\lambda\ge0$,
\begin{equation}\label{eq:integration-log-moments}
 \log\E_{\mathrm{beh}}W^{1+\lambda}
 \le V_{\log}\left[\lambda+
       \frac{e^{b\lambda}-1-b\lambda}{b^2}\right].
\end{equation}
For $b\lambda<3$, the right side is at most
$V_{\log}[\lambda+\lambda^2/(2(1-b\lambda/3))]$.
The same bounds hold for every prefix ratio. Under
Eq.~\eqref{eq:integration-token-cap}, take $\lambda=1/\nu$ and
$b=3\nu/2$ to obtain Eq.~\eqref{eq:integration-bernstein-coefficient}.
\end{lemma}
The conditional mean $\E_{\pi_h}Y_h$ is a nonnegative KL divergence, so
an upper bound on $Y_h$ also bounds its centered value. These statements
are conditional on the realized policies; the deterministic uniform budget
in Eq.~\eqref{eq:integration-mismatch} covers prompts, rounds, and training
randomness. See \citet[Lemma~\ProximalLogMoment, proof]{bai2026proximalkl}.

\begin{lemma}[Bounded-function moment transfer]\label{lem:integration-function-transfer}
For distributions $P,Q$ on a common finite support, let $W=dP/dQ$,
$0\le f\le B_f$, and $a=\E_Qf$. For every real $p\ge2$,
\[
 \E_Pf\le a+(\E_QW^p-1)^{1/p}B_f^{1/p}a^{1-1/p}.
\]
If $W\le\mathcal C_\infty$ almost surely, then
$\E_Pf\le\mathcal C_\infty a$.
\end{lemma}
See \citet[Lemma~\ProximalTransfer, proof]{bai2026proximalkl} for both
inequalities.

\begin{lemma}[Raw-loss calibration]\label{lem:integration-moment-transfer}
Let learner $\pi$ and behavior $b$ be positive on a common support, and
let the response score satisfy $\operatorname{osc}_y(R/\beta)\le1$ at every
prompt. For $m\ge2$ independent responses from $b$, let $L$ be the expected
raw group loss, averaged over prompts, and define
\[
 G_\pi(R)=\E_{X\sim\rho}\left[
 \beta\log\E_{y\sim\pi(\cdot\mid X)}e^{R(X,y)/\beta}
 -\E_{y\sim\pi(\cdot\mid X)}R(X,y)\right].
\]
If $\sup_x\E_b(\pi/b)^p\le\mathcal C_p^{p-1}$ for $2\le p<\infty$, then
\begin{equation}\label{eq:integration-moment-transfer}
 G_\pi(R)\le4L+
 4\beta^{1/p}(\mathcal C_p^{p-1}-1)^{1/p}L^{1-1/p}.
\end{equation}
If instead $\pi/b\le\mathcal C_\infty$ uniformly, then
$G_\pi(R)\le4\mathcal C_\infty L$.
For $R=r-\beta\log(\pi/\pi_{\mathrm{ref}})$,
$G_\pi(R)=F_{\pi_{\mathrm{ref}}}^*-F_{\pi_{\mathrm{ref}}}(\pi)$.
\end{lemma}
\begin{proof}
Fix a prompt and any response distribution $P$. For $w=R/\beta$ and
$g(\tau)=\log\E_Pe^{\tau w}$, tilting changes densities by factors between
$e^{-\tau}$ and $e^\tau$. The minimum-over-centers formula for variance gives
$e^{-\tau}\Var_Pw\le g''(\tau)\le e^\tau\Var_Pw$.
Integrating with weight $1-\tau$ yields
\begin{equation}\label{eq:integration-curvature}
 \frac{\Var_P R}{e\beta}
 \le\beta\log\E_Pe^{R/\beta}-\E_PR
 \le\frac{e-2}{\beta}\Var_P R.
\end{equation}
For the uniform distribution on a sampled group, the expected empirical
variance is $(1-1/m)\Var_bR$. Hence the conditional loss $L_x$ satisfies
$v:=\Var_bR\le2e\beta L_x$. Put $u=R-\E_bR$, so $|u|\le\beta$.
Apply Lemma~\ref{lem:integration-function-transfer} with $P=\pi$, $Q=b$,
$f=u^2$, $B_f=\beta^2$, and $D_p=\mathcal C_p^{p-1}-1$:
\[
 \Var_\pi R\le\E_\pi u^2
 \le v+D_p^{1/p}\beta^{2/p}v^{1-1/p}.
\]
The upper curvature bound now gives
\[
 G_{\pi,x}(R)\le2e(e-2)L_x+
 (e-2)(2e)^{1-1/p}\beta^{1/p}D_p^{1/p}L_x^{1-1/p}.
\]
Both coefficients are below $4$. Average over prompts and use concavity
of $L_x^{1-1/p}$. Under bounded coverage the same local transfer lemma gives
$\Var_\pi R\le\mathcal C_\infty v$, proving the second conclusion.
\end{proof}

\subsection{Oracle inequality for each refresh round}
\label{sec:integration-oracle}
The next lemma compares the population raw loss of each returned policy
with the best population loss in that round's candidate class. Its
probability statement holds simultaneously over all refresh rounds, so
it applies to the entire sequence used in the final-reward analysis.

\begin{lemma}[Per-refresh oracle inequality]\label{lem:integration-oracle}
Suppose Assumptions~\ref{ass:integration-local} and~\ref{ass:integration-errors}
hold. In each of $T$ rounds, freeze the reference, behavior policy, candidate
class, and loss cover before sampling $n$ fresh independent prompt groups
as in Appendix~\ref{sec:integration-formal-setup}. Each group contains
$m\ge2$ independent behavior-policy responses and is unfiltered.
Let $\pi_t\in\Pi_t$ be the returned policy with empirical suboptimality
at most $\xi$, as specified in Eq.~\eqref{eq:integration-error-assumptions}.
Fix $\delta\in(0,1)$ and set $z=\log(2TM/\delta)$.
With probability at least $1-\delta$, simultaneously for every
$t=1,\ldots,T$,
\begin{equation}\label{eq:integration-oracle}
 L_t(\pi_t)\le4\inf_{\pi\in\Pi_t}L_t(\pi)
       +2\xi+\frac{24\beta z}{n}+9\eta.
\end{equation}
In particular, with $\varepsilon_A=4\alpha+2\xi+9\eta$,
\begin{equation}\label{eq:integration-population-loss-bound}
 L_t(\pi_t)\le\varepsilon_A+\frac{24\beta z}{n}.
\end{equation}
\end{lemma}
\begin{proof}
Condition on all data and policies from the preceding rounds. The current
reference, behavior, candidate class, and loss cover are then fixed, and
the new groups are independent. Assumption~\ref{ass:integration-local}
bounds the adjusted-reward range by $1+2H\le\beta$, so every candidate's
raw group loss lies in $[0,\beta]$.

Apply the raw-loss oracle inequality in
Proposition~\ref{prop:gcpo-oracle}, Eq.~\eqref{eq:gcpo-raw-loss-oracle},
with $\Lambda=\beta$, candidate class $\Pi_t$, and failure probability
$\delta/T$. Its population and empirical losses are now $L_t$ and
$\widehat L_t$. The round's loss-cover size is at most $M$, so the
logarithmic term in that proposition is at most
\[
 \log\frac{2M}{\delta/T}=\log\frac{2MT}{\delta}=z.
\]
The proposition therefore gives, with conditional probability at least
$1-\delta/T$,
\[
 L_t(\pi_t)\le4\inf_{\pi\in\Pi_t}L_t(\pi)
       +2\xi+\frac{12\beta z}{n}+9\eta.
\]
Here the statistical term is $3\gamma_n$ with
$\gamma_n\le4\beta z/n+3\eta$. This application uses only the raw-loss
conclusion of the proposition; its best-of-$K$ reward assumptions are
not needed.

The conditional failure bound holds for every preceding history.
Averaging over that history gives unconditional failure probability at
most $\delta/T$ for each round. A union bound over the $T$ rounds then
gives one event of probability at least $1-\delta$ on which all the
round-wise oracle inequalities hold. Loosening $12$ to $24$ gives
Eq.~\eqref{eq:integration-oracle}. Finally, the representation bound
$\inf_{\pi\in\Pi_t}L_t(\pi)\le\alpha$ and the definition
$\varepsilon_A=4\alpha+2\xi+9\eta$ give
Eq.~\eqref{eq:integration-population-loss-bound}.
\end{proof}

\subsection{Reference refresh and final reward}
We use the relative angular transfer bound of
\citet[Lemma~\ProximalAngular]{bai2026proximalkl},
restated below in our notation. It converts the per-round Gibbs gaps
into a final-policy reward bound.

\begin{lemma}[Polynomial angular propagation]\label{lem:integration-angular}
Suppose rewards are binary, the KL coefficient in \eqref{eq:main-regularized-value}
is $\beta>0$, and the comparator has coverage constant $C_\star$
as defined in Appendix~\ref{sec:integration-formal-setup}. For a positive
reference policy $\mu$, recall the prompt-averaged regularized objective
and its optimal value:
\[
 F_\mu(\pi):=\E_{X\sim\rho}\!\left[
 \E_{Y\sim\pi(\cdot\mid X)}r(X,Y)
 -\beta\KL(\pi(\cdot\mid X)\|\mu(\cdot\mid X))\right],
 \qquad F_\mu^*:=\max_\pi F_\mu(\pi),
\]
where the maximum is over all response policies on the common support.
Fix $s\in(0,1)$ and
$T\ge\lceil\beta\log(3C_\star/s)\rceil$. For any positive policy sequence
$\pi_0,\ldots,\pi_T$, set
$\Delta_t=F_{\pi_{t-1}}^*-F_{\pi_{t-1}}(\pi_t)$.
Then
\begin{equation}\label{eq:integration-angular-gap}
 J(\pi_\star)-J(\pi_T)
 \le s+\frac{2\pi^2C_\star}{\beta}
                \left(\sum_{t=1}^T\sqrt{\Delta_t}\right)^2.
\end{equation}
In particular,
\begin{equation}\label{eq:integration-angular-uniform}
 J(\pi_\star)-J(\pi_T)
 \le s+\frac{2\pi^2C_\star T}{\beta}\sum_{t=1}^T\Delta_t
 \le s+\frac{2\pi^2C_\star T^2}{\beta}\max_t\Delta_t.
\end{equation}
Here $C_\star$ controls initial correct-solution coverage and the
scalar $\pi$ is the circle constant. The error coefficient grows
polynomially with $T$.
\end{lemma}

The cited result assumes binary rewards, positive policies, and the initial
correct-response coverage in Assumption~\ref{ass:integration-coverage}.
Its round count $K$ is our $T$, and its proximal gap
$D_{t-1}=\beta\E_X\KL(\pi_t\|p_{\pi_{t-1}}^*)$ equals $\Delta_t$
by the Gibbs identity above. This is a deterministic result for the realized
policy sequence; its proof is given in the cited lemma. Substituting gives
Eq.~\eqref{eq:integration-angular-gap}; Cauchy--Schwarz gives
Eq.~\eqref{eq:integration-angular-uniform}.

\subsection{Proof of Theorem~\ref{thm:integration-convergence}}
\label{sec:integration-finite-data}
\begin{proof}[Proof of Theorem~\ref{thm:integration-convergence} and Remark~\ref{rem:integration-bounded-coverage}]
\leavevmode\par\noindent
\textbf{Step 1: Bound the population raw loss in every round.}
Apply Lemma~\ref{lem:integration-oracle} with
$z=\log(2TM/\delta)$ as in Eq.~\eqref{eq:integration-budget}.
With probability at least $1-\delta$, the population-loss bound
in Eq.~\eqref{eq:integration-population-loss-bound} holds for every round.
Work on this simultaneous event for the remainder of the proof.
The moment bound is uniform over training randomness and therefore applies
to every realized returned policy.

\textbf{Step 2: Convert raw loss into a per-refresh optimality gap.}
Apply Lemma~\ref{lem:integration-moment-transfer} to
$R=r-\beta\log(\pi_t/\pi_{t-1})$, whose score range is bounded by the
local reference and temperature conditions.
For finite $p$, put $q=1-1/p$ and $D_p=\mathcal C_p^{p-1}-1$.
Lemma~\ref{lem:integration-moment-transfer} and
$(a+b)^q\le a^q+b^q$ give
\[
 \Delta_t\le
 4\varepsilon_A+4\beta^{1/p}D_p^{1/p}\varepsilon_A^q
 +\frac{96\beta z}{n}+4\beta D_p^{1/p}(24z/n)^q.
\]
Set $\zeta=\beta/(4\pi^2X)$. The first statistical term is at
most $\zeta/2$ when $n\ge768\pi^2zX$. The second is at most
$\zeta/2$ when
$n\ge24z(32\pi^2X)^{1+\nu}D_p^\nu$; use
$D_p^\nu\le\mathcal C_p$ to obtain
Eq.~\eqref{eq:integration-moment-budget}.
Under bounded coverage, Lemma~\ref{lem:integration-moment-transfer} gives
$\Delta_t\le4\mathcal C_\infty L_t(\pi_t)$. Its statistical term
$96\mathcal C_\infty\beta z/n$ is at most $\zeta$ under
Eq.~\eqref{eq:integration-uniform-budget}.

\textbf{Step 3: Propagate the gaps to the final reward.}
Apply Lemma~\ref{lem:integration-angular}, using
Eq.~\eqref{eq:integration-angular-uniform} with $s=\varepsilon/2$,
on the simultaneous event from Step 1. The finite-moment gap bound gives
Eq.~\eqref{eq:integration-final}, and the bounded-coverage gap bound gives
Eq.~\eqref{eq:integration-uniform-reward}.
This proves the bounded-coverage case of Remark~\ref{rem:integration-bounded-coverage};
its fourth-moment case follows by setting $p=4$ in the finite-moment bound.
Finally, Lemma~\ref{lem:integration-log-moments} supplies
$\mathcal C_p\le c_\nu$ under the stated token conditions, and
$(32\pi^2)^{1+\nu}\le(32\pi^2)^2$ for $0<\nu\le1$ gives the absolute
order constants. Only groups count as independent prompt observations.
\end{proof}

\paragraph{Token moments and sequence coverage.}\label{sec:integration-comparison}
The coefficient $\mathcal C_\infty$ controls a full-response ratio, whereas
$V_{\log}$ aggregates conditional token second moments along a path.
For example, $|\log w_h|\le a/\sqrt H$ allows $V_{\log}=a^2$ and the
sequence-ratio envelope $\exp(a\sqrt H)$. When the token cap also holds,
the Bernstein coefficient is at most $\exp((1+\nu^{-1})a^2)$.
The representation and achieved optimization errors remain separate
conditions on the model class and solver.
Uniform coverage is the infinite-order branch because
$\E_bW^p\le\mathcal C_\infty^{p-1}\E_bW=\mathcal C_\infty^{p-1}$.
If a returned learner equals behavior, its response ratio is one and
the direct calibration reduces to $\Delta_t\le4L_t(\pi_t)$. In the
analyzed procedure behavior is frozen before fitting, so sampling from
the current reference generally still produces learner--behavior mismatch.

\subsection{The Gibbs target and behavior lag}
\label{sec:integration-local-relaxation}
The conditional-normalizer calculation after
Assumption~\ref{ass:integration-local} shows that the exact Gibbs target
satisfies the local constraint. Under its relaxed version
$|\log(\pi_h/\pi_{t-1,h})|\le(1+\gamma)/\beta$ for $\gamma\ge0$,
the score range is at most $1+2H(1+\gamma)$. Thus
$\beta\ge1+2H(1+\gamma)$ preserves the calibration and concentration
arguments, provided representation and loss complexity hold for the
relaxed class.

For stored-policy behavior $\pibt{t}=\pi_{t-d_t}$ with
$1\le d_t\le\min(d,t)$, telescoping gives
$|\log w_{t,h}|\le d(1+\gamma)/\beta$ and
one may take $V_{\log}=Hd^2(1+\gamma)^2/\beta^2$.
This supplies the Bernstein token condition when
$d(1+\gamma)/\beta\le\min\{\log2,3\nu/2\}$. Independently of this cap,
multiplying token ratios gives
$\mathcal C_\infty\le\exp(Hd(1+\gamma)/\beta)$, so the direct branch
has a sufficient budget $n=O(zX\exp(Hd(1+\gamma)/\beta))$. For behavior equal to the
current reference, $d=1$; the relaxed temperature condition makes this
coverage coefficient smaller than $\sqrt e$. The moment branch instead
uses $c_\nu=\exp((1+\nu^{-1})Hd^2(1+\gamma)^2/\beta^2)$ with an
$X^{1+\nu}$ term. Which budget is smaller depends on both these
coefficients and the target accuracy.
The lag $d$ counts reference rounds, separately from experimental
synchronization delay. Additional training--inference mismatch requires
bounds on the actual learner-to-behavior ratios in
Eqs.~\eqref{eq:integration-mismatch}
and~\eqref{eq:integration-token-cap}.

\paragraph{Alternative bound from token spans.}
Let $b_h^{\rm span}$ be deterministic bounds, uniform over histories,
prompts, rounds, and training randomness, on the span across actions of
$\log w_h$. Put $r=\sum_h(b_h^{\rm span})^2/8$.
At each history write $Y_h=\log w_h$.
Normalization $\E_{\pi_h}e^{-Y_h}=1$ and Hoeffding's lemma give
$\E_{\pi_h}Y_h\le(b_h^{\rm span})^2/8$ and
$\log\E_{b_h}e^{pY_h}\le p(p-1)(b_h^{\rm span})^2/8$.
Successive conditioning gives $\mathcal C_p\le e^{rp}$.
This bound applies without the lower-ratio condition of
Lemma~\ref{lem:integration-log-moments}. For stored-policy lag,
$r\le Hd^2(1+\gamma)^2/(2\beta^2)$. Use $e^{rp}$ as an upper bound
on $\mathcal C_p$ in Theorem~\ref{thm:integration-convergence}.

\endgroup
\section{Experimental details}
\label{app:experiments}

\subsection{Data and evaluation}
\label{app:data-evaluation}

\paragraph{Training data.}
We describe the data construction and evaluation protocol here so that the
reported comparisons can be reproduced. ALP-Math combines the mathematics split of
\href{https://huggingface.co/datasets/LLM360/guru-RL-92k}{Guru-RL}
\citep{cheng2025guru} with an
OpenR1-derived dataset \citeyearpar{weqweasdasOpenR1Scores}.\footnote{Dataset revisions:
Guru-RL, \texttt{2e2790a962a3c099bfb5ea61389cbf98a5ea439b};
OpenR1, \texttt{da327b36b4dc077240aa8efd7f25c4017cd74ef4}.}
We use the 54,404 problems in \nolinkurl{train/math__combined_54.4k.parquet}
from \nolinkurl{LLM360/guru-RL-92k}.
From \nolinkurl{weqweasdas/from_default_filtered_openr1_with_scores},
we retain the 23,036 examples
whose upstream mean verifier score lies strictly between $0.2$ and $0.8$.
We merge the two sources and deduplicate by exact user-content string after
applying the common prompt format, within and across sources. For each duplicated prompt, we keep the first occurrence, processing
Guru-RL before OpenR1. This reduces 77,440
records to 73,293 unique prompts. Length filtering retains 73,290 prompts:
each input has at most 2,048 tokens after applying the chat template,
including system and user instructions and the generation prompt.
We independently checked the training data against all five evaluation
benchmarks and found no overlapping problems.
The \qwenone{} default-setting dataloaders shuffle with seeds 42, 1234, and 2027.
The $N=128$ condition also uses three training seeds; other settings with
periodic behavior synchronization use seed 42.

Training-time evaluation uses
\href{https://huggingface.co/datasets/HuggingFaceH4/MATH-500}{MATH-500}
from PRM800K. Final evaluation additionally uses
\href{https://huggingface.co/datasets/math-ai/aime24}{AIME 2024},
\href{https://huggingface.co/datasets/math-ai/aime25}{AIME 2025},
\href{https://huggingface.co/datasets/math-ai/olympiadbench}{OlympiadBench}
and \href{https://huggingface.co/datasets/math-ai/minervamath}{Minerva Math}
\citep{lewkowycz2022solving}. MATH-500 is the PRM800K subset of MATH
\citep{hendrycks2021math,lightman2024verify}, and OlympiadBench is the
text-only mathematics split \citep{he2024olympiadbench}. Ground truths are
their boxed answers. The evaluation sets contain 500 MATH-500 problems, 30
AIME 2024 problems, 30 AIME 2025 problems, 674 OlympiadBench problems, and
272 Minerva Math problems, for 1{,}506 problems in total. Every
problem-weighted average we report uses these dataset sizes as weights.
The accompanying preprocessing code pins these evaluation datasets to fixed
revisions.\footnote{MATH-500:
\href{https://huggingface.co/datasets/HuggingFaceH4/MATH-500/tree/6e4ed1a2a79af7d8630a6b768ec859cb5af4d3be}{\texttt{6e4ed1a}};
AIME 2024:
\href{https://huggingface.co/datasets/math-ai/aime24/tree/83a7f387baaa524a8bda0022eac0541582297103}{\texttt{83a7f38}};
AIME 2025:
\href{https://huggingface.co/datasets/math-ai/aime25/tree/563bb8404243c5f09de6ec262f2db674fe5bce9b}{\texttt{563bb84}};
OlympiadBench:
\href{https://huggingface.co/datasets/math-ai/olympiadbench/tree/4faaf1e6ec17d11a4218a9bf4c049ecaf954dd84}{\texttt{4faaf1e}};
Minerva Math:
\href{https://huggingface.co/datasets/math-ai/minervamath/tree/ee46ddc498933b1977577953250ca5c66be64f96}{\texttt{ee46ddc}}.}

\paragraph{Models and prompts.}
We use \href{https://huggingface.co/Qwen/Qwen3-1.7B-Base}{Qwen3-1.7B-Base} and
\href{https://huggingface.co/Qwen/Qwen3-4B-Base}{Qwen3-4B-Base}
with their corresponding tokenizers. The exact system prompt is ``Please reason step by
step, and put your final answer within \texttt{\textbackslash boxed\{\}}.'' Each
user problem is followed by ``Let's think step by step and output the final
answer within \texttt{\textbackslash boxed\{\}}.''

\paragraph{Reward and decoding.}
The binary reward uses Math-Verify \citep{kydlicek2025mathverify}
to verify answers against the ground truth. We use no separate format bonus
or length penalty.

Training rollouts use vLLM \citep{kwon2023vllm} at sampling temperature 1.0, top-$p=1.0$, top-$k$
disabled, 2{,}048-token prompt and response maxima, and $m=8$ responses per prompt.
Training-time and final evaluation use the same sampling temperature, top-$p$, top-$k$,
and length settings. Each checkpoint is evaluated once, drawing 32 responses
per problem.
The end-of-sequence (EOS) token is included in the response probability, and
the response mask includes generated tokens through the first EOS token while
excluding subsequent padding. If generation reaches the 2,048-token response
maximum without EOS, all generated tokens define the maximum-length-truncated
prefix; no unobserved EOS factor is included.

\paragraph{Trainer and evaluation policies.}
The bounds use $\pi_\theta$, the policy defined by the trainer likelihoods.
Evaluation samples from $\pi_\theta^{\mathrm{eval}}$, the rollout implementation
of the same checkpoint; mean@32 estimates $J(\pi_\theta^{\mathrm{eval}})$.
For a common prompt distribution $\rho$ and reward function $r\in[0,1]$,
\[
\bigl|J(\pi_\theta^{\mathrm{eval}})-J(\pi_\theta)\bigr|
\le \mathbb E_{X\sim\rho}\operatorname{TV}\!\left(
\pi_\theta^{\mathrm{eval}}(\cdot\mid X),\pi_\theta(\cdot\mid X)\right),
\]
where $\operatorname{TV}(p,q)=\tfrac12\sum_a|p(a)-q(a)|$ compares distributions
over complete responses. Subtracting this term transfers a learner reward
lower bound to the evaluation policy.

\paragraph{Best-of-$K$ evaluation.}
For binary rewards, best@$K$ and pass@$K$ both estimate $J_K(\pi_\theta^{\mathrm{eval}})$ in
Eq.~\eqref{eq:best-of-k-reward}. Given $c$ correct responses out of 32,
the unbiased estimator is $1-\binom{32-c}{K}/\binom{32}{K}$ for
$1\le K\le32$ \citep{chen2021codex}, with $\binom{32-c}{K}=0$ when $K>32-c$.
It averages the maximum reward over all size-$K$ subsets: at $K=1$, it
equals $c/32$ (mean@32); at $K=32$, it equals $\mathbf{1}\{c>0\}$.
Reported scores average these values over problems.

\paragraph{Mixed-reward filtering and refill.}
In the default setting ($N=8$), each behavior collection draws the next 256
prompts from the shuffled dataloader and generates $m$ responses per prompt
under one frozen pre-update policy. We retain groups containing at least one correct and one
incorrect response and continue sampling until the collection is full.
Each collection allows at most 32 generation batches, including the initial
batch; training stops with an error if the target batch cannot be filled.
For a prompt with success probability $p$, the probability that a sampled
group is retained is $1-p^m-(1-p)^m$. The procedure therefore changes the
training distribution by favoring prompts that produce mixed-reward groups.
Groups with equal rewards carry no verifier signal for GRPO, AGRO, or GGPO,
although reference regularization can still contribute an update.
Within each matched comparison, methods use the same filtering rule, refill
budget, batch geometry, and prompt order.

\paragraph{Mismatch diagnostics.}
K3 and ESS are computed on learner minibatches retained after mixed-reward filtering.
The behavior--learner mismatch curve reports token-mean K3 \citep{schulman2020klapprox},
$\exp(\lambda_\tau)-\lambda_\tau-1$, for
$\lambda_\tau=\log\pi_\theta(a^\tau\mid a^{<\tau})-\log\pib(a^\tau\mid a^{<\tau})$,
averaged over the logged response tokens.
For the $B_{\mathrm{resp}}$ responses in a logged batch, the learner-to-behavior sequence
importance weight is
$w_i=\exp(\sum_\tau M_{i\tau}\lambda_{i\tau})$, and the reported normalized effective
sample size is
$\mathrm{ESS}_{\mathrm{IS}}=(\sum_i w_i)^2/(B_{\mathrm{resp}}\sum_i w_i^2)\in[1/B_{\mathrm{resp}},1]$.
These sequence-level weights provide a mismatch diagnostic for all methods.
For Figure~\ref{fig:offline}, drift is the token mean of $|\lambda_\tau|$ over response tokens selected by the mask $M$. Its ESS uses the $B_{\mathrm{resp}}=2{,}048$ responses of each logged block of
eight learner minibatches, so the minimum normalized ESS is $1/2{,}048$. Gradient norms use each method's native actor-loss normalization and
are therefore interpreted only within a method over training.

\subsection{Training setup}

Training uses AdamW \citep{loshchilov2019adamw}: learning rate $\actorlr$, betas $(0.9,0.999)$,
weight decay $0.01$, constant schedule, no warmup, and gradient clip 1. We use
one actor epoch per collection and no entropy bonus. In the default setting
($N=8$), each 256-prompt collection
contains $m=8$ responses per prompt and is divided into 32-prompt minibatches,
giving $N=8$ optimizer updates. Unless stated otherwise, every method uses
$\beta=0.01$, and neither AGRO nor GGPO applies a token-level
current-to-pre-update ratio, clipping, learner-to-behavior importance weights,
or an additive KL term.

\paragraph{Objectives and reference policies.}
GGPO optimizes the loss in Eq.~\eqref{eq:gcpo_loss_norm}. In every main
comparison, AGRO and GGPO refresh $\piref\leftarrow\pi_\theta$ every
$\Tref=160$ optimizer updates, because their regularized targets depend on the
active reference, whereas our GRPO implementation retains its initial reference.
Appendix~\ref{sec:reference-regularization}
compares a fixed reference, $\Tref=160$, and $\Tref=320$ for all three
methods.

\paragraph{Default GRPO.}
GRPO standardizes rewards within each group of $m=8$ responses by subtracting
the group mean and dividing by the sample standard deviation plus $10^{-6}$;
the resulting advantages are held fixed during differentiation.
The token-level PPO ratio \citep{schulman2017ppo} divides the current learner
probability by the behavior probability recorded during rollout.
The surrogate clips this ratio to $[0.8,1.2]$, applies dual clipping with
$c=10$ for negative advantages, and averages over response tokens selected
by the mask. It uses the recorded rollout denominator directly, with no
old-trainer recomputation or additional off-policy correction.
A K3 KL penalty (\texttt{low\_var\_kl}) with coefficient $0.01$ regularizes
the learner toward the fixed initial reference; the entropy coefficient is zero.

\paragraph{GRPO with token-level truncated importance sampling.}
\label{sec:grpo-tis}
GRPO+TIS uses a token-level importance-weight cap of 2 in the $N=128$
staleness comparison and offline training. Let $\mu$ be the behavior policy
that generated the data and $\pi_{\mathrm{old}}$ the trainer policy at the
start of an update block. At $N=128$, each update block consists of
8 optimizer updates. Offline training uses blocks of 8 optimizer
updates, refreshing $\pi_{\mathrm{old}}$ at each block boundary while
keeping $\mu$ and its stored token log probabilities fixed throughout
training. We compute each block's old-trainer token log probabilities once
under its $\pi_{\mathrm{old}}$.
For token $a_{i\tau}$ with prompt-and-response-prefix context $h_{i\tau}$,
the TIS weight and PPO ratio are
\begin{equation}
\begin{aligned}
 w_{i\tau}
 &=\min\!\left\{\frac{\pi_{\mathrm{old}}(a_{i\tau}\mid h_{i\tau})}
 {\mu(a_{i\tau}\mid h_{i\tau})},\,2\right\},\\
 r_{i\tau}(\theta)
 &=\frac{\pi_\theta(a_{i\tau}\mid h_{i\tau})}
 {\pi_{\mathrm{old}}(a_{i\tau}\mid h_{i\tau})}.
\end{aligned}
\label{eq:grpo-token-tis}
\end{equation}
The TIS weight multiplies the token-level PPO clipped surrogate. The
old-trainer log probabilities, TIS weights, and PPO denominator remain fixed
within each block; the current-policy numerator changes with each learner
update. Weights above 2 are capped at 2, without batch normalization or
rejection sampling. All other settings match the GRPO baseline in the
corresponding setting.

\paragraph{Default AGRO.}
AGRO \citep{tang2025agro} uses leave-one-out centered adjusted rewards.
For response $i$ in group $g$, let
\[
 s_{g,i}(\theta)=\sum_{\tau\in\mathrm{response}}
 \log\pi_\theta(a_{g,i}^{\tau}\mid x_g,a_{g,i}^{<\tau}),
 \qquad
 z_{g,i}=r_{g,i}-\beta\bigl(s_{g,i}(\theta)-s_{g,i}^{\mathrm{ref}}\bigr).
\]
With $\bar z_g=m^{-1}\sum_i z_{g,i}$, the leave-one-out advantage is
\[
 a_{g,i}^{\mathrm{LOO}}
 =z_{g,i}-\frac{1}{m-1}\sum_{j\ne i}z_{g,j}
 =\frac{m}{m-1}(z_{g,i}-\bar z_g).
\]
We remove the exterior $\beta$ factor from AGRO's gradient and use token-mean
aggregation to reuse GRPO's training hyperparameters. For a global learner
minibatch containing $G$ groups with mean response length $\bar\ell$, the
optimized surrogate is
\begin{equation}
 L_{\mathrm{AGRO}}
 =-\frac{1}{Gm\bar\ell}\sum_{g=1}^{G}\sum_{i=1}^{m}
 \operatorname{sg}\!\left(a_{g,i}^{\mathrm{LOO}}\right)s_{g,i}(\theta),
 \label{eq:agro-surrogate}
\end{equation}
where $\beta=0.01$ remains in the adjusted reward and $Gm\bar\ell$ is the
total number of valid response tokens selected by the mask. The operator
$\operatorname{sg}$ stops gradients through the
entire advantage, including its current-policy log ratio; only the final
$s_{g,i}(\theta)$ factor receives gradients. Rewards, reference log
probabilities, and response masks are held fixed. The surrogate uses no PPO
ratio, clipping, additional importance correction, additive KL loss, or
entropy bonus. The reference is refreshed every 160 optimizer updates.

\paragraph{Numerical precision.}
The default GRPO, AGRO, and GGPO runs, as well as the $N=128$
GRPO+TIS runs, store actor parameters in FP32,
execute the forward pass with BF16 mixed precision, and reduce gradients
in FP32. Rollout uses BF16; quantization and quantization-aware training
are disabled.

\paragraph{Implementation and seeds.}
All three objectives are implemented in one fork of verl
\citep{sheng2024hybridflow} (version 0.9.0.dev) with FSDP actor sharding and
vLLM \citep{kwon2023vllm} rollout, so they share the data pipeline, verifier,
refill logic, and evaluation code; their objectives and reference-policy
handling follow the complete configurations above.
Seeds and checkpoint reporting follow
Section~\ref{sec:setup}.

\subsection{Experimental protocols}
\label{app:protocols}

\paragraph{Default training and policy staleness.}
The default and controlled policy staleness runs use 8{,}000 optimizer updates.
Each behavior collection supplies $N$ distinct learner minibatches, with one
update per minibatch and behavior synchronization after the collection;
Section~\ref{sec:async} states the minibatch partition and sample ages.
For $N=128$, the final collection is generated in full, and 64 of its
minibatches are used to complete the 8{,}000-update budget.

Evaluation and checkpointing occur every 160 optimizer updates in every run,
independently of $\Tref$.

\paragraph{Offline setting.}
The fixed rollout pool contains only mixed-reward groups generated once by
\qwenone{}, which is also the initial reference policy. At generation
time, we store each response's behavior token log probabilities and reuse
them throughout training. Current learner token log probabilities are evaluated
on the stored responses for the K3, drift, and ESS measurements; GRPO's
learner-to-behavior ratios are formed against the stored behavior values.
Methods share the same groups and order without new generation. For each seed, we prepare 400
fixed learner minibatches and perform one optimizer update on each, with
evaluation every 80 updates and no response revisited. GRPO+TIS refreshes
its old-trainer policy every 8 updates as specified in
Appendix~\ref{sec:grpo-tis}. AGRO and GGPO use $\Tref=160$, while GRPO
and GRPO+TIS keep the initial reference fixed.

\subsection{GGPO implementation}
\label{app:implementation}

\paragraph{Differentiation semantics.}
For response $i$, the learner log likelihood is
$s_i=\sum_\tau M_{i\tau}\log\pi_\theta(a_i^\tau\mid a_i^{<\tau})$, with the response mask
$M$ selecting the valid generated tokens; the reference log likelihood is
formed over the same masked sequence. Sampled tokens, verifier rewards,
reference log likelihoods, and masks are held fixed during the actor update.
We form the adjusted rewards and directly differentiate the
log-mean-exp loss in Eq.~\eqref{eq:gcpo_loss_norm} with respect to the
learner log probabilities. A detached copy of $q$ is used only to report the
resulting coefficients and diagnostics, not as a separately optimized
surrogate. Autograd therefore produces the gradient in
Eq.~\eqref{eq:gcpo-gradient} exactly, followed by the external normalization in
Eq.~\eqref{eq:gcpo_loss_norm}.

\paragraph{End-to-end GGPO update.}
The update follows Algorithm~\ref{alg:ggpo}. Refill
continues until there are enough mixed-reward groups for $N$ distinct 32-prompt learner
minibatches; the loss is evaluated with a stable log-sum-exp, and each minibatch
receives one AdamW update.

When $N$ does not divide $\Tref$, as at $N=64$ and $N=128$, a refresh boundary can fall
between two minibatches from the same behavior collection. Algorithm~\ref{alg:ggpo} refreshes the reference at that
exact boundary;
the already sampled behavior data remain unchanged, and subsequent learner updates
evaluate their minibatches under the active reference.

\paragraph{Weight-concentration diagnostics.}
With $\mathcal S$ and $k$ as in Appendix~\ref{app:gradient},
over the $N\in\{8,16,32,64\}$ diagnostic subset of the
staleness ladder, the logged within-class normalized effective
sample size---the $\mathrm{ESS}$ of $q$ restricted to $\mathcal S$, divided by
$k$, over groups with $k\ge2$---has medians of $0.77$--$0.80$, while the whole-group
$\mathrm{ESS}_q=1/\sum_iq_i^2$ has medians of $1.96$--$2.28$ out of $m=8$. The
whole-group concentration reflects separation between reward classes.

\section{Additional experiments}
\label{app:additional-experiments}

\subsection{Learning by problem difficulty}
\label{sec:difficulty}

To examine where learning occurs, Figure~\ref{fig:difficulty} splits
MATH-500 by the base model's empirical success rate $\hat p_0$ over 32 responses per problem, into the 410 problems with
$0<\hat p_0\le0.75$ and the 63 with no correct response ($\hat p_0=0$), and tracks
the three default-setting runs on each subset. The remaining 27 problems
reach above 95\% mean@32 for every method within 640 optimizer updates. This division concerns the difficulty of the problems under the
base model. Figure~\ref{fig:group-weight} separately compares the methods'
analytic per-group coefficients.

\begin{figure}[!htbp]
  \centering
  \includegraphics[width=0.80\textwidth]{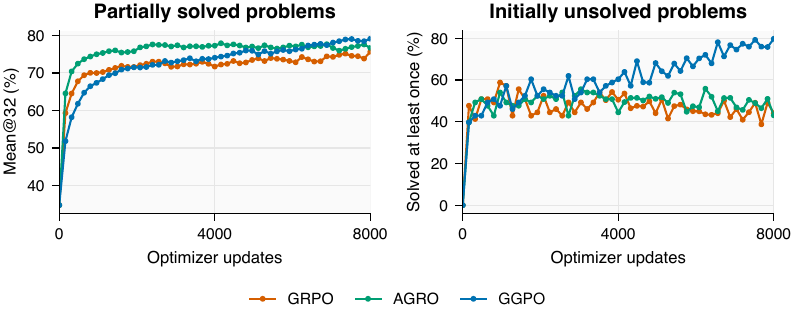}
  \caption{\textbf{MATH-500 learning curves by base-model difficulty on \qwenone{}.}
  Left: mean@32 on the 410 problems the base model solves part of the time
  ($0<\hat p_0\le0.75$). Right: fraction of the 63 initially unsolved problems
  (no correct response among the base model's 32 samples, $\hat p_0=0$) with at
  least one correct response among 32 samples. Both panels plot
  GRPO, AGRO, and GGPO in the default setting with seed 42 against optimizer
  updates.}
  \label{fig:difficulty}
\end{figure}

\begin{figure}[htb]
  \centering
  \includegraphics[width=0.5\textwidth]{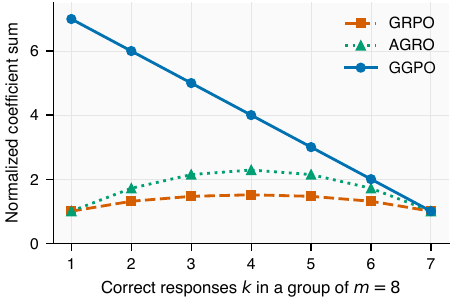}
  \caption{\textbf{Analytic coefficient sum for correct responses in a
  mixed-reward group.} For $m=8$, curves show the signed sum
  $W_{\mathrm{correct}}=\sum_{i:r_i=1}c_i$, normalized to its value at $k=7$
  for each method. GRPO uses standardized reward advantages,
  $W_{\mathrm{correct}}\propto\sqrt{k(m-k)}$; AGRO uses leave-one-out
  advantages with equal learner-to-reference log ratios within the group,
  $W_{\mathrm{correct}}\propto k(m-k)/(m-1)$; GGPO uses
  Eq.~\eqref{eq:normalized-gcpo-gradient} in the binary-reward limit
  $\beta\to0$ at fixed finite log ratios,
  $W_{\mathrm{correct}}=(m-k)/\sqrt{m-1}$.}
  \label{fig:group-weight}
\end{figure}

\subsection{\texorpdfstring{Pass@$K$}{Pass@K} evaluation}
\label{sec:passk}

For each problem, pass@$K$ is the probability that at least one of $K$
independently sampled responses is correct, estimated per problem from the 32
evaluation samples using the unbiased estimator in
Appendix~\ref{app:data-evaluation}. For binary rewards, pass@$K$ and
best@$K$ coincide; the estimated pass@1 equals mean@32. The training curves
report MATH-500 best@32, while Table~\ref{tab:passk-results} reports
problem-weighted pass@$K$ over all five benchmarks at each method's final
8{,}000-update checkpoint.

\begin{table}[htb]
  \centering
  \small
  \renewcommand{\arraystretch}{0.92}
  \setlength{\tabcolsep}{11pt}
  \caption{\textbf{Pass@$K$ across five benchmarks on \qwenone{}.}
  Problem-weighted pass@$K$ (\%) at the final 8{,}000-update checkpoint for
  training seed 42. The base model is evaluated under the same protocol.
  Bold marks the best method at each $K$.}
  \label{tab:passk-results}
  \begin{tabular}{@{}lccccc@{}}
    \toprule
    Method & $K=1$ & $K=4$ & $K=8$ & $K=16$ & $K=32$ \\
    \midrule
    Base model & 19.24 & 38.57 & 47.43 & 55.43 & 62.68 \\
    GRPO & 42.55 & 56.10 & 61.38 & 65.74 & 69.39 \\
    AGRO & 45.06 & 57.41 & 61.96 & 65.85 & 69.19 \\
    GGPO & \textbf{46.60} & \textbf{60.91} & \textbf{66.22} & \textbf{70.46} & \textbf{73.90} \\
    \bottomrule
  \end{tabular}
\end{table}

\paragraph{Results.}
All three methods raise pass@$K$ above the base model at every $K$, and GGPO
has the highest value at every $K$ (Table~\ref{tab:passk-results}).
Its margin over the strongest baseline grows from
1.5 points at $K=1$ to about 4.5 points at $K\ge16$, while AGRO's advantage
over GRPO disappears at larger sampling budgets. GGPO's advantage also
appears in the MATH-500 best@32 curves and the fraction of initially
unsolved problems it solves within 32 samples
(Figures~\ref{fig:main-dynamics} and~\ref{fig:difficulty}).
For a rarely solved problem with success probability $p$ and $Kp\ll1$,
the contribution to pass@$K$ is approximately $Kp$, compared with $p$ to
pass@1. Improvements on such problems can therefore produce larger gains
at higher sampling budgets, consistent with the difficulty analysis.

\subsection{Full results under controlled policy staleness}
\label{sec:staleness-details}

Table~\ref{tab:lag-results-full} gives the per-benchmark final-checkpoint
results across synchronization delays.

\begin{table}[htb]
  \centering
  \small
  \renewcommand{\arraystretch}{0.92}
  \setlength{\tabcolsep}{2.25pt}
  \caption{\textbf{Full evaluation across controlled policy staleness on
  \qwenone{}.} Final-checkpoint mean@32 accuracy after 8{,}000 optimizer updates
  on five benchmarks (\%), by synchronization delay $N$. Results at $N=8,128$
  are three-seed means; $N=8$ reuses Table~\ref{tab:main-results}, and
  $N=16,32,64$ use seed 42. The final column gives the problem-weighted
  average. Sample standard deviations across seeds are reported for all
  columns at $N=8,128$.}
  \label{tab:lag-results-full}
  \begin{tabular}{@{}lcccccc@{}}
    \toprule
    Method ($N$) & MATH-500 & AIME 2024 & AIME 2025 & OlympiadBench & Minerva Math & Avg. \\
    \midrule
    GRPO ($N=8$)  & $68.86\pm0.32$ & $9.03\pm0.90$ & $5.14\pm0.63$ & $32.13\pm0.27$ & $28.33\pm0.40$ & $42.64\pm0.18$ \\
    GRPO ($N=16$) & 68.56 & 8.96 & 5.00 & 31.87 & 28.15 & 42.39 \\
    GRPO ($N=32$) & 68.43 & 8.96 & 5.00 & 31.76 & 28.07 & 42.28 \\
    GRPO ($N=64$) & 68.28 & 8.85 & 5.00 & 31.62 & 27.98 & 42.15 \\
    GRPO ($N=128$) & $64.53\pm0.85$ & $7.71\pm2.35$ & $3.96\pm1.51$ & $28.30\pm0.69$ & $26.02\pm1.01$ & $39.02\pm0.45$ \\
    \shortstack[l]{GRPO+TIS\\($N=128$)} & $66.46\pm0.62$ & $8.30\pm1.20$ & $4.51\pm1.05$ & $30.12\pm0.47$ & $27.09\pm0.64$ & $40.69\pm0.38$ \\
    \midrule
    AGRO ($N=8$)  & $70.68\pm0.33$ & $9.27\pm1.15$ & $6.56\pm0.81$ & $35.73\pm0.30$ & $29.30\pm0.44$ & $45.07\pm0.19$ \\
    AGRO ($N=16$) & 70.88 & 9.06 & 6.25 & 35.91 & 29.11 & 45.17 \\
    AGRO ($N=32$) & 69.98 & 8.85 & 6.04 & 34.48 & 28.83 & 44.17 \\
    AGRO ($N=64$) & 68.29 & 7.29 & 6.25 & 31.07 & 28.10 & 41.92 \\
    AGRO ($N=128$) & $66.94\pm0.67$ & $6.35\pm2.26$ & $5.10\pm1.60$ & $30.12\pm0.60$ & $27.38\pm0.92$ & $40.88\pm0.39$ \\
    \midrule
    GGPO ($N=8$)  & $72.43\pm0.29$ & $11.11\pm0.89$ & $8.30\pm0.68$ & $37.11\pm0.25$ & $30.64\pm0.36$ & $46.58\pm0.16$ \\
    GGPO ($N=16$) & 72.17 & 9.58 & 8.33 & 37.37 & 30.46 & 46.54 \\
    GGPO ($N=32$) & 71.70 & 9.27 & 8.96 & 37.51 & 29.70 & 46.32 \\
    GGPO ($N=64$) & 72.72 & 10.21 & 8.44 & 37.34 & 29.97 & 46.64 \\
    GGPO ($N=128$) & $71.94\pm0.40$ & $10.10\pm1.10$ & $8.33\pm0.90$ & $37.32\pm0.33$ & $30.09\pm0.48$ & $46.39\pm0.21$ \\
    \bottomrule
  \end{tabular}
\end{table}

\begin{figure}[htb]
  \centering
  \includegraphics[width=\textwidth]{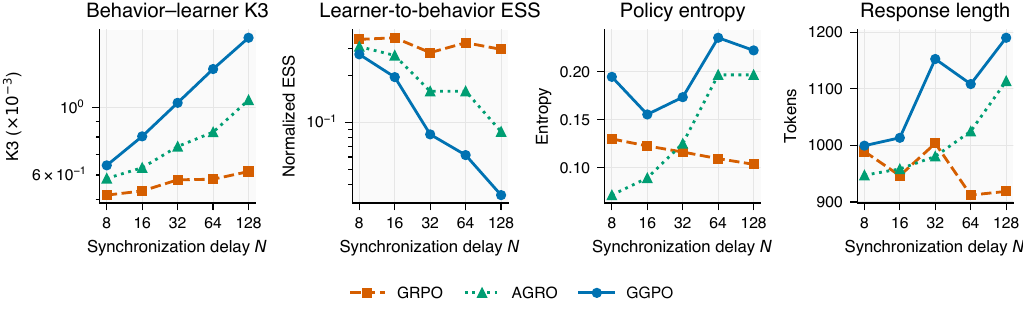}
  \caption{\textbf{Training-time diagnostics across the staleness ladder on
  \qwenone{}.} Each panel plots one metric against synchronization delay $N$
  (equally spaced positions), averaged over the last 800 optimizer updates of
  each run: the token-mean K3 diagnostic of behavior--learner mismatch
  ($\times10^{-3}$),
  the normalized ESS of the diagnostic sequence-level learner-to-behavior
  importance weights (Appendix~\ref{app:data-evaluation}), policy entropy, and
  mean response length in tokens (GRPO orange squares, AGRO green triangles,
  GGPO blue circles). All diagnostics use seed 42; K3 and ESS are computed
  on mixed-reward-filtered training responses.
  K3 and ESS use logarithmic vertical axes.
  Corresponding accuracies are in Table~\ref{tab:lag-results-full}.}
  \label{fig:staleness-mismatch}
\end{figure}

\paragraph{Results.}
At $N=128$, GRPO+TIS improves on GRPO across all five benchmarks, while
GGPO achieves higher accuracy than GRPO+TIS on each benchmark
(Table~\ref{tab:lag-results-full}).
GGPO also retains higher entropy than GRPO at every delay, as in the default
setting (Figure~\ref{fig:staleness-mismatch}).

\subsection{Offline training details}
\label{sec:offline}

Appendix~\ref{app:protocols} gives the generation, logging, and refresh
protocol. Table~\ref{tab:offline-results} reports the mean and sample standard
deviation of MATH-500 accuracy across the three seeds; Figure~\ref{fig:offline}
provides the corresponding curves and seed-42 diagnostics.

\begin{table}[htb]
  \centering
  \small
  \renewcommand{\arraystretch}{0.92}
  \setlength{\tabcolsep}{9pt}
  \caption{\textbf{Offline MATH-500 accuracy across three seeds on
  \qwenone{}.} Mean@32 accuracy (\%) at each evaluated optimizer update,
  reported as the mean $\pm$ sample standard deviation across training seeds
  42, 1234, and 2027. Update 0 is the base model.}
  \label{tab:offline-results}
  \begin{tabular}{@{}rcccc@{}}
    \toprule
    Optimizer updates & GRPO & GRPO+TIS & AGRO & GGPO \\
    \midrule
      0 & 32.96 & 32.96 & 32.96 & 32.96 \\
     80 & 37.63 $\pm$ 0.38 & 54.58 $\pm$ 0.11 & 54.58 $\pm$ 0.35 & 54.83 $\pm$ 0.62 \\
    160 & 38.36 $\pm$ 0.27 & 57.93 $\pm$ 0.33 & 53.72 $\pm$ 0.44 & 59.31 $\pm$ 0.35 \\
    240 & 38.93 $\pm$ 0.31 & 58.28 $\pm$ 0.42 & 57.91 $\pm$ 0.73 & 59.79 $\pm$ 0.69 \\
    320 & 39.79 $\pm$ 0.32 & 54.51 $\pm$ 0.49 & 53.03 $\pm$ 0.68 & 55.83 $\pm$ 0.34 \\
    400 & 39.64 $\pm$ 0.72 & 42.71 $\pm$ 1.23 & 46.43 $\pm$ 0.71 & 48.31 $\pm$ 0.36 \\
    \bottomrule
  \end{tabular}
\end{table}

\subsection{Reference-refresh cadence and regularization}
\label{sec:reference-regularization}

All three methods regularize toward a reference policy; the main comparisons
use the settings of Section~\ref{sec:setup}. We compare each method under a fixed reference,
$\Tref=160$, and $\Tref=320$, and evaluate GGPO at
$\beta=0.1$ with a fixed reference and $\Tref=160$. Table~\ref{tab:reference-refresh} reports the eleven-condition
grid across all five benchmarks at the final checkpoint.

\begin{table}[htb]
  \centering
  \small
  \renewcommand{\arraystretch}{0.92}
  \setlength{\tabcolsep}{2pt}
  \caption{\textbf{Reference-refresh cadence and regularization on \qwenone{}.}
  Final-checkpoint mean@32 accuracy after 8{,}000 optimizer updates on five
  benchmarks (\%). ``Fixed'' keeps the initial reference throughout training;
  $\Tref$ is the refresh cadence in optimizer updates. The $\beta=0.01$ rows at
  each method's default cadence reuse the three-seed means from
  Table~\ref{tab:main-results}; the remaining conditions use seed 42.
  The final column uses the problem-weighted average of
  Table~\ref{tab:main-results}. The compact main-text summary in
  Table~\ref{tab:reference-refresh-summary} retains this seed convention.}
  \label{tab:reference-refresh}
  \begin{tabular}{@{}lcccccc@{}}
    \toprule
    Method, $\beta$, reference & MATH-500 & AIME 2024 & AIME 2025 & OlympiadBench & Minerva Math & Avg. \\
    \midrule
    GRPO, $\beta=0.01$, fixed & 68.86 & 9.03 & 5.14 & 32.13 & 28.33 & 42.64 \\
    GRPO, $\beta=0.01$, $\Tref=160$ & 68.45 & 9.27 & 4.90 & 31.79 & 28.57 & 42.40 \\
    GRPO, $\beta=0.01$, $\Tref=320$ & 68.36 & 9.06 & 4.69 & 32.29 & 27.99 & 42.48 \\
    \midrule
    AGRO, $\beta=0.01$, fixed & 66.61 & 8.13 & 4.06 & 31.64 & 27.48 & 41.48 \\
    AGRO, $\beta=0.01$, $\Tref=160$ & 70.68 & 9.27 & 6.56 & 35.73 & 29.30 & 45.07 \\
    AGRO, $\beta=0.01$, $\Tref=320$ & 69.94 & 8.13 & 4.58 & 34.07 & 28.47 & 43.86 \\
    \midrule
    GGPO, $\beta=0.01$, fixed & 56.52 & 4.79 & 2.92 & 23.26 & 20.11 & 32.96 \\
    GGPO, $\beta=0.01$, $\Tref=160$ & 72.43 & 11.11 & 8.30 & 37.11 & 30.64 & 46.58 \\
    GGPO, $\beta=0.01$, $\Tref=320$ & 70.13 & 9.69 & 6.98 & 35.73 & 29.62 & 44.96 \\
    GGPO, $\beta=0.1$, fixed & 54.09 & 3.44 & 2.08 & 20.98 & 19.49 & 30.98 \\
    GGPO, $\beta=0.1$, $\Tref=160$ & 72.04 & 9.69 & 6.88 & 35.48 & 29.73 & 45.50 \\
    \bottomrule
  \end{tabular}
\end{table}

\begin{table}[htb]
  \centering
  \small
  \renewcommand{\arraystretch}{0.92}
  \setlength{\tabcolsep}{9pt}
  \caption{\textbf{Learner-to-reference log-ratio and entropy diagnostics on
  \qwenone{}.} Mean sequence log ratio $\log[\pi_\theta/\piref]$ over
  the responses in the training batch, and policy entropy, each averaged over
  the last 800 optimizer updates of the seed-42 run.}
  \label{tab:refresh-diagnostics}
  \begin{tabular}{@{}llccc@{}}
    \toprule
    Method & Metric & Fixed & $\Tref=160$ & $\Tref=320$ \\
    \midrule
    AGRO & Log ratio (nats) & 19.2 & 3.8 & 6.6 \\
         & Entropy & 0.19 & 0.07 & 0.10 \\
    \midrule
    GGPO & Log ratio (nats) & 6.0 & 0.6 & 0.7 \\
         & Entropy & 0.35 & 0.19 & 0.18 \\
    \bottomrule
  \end{tabular}
\end{table}

\begin{figure}[!htb]
  \centering
  \includegraphics[width=0.70\textwidth]{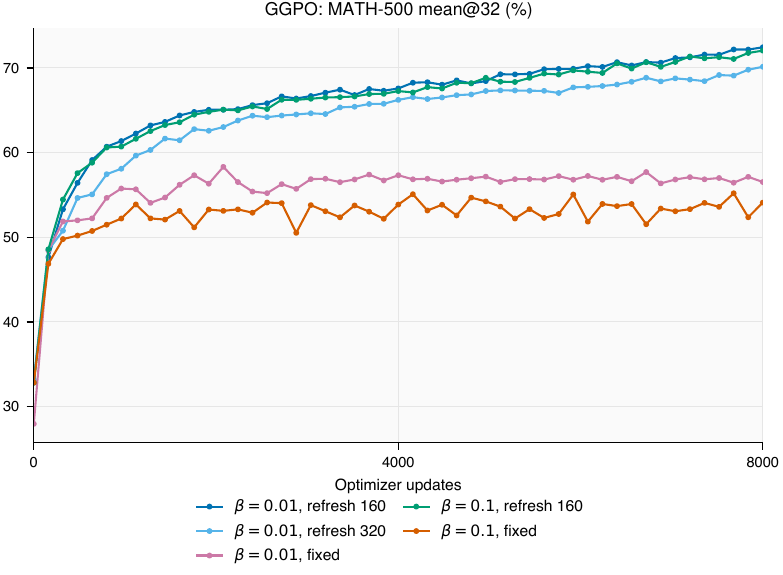}
  \caption{\textbf{GGPO training under different reference-refresh cadences
  and regularization strengths on \qwenone{}.} MATH-500 mean@32 accuracy
  against optimizer updates for $\beta\in\{0.01,0.1\}$ with a fixed reference
  or refresh every 160 updates, and for $\beta=0.01$ with refresh every 320
  updates. All curves use seed 42. Final results across all five benchmarks are reported in
  Table~\ref{tab:reference-refresh}.}
  \label{fig:reference-refresh}
\end{figure}

\paragraph{Results.}
The ablation shows that periodic reference refresh is an empirically important
part of GGPO, rather than only a theoretical extension. At $\beta=0.01$, the
problem-weighted average rises from $32.96$ with a fixed reference to $46.58$
with the default $\Tref=160$ schedule; $\Tref=320$ reaches $44.96$
(Table~\ref{tab:reference-refresh}). Thus a fixed reference costs GGPO 13.6
points and leaves it below fixed-reference GRPO. AGRO also performs best with
$\Tref=160$, but its fixed-reference penalty is smaller at 3.6 points, while
GRPO is essentially unchanged across the three conditions.

Figure~\ref{fig:reference-refresh} shows that fixed-reference GGPO levels
off after its early gains, with a lower plateau at $\beta=0.1$.
With reference refresh, accuracy continues to improve through 8{,}000
optimizer updates at both temperatures.

The reference shapes AGRO's and GGPO's training targets, whereas GRPO
uses it only in an additive KL penalty
(Appendix~\ref{app:experiments}). This distinction is consistent with
GRPO's small accuracy changes across refresh cadences.
GGPO shows a larger accuracy improvement from reference refresh than
AGRO and GRPO in this comparison. The following diagnostics examine the associated reference drift
and changes in response weights.

\paragraph{Reference-drift diagnostics.}
Table~\ref{tab:refresh-diagnostics} provides complementary empirical evidence.
With the common fixed reference, GGPO's mean sequence log ratio is $6.0$ nats,
compared with AGRO's $19.2$. The same ordering holds with $\Tref=160$
($0.6$ versus $3.8$) and $\Tref=320$ ($0.7$ versus $6.6$).
GGPO's consistently lower log ratios on the retained training responses are
consistent with stronger effective reference anchoring under the implemented
updates. The fixed-reference comparison is especially
informative because both methods use the same initial anchor.
These means are computed on each method's behavior-policy samples after
mixed-reward filtering.

\paragraph{Weight concentration among successful responses.}
Figure~\ref{fig:reference-weight-ess} compares the distribution of GGPO's
weights among successful responses with a fixed reference and with
$\Tref=160$ at $\beta=0.1$. For groups with
$k\ge2$ successful responses, let $\widetilde q_i=q_i/\sum_{j\in\mathcal S}q_j$
for $i\in\mathcal S$, where $\mathcal S$ is the successful subset. The
within-success normalized ESS is
$[k\sum_{i\in\mathcal S}\widetilde q_i^2]^{-1}\in[1/k,1]$;
one denotes uniform weights, and smaller values indicate greater concentration.
Because these responses have equal rewards, Eq.~\eqref{eq:q-logits} gives
$\widetilde q_i\propto e^{-\Delta_i}$, where
$\Delta_i=\log[\pi_\theta(a_i\mid x)/\piref(a_i\mid x)]$.
Thus this diagnostic isolates how learner-to-reference log-ratio differences
redistribute weight within the successful subset.

\begin{figure}[!htb]
  \centering
  \includegraphics[width=0.70\textwidth]{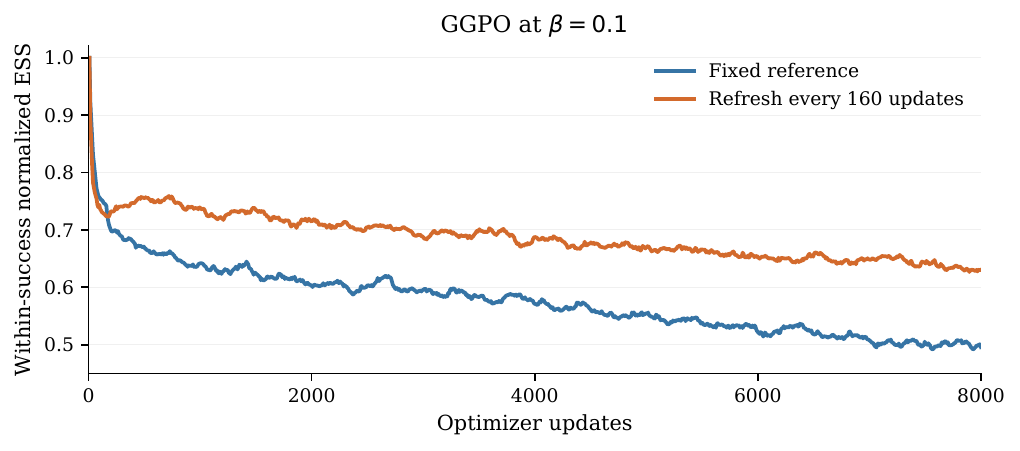}
  \caption{\textbf{Within-success weight concentration for GGPO on \qwenone{}.}
  Normalized ESS of the weights restricted to successful responses, for groups
  with at least two successes, at $\beta=0.1$ with seed 42. Blue uses a fixed
  reference; orange refreshes every 160 optimizer updates. Higher values
  indicate more uniform weights among successful responses.}
  \label{fig:reference-weight-ess}
\end{figure}

After the initial decline, the refreshed-reference curve remains above the
fixed-reference curve, although both decrease over training.
The higher ESS with refresh is consistent with limiting accumulated
log-ratio differences among successful responses, as described in
Appendix~\ref{app:anchoring}. This provides a weight-level
view of the reference dependence accompanying the accuracy trends in
Figure~\ref{fig:reference-refresh}.

Averaging over the final five recorded
refresh windows and the matching fixed-reference interval gives normalized
ESS values of $0.582$ with a fixed reference and $0.690$ with refresh;
the corresponding fractions of groups with $k\ge2$ are $0.758$ and $0.839$.
The raw group loss remains nonzero, averaging $0.619$ and $0.531$, respectively.
These diagnostics establish the presence of reference-dependent reweighting
during finite-loss training but do not isolate its contribution to the accuracy gain.

\subsection{Ablations on GGPO's group weights}
\label{sec:group-weight-ablation}

GGPO's weights are $q_i=\softmax_i(r_i/\beta-\Delta_i)$, where
$\Delta_i=\log[\pi_\theta(a_i\mid x)/\piref(a_i\mid x)]$.
We vary the reward temperature $\beta$ and remove the reference-dependent
term $-\Delta_i$ under the default \qwenone{} configuration with $N=8$.

\paragraph{Temperature choice.}
Reducing $\beta$ from $0.01$ to $0.001$ changes GGPO's problem-weighted
average from $46.58$ to $46.78$, while AGRO decreases from $45.07$ to $44.37$.
We use $\beta=0.01$ in the main comparisons.

\paragraph{Removing the reference-dependent term.}
Under the default $N=8$ setting, we isolate the contribution of
reference-dependent weighting by replacing $q_i=\softmax_i(r_i/\beta-\Delta_i)$
with $q_i=\softmax_i(r_i/\beta)$ at $\beta=0.01$, while retaining the
remaining GGPO update. The reward-only variant scores $45.12$ with seed 42,
compared with the full method's three-seed mean of $46.58$.
For mixed binary-reward groups, the resulting update coefficients approach
the success-normalized form of MaxRL \citep{tajwar2026maxrl} at small $\beta$, up to
GGPO's normalization (Appendix~\ref{app:gradient}). This component ablation
supports reference-dependent weighting in the default $N=8$ setting.

\subsection{Finite group size}
\label{sec:group-size-ablation}

We study the effect of group size at a fixed optimizer-update budget by
varying the number of responses per prompt over $m\in\{4,8,16\}$ for
\qwenone{} under the default GGPO configuration. Every run uses 32 groups
per learner minibatch and 8{,}000 optimizer updates; the $m=8$ condition
reuses the default-setting run. Each update uses $32m$ retained responses,
and raw generation cost also depends on mixed-group acceptance and refill.

\begin{table}[H]
  \centering
  \small
  \renewcommand{\arraystretch}{0.92}
  \setlength{\tabcolsep}{11pt}
  \caption{\textbf{Evaluation across GGPO group sizes on \qwenone{}.}
  Problem-weighted best@$K$ (\%) over the five benchmarks at the final
  8{,}000-update checkpoint for training seed 42. For binary rewards,
  best@$K$ equals pass@$K$. The $m=8$ row reuses the GGPO results in
  Table~\ref{tab:passk-results}.
  All runs use 32 groups per learner minibatch.}
  \label{tab:group-size}
  \begin{tabular}{@{}lccccc@{}}
    \toprule
    Group size & $K=1$ & $K=4$ & $K=8$ & $K=16$ & $K=32$ \\
    \midrule
    $m=4$  & 46.51 & 60.73 & 65.97 & 70.12 & 73.17 \\
    $m=8$  & 46.60 & 60.91 & 66.22 & 70.46 & 73.90 \\
    $m=16$ & 46.88 & 61.43 & 66.80 & 71.25 & 75.10 \\
    \bottomrule
  \end{tabular}
\end{table}

\paragraph{Results.}
Increasing $m$ from 4 to 16 improves best@$K$ at every reported $K$, with
gains growing from 0.37 points at $K=1$ to 1.93 points at $K=32$
(Table~\ref{tab:group-size}). Larger groups are
more likely to contain a correct response on difficult prompts, providing a
mixed-reward training signal that smaller groups may miss. The larger gains
at higher sampling budgets are consistent with improvements on problems that
the learner solves infrequently.

\subsection{Sampling budget}
\label{sec:sampling-budget}

Figure~\ref{fig:sampling-budget} complements the update-based comparison in
Section~\ref{sec:default-training} with cumulative generated training responses
for \qwenone{} under the default setting ($N=8$, $m=8$).
Counts include all generated candidate responses, including those discarded
by mixed-reward filtering and those generated during refill
(Appendix~\ref{app:data-evaluation}).

\subsection{Token-clipping training rewards}
\label{sec:token-clipping-results}

Figure~\ref{fig:token-clipping} gives the single-seed training-reward
comparison discussed in Section~\ref{sec:token-clipping-experiment}.
The rewards are measured on mixed-reward-filtered training responses;
Appendix~\ref{app:token-clipping} gives the implementation.

\begin{figure}[!htbp]
  \centering
  \includegraphics[width=0.70\textwidth]{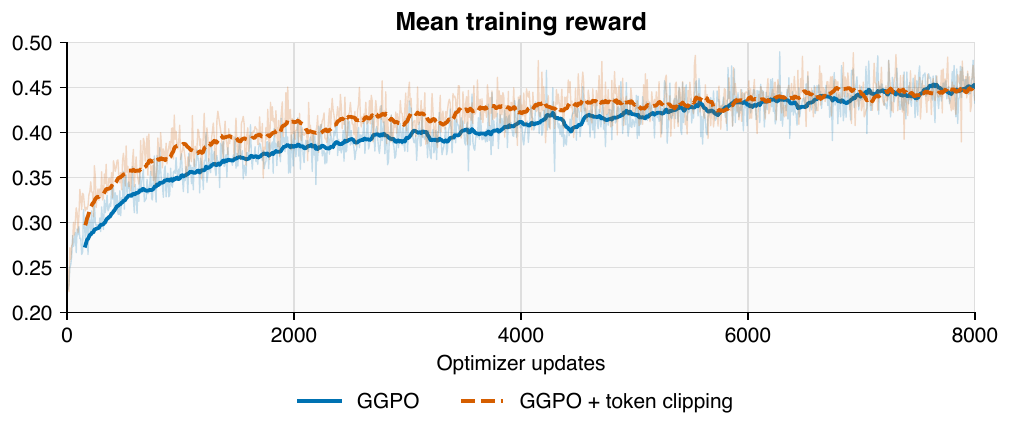}
  \captionsetup{font=small}
  \caption{\textbf{Token clipping improves early training progress.}
  Mean training reward (higher is better) on \qwenone{} against optimizer
  updates, with $N=8$, $m=8$, $\beta=0.01$, $\Tref=160$, and seed 42.
  Blue shows the original GGPO update; orange shows the detached-coefficient
  variant with token log-ratio clipping at $\pm0.05$.
  Rewards are measured on mixed-reward-filtered training data. Faint lines
  show values aggregated over eight updates; bold lines show trailing
  160-update means. Clipping raises early training reward, with a smaller
  difference later in training.}
  \label{fig:token-clipping}
\end{figure}

\clearpage
\section{Token-clipping implementation}
\label{app:token-clipping}
This appendix explains how the clipped adjusted rewards determine response
coefficients and how those coefficients enter the detached-gradient surrogate
used in Section~\ref{sec:token-clipping-experiment}. The original
GGPO algorithm uses no token log-ratio clipping.

\paragraph{Clipped reference correction.}
For response $a_i^g$ in group $g$, let
$h_{g,i,t}=(x_g,a^g_{i,<t})$ and let $M_{g,i,t}$ be the valid-token mask.
We clip each learner-to-reference token log ratio before summing it:
\begin{align}
 d_{g,i,t}(\theta)
 &=\log\frac{\pi_\theta(a^g_{i,t}\mid h_{g,i,t})}
                  {\piref(a^g_{i,t}\mid h_{g,i,t})},
 \nonumber\\
 \bar R_{g,i}(\theta)
 &=r(x_g,a_i^g)-\beta\sum_t M_{g,i,t}
       \operatorname{clip}(d_{g,i,t}(\theta),-c_{\mathrm{clip}},c_{\mathrm{clip}}),
 \label{eq:clip-adjusted-reward}\\
 \bar q_{g,i}(\theta)
 &=\frac{\exp(\bar R_{g,i}(\theta)/\beta)}
         {\sum_j\exp(\bar R_{g,j}(\theta)/\beta)},
 \qquad
 \bar c_{g,i}(\theta)=\frac{m\bar q_{g,i}(\theta)-1}{\sqrt{m-1}}.
 \label{eq:clip-response-coefficient}
\end{align}
We use $c_{\mathrm{clip}}=0.05$, equivalent to clipping the token probability
ratio to $[e^{-0.05},e^{0.05}]$. The denominator is the reference policy;
this is not a learner-to-behavior importance weight.

\paragraph{Objective and stop-gradient.}
Let $T_{\mathrm{tok}}=\sum_{g,i,t}M_{g,i,t}$ be the total valid-token count
in the learner minibatch, and let $\theta_0$ denote the parameters at which
the response coefficients are computed for the current update. We minimize
the local surrogate
\begin{equation}
 L_{\mathrm{clip}}(\theta;\theta_0)
 =-\frac{1}{T_{\mathrm{tok}}}\sum_{g,i}
   \operatorname{sg}[\bar c_{g,i}(\theta_0)]
   \sum_t M_{g,i,t}\log\pi_\theta(a^g_{i,t}\mid h_{g,i,t}),
 \label{eq:clip-surrogate}
\end{equation}
where $\operatorname{sg}$ holds its entire argument constant during
differentiation. Clipping changes the coefficients, while
differentiation uses the original, unclipped token log probabilities. Tokens outside the clipping
interval still contribute their score gradients, and all valid tokens
remain in $T_{\mathrm{tok}}$. Coefficients are recomputed for each update;
rewards, sampled tokens, reference probabilities, and masks are held fixed.
The normalization matches Eq.~\eqref{eq:gcpo_loss_norm}, with no additional
exterior factor of $\beta$.

For logging, the implementation retains the clipped log-mean-exp statistic
\begin{equation}
 \widehat{\mathcal J}_{m,c_{\mathrm{clip}}}^{\,g}
 =\beta\log\!\left(\frac1m\sum_i e^{\bar R_{g,i}/\beta}\right)
  -\frac1m\sum_i\bar R_{g,i}.
 \label{eq:clip-gap-statistic}
\end{equation}
The update is the gradient of Eq.~\eqref{eq:clip-surrogate} at
$\theta=\theta_0$, rather than ordinary differentiation through
Eq.~\eqref{eq:clip-gap-statistic}. The latter would introduce a clamp
derivative that suppresses gradients at saturated tokens. Removing clipping
recovers the normalized original GGPO gradient on the same minibatch.

\end{document}